\documentclass{article}

\input{macro.sty}

\hypersetup{
  colorlinks,
  linkcolor={red!60!black},
  citecolor={blue!60!black},
  urlcolor={blue!80!black}
}

\title{High-Dimensional Asymptotics and Dataset Selection\\ for Private Transfer Learning}

\makeatletter
\renewcommand{\@fnsymbol}[1]{\@arabic{#1}}
\makeatother

\author{
Filip Kova\v{c}evi\'c\thanks{
Institute of Science and Technology Austria.\protect\\
\hspace*{3mm} $^2$CNRS Dauphine PSL.\protect\\
\hspace*{3mm} $^3$Computer Science and Engineering Department,
Chalmers University of Technology and University of Gothenburg.\protect\\
\hspace*{3mm} $^4$School of Computer Science and Applied Mathematics,
University of the Witwatersrand.\protect\\
\hspace*{4.3mm} Correspondence: \href{mailto:filip.kovacevic@ist.ac.at}\texttt{filip.kovacevic@ist.ac.at}.
}
\qquad
Edwige Cyffers$^2$
\qquad
Stefano Sarao Mannelli$^{3,4}$
\qquad
Marco Mondelli$^1$
}
\begin{document}

\maketitle

\begin{abstract}
To commit to buying external data or participate in collaborative learning, one must decide whether the additional data will improve prediction
enough to justify the cost. This comes with several challenges: \emph{(i)} the decision often relies only on aggregated statistics available publicly, rather than individual-level data; \emph{(ii)} covariate and model shifts can induce negative transfer, so the additional data deteriorates rather than improves performance; \emph{(iii)} if the data is sensitive, its privatization requires the injection of noise, which can also offset the benefit of a larger sample size. In this paper, we model the problem of dataset selection through high-dimensional regression with multiple heterogeneous sources and a weighted ridge estimator. Our approach uses only summary statistics and it gives privacy guarantees either on labels only or jointly on features and labels, in terms of $\rho$-zero-concentrated differential privacy. The main technical contribution is a deterministic equivalent of the test error, which captures the interactions between sample size, covariance structure, model shift, regularization and privacy noise. Our theory allows to optimize hyperparameters (weights and ridge regularizers) and, more broadly, to decide when private external datasets are useful without accessing the data itself but only relying on population-level quantities. This provides a theoretically tractable foundation for private transfer learning, which we support via experiments on both synthetic and real-world datasets.
\end{abstract}

\section{Introduction}
Should one buy external data or participate in collaborative learning? The answer depends on whether the extra data improves prediction 
enough to justify its cost. When handling sensitive data, this decision may rely on aggregate statistics available before access to individual records. For example, the UK Biobank Data Showcase provides public summaries to assess relevance before applying for access \citep{Sudlow2015} and the US Census Bureau releases statistics aggregated at county or state levels~\citep{USCensus, Abowd2022}. Collaboration between hospitals raises the same question for heterogeneous, sensitive datasets 
\citep{ogier2022flamby}: are the efforts and risks worth the accuracy gain?

Indeed, incorporating source data does not necessarily improve
target prediction.
Differences in feature distributions
(covariate shift,
\citep{sugiyama2007covariate,ma2023optimally})
and in the relationship between features and labels
(model shift,
\citep{yang2025precise,li2022transfer})
can limit the benefits of transfer and, depending on how source
data are incorporated, lead to negative transfer
\citep{pan2009survey,wang2019characterizing}.
To understand when transfer remains beneficial despite these shifts,
recent work has developed precise characterizations of the prediction
risk for high-dimensional linear regression
\citep{song2024generalization,yang2025precise}.


Differential privacy adds another drawback: the noise
injected to protect individuals can offset the benefit of
additional samples \citep{dwork2014algorithmic}. Thus, 
a rich line of theoretical work has quantified
privacy-accuracy trade-offs in empirical risk minimization
\citep{chaudhuri2011differentially,bassily2014private}
and stochastic convex optimization \citep{bassily2019private},
alongside minimax guarantees for statistical estimation
\citep{duchi2018minimax,cai2021cost}.
In collaborative settings, unequal privacy noise levels across
sources affect how their contributions should be combined
\citep{liu2021projected}, and recent work has characterized the accuracy
achievable under privacy constraints in federated transfer learning
with heterogeneous sources \citep{li2024federated}. We also note that, although estimating the gains from collaboration has been widely studied in federated learning \citep{mcmahan2017communication,sattler2020clustered}, the cost of participation motivates assessing the value of collaboration \emph{before} committing. 
A central question therefore remains: 

\begin{quote}
\emph{Can we predict whether and by how much heterogeneous private data will
improve target prediction, before processing the data itself?}
\end{quote}

We address this question for a high-dimensional linear model by analyzing weighted ridge regression with multiple heterogeneous and differentially private sources. This approach uses only summary statistics and it allows both model and covariance shift. Our main result is a deterministic equivalent for the risk which provides a principled basis to select sources, having access only to population-level quantities (as for, e.g., the US Census~\citep{USCensus, Abowd2022} or the UK Biobank \citep{Sudlow2015}). More precisely, we summarize our contributions below:
\begin{itemize}[itemsep=2pt,leftmargin=*]
    \item \textbf{Deterministic equivalent for the risk.} In the regime where sample sizes and feature dimension scale proportionally, we provide a precise characterization for the bias and variance of the weighted ridge estimator under label privatization (Theorem~\ref{thm:main-det-eq}), and extend the analysis to joint feature and label privatization for Gaussian data (Corollary~\ref{cor:C-client-xy-private}). The characterizations depend only on population quantities and they quantitatively capture covariate shift, model shift, regularization, and privacy noise.
    \item \textbf{Optimal tuning.} For a Gaussian prior, we identify source weights and ridge penalties that minimize the risk and show that the resulting ridge estimator is Bayes-optimal (Proposition~\ref{prop:bayes-opt-param}). 
    \item \textbf{Risk prediction and dataset selection in real-world scenarios.} We demonstrate the validity of our analysis beyond linear models. On three real-world datasets, we compare deterministic equivalent predictions with empirical changes in target prediction risk as auxiliary data is added (Figure~\ref{fig:exact_fit}), as well as the effectiveness of the resulting selection of datasets (Figure~\ref{fig:source_selection}). 
\end{itemize}

\section{Related work}

\paragraph{Transfer learning under covariate and model shifts.}
The usefulness of auxiliary data depends on differences between source and target distributions and on how information is shared across tasks. Statistical approaches provide estimation and prediction guarantees, together with methods for identifying and exploiting informative sources in high-dimensional regression \citep{li2022transfer,tian2023transfer}. The benefits of sharing information also depend on task similarity, data geometry, and the representations learned across tasks, as studied through joint training and representation transfer \citep{wu2020understanding,dhifallah2021phase,sarao2025bias,gerace2022probing,ingrosso2025statistical}. Related work examines how regularization and minimum-norm interpolation affect prediction under distribution shift \citep{patil2024optimal,mallinar2024minimumnorm}. Closer to our setting, precise high-dimensional risk analyses show that adding source data can worsen target prediction, with positive and negative transfer depending on distribution shifts, noise levels, and sample sizes \citep{yang2025precise,song2024generalization}. Source weighting and regularization can improve prediction by accounting for task relatedness and noise \citep{zhang2023transfer,ilbert2024analysing}. Similar questions arise when combining real and surrogate data or selecting data under weak supervision, where risk analyses inform data weighting and selection \citep{jain2024scaling,kolossov2024towards,rezaei2026high}. We highlight that none of these analyses address source-specific privacy constraints. We instead characterise how privatising labels, or both features and labels, changes target prediction risk under covariate and model shifts.

\paragraph{Data acquisition and differential privacy.}
Privacy-preserving methods have long addressed how to provide sufficient protection while retaining the benefits of collaboration, through the pricing of noisy query answers \citep{pricingprivatedata} and data acquisition more generally. \citet{Fallah2022OptimalAD,Cummings2022OptimalDA} for instance examine the interplay between privacy costs and statistical accuracy. Privacy can undermine the benefits of collaboration, motivating the need for decision rules on whether or not joining forces, which has been studied from a federated learning perspective \citep{pmlr-v238-tsoy24a, mcmahan2017communication, sattler2020clustered, 10992542}. Another related line of work combines public and private data, showing how a large private dataset can complement a small number of observations that do not require differential privacy protection \citep{publicprivate19,NEURIPS2022_765ec499,NEURIPS2024_243697ac,Block2024OracleEfficientDP}. These approaches focus on how to combine the available data, rather than predicting whether to collaborate, and often assume the same data distribution or have no theoretical guarantees for optimization \citep{NEURIPS2019_3c88c1db}. We instead predict the value of heterogeneous private datasets, accounting for covariate and model shifts as well as privacy noise. In collaborative settings, aggregate feature statistics are sometimes publicly available. When the features themselves are public and only the labels are sensitive, label differential privacy provides an appropriate model \citep{NEURIPS2021_e3a54649}. This setting has also been studied for regression: \citet{ghazi2023regression} derive optimal label-randomization mechanisms for a given prior distribution over labels and regression loss. We use this setting in Section~\ref{sec:labeldpresult} before considering the privatization of both features and labels in Section~\ref{subsec:xy-private}.

\vspace{-.4em}

\section{Preliminaries}
\label{sec:preliminaries}

\vspace{-.4em}

\paragraph{Linear model for transfer learning.}
We consider a linear regression model with $C$ clients, where $C\geq 2$ does not grow with the dimension.  Client
$c\in \cC \coloneqq \{1,\ldots,C\}$ observes a local dataset $(X_c,y_c)$, where the design matrix
$X_c\in\bbR^{n_c\times d}$ and the observation $y_c\in\bbR^{n_c}$ are given by
\begin{equation*}
    y_c=X_c\beta_c+\varepsilon_c,
\end{equation*}
where the entries of $\varepsilon_c$ are independent (between themselves and w.r.t.\ $\{(X_c, \beta_c)_{c=1}^C\}$), have zero mean, variance $\sigma_c^2$ and uniformly bounded $8$-th moments. The regression vectors $\beta_c\in\bbR^d$ and the feature distributions may vary across clients. In particular,
\begin{equation*}
    X_c=Z_c\Sigma_c^{1/2},
\end{equation*}
where $\Sigma_c\in\bbR^{d\times d}$ is the deterministic population covariance of client
$c$, and the entries of $Z_c\in\bbR^{n_c\times d}$ are random. 
%
Without loss of generality, client $1$ is the target client, and the prediction performance is evaluated on new observations drawn from the same distribution as client $1$. The remaining
clients provide auxiliary data that can be shared only in a privatized way, and they form the fixed set $\cA\coloneqq\{2,\ldots,C\}$. 
We simplify the model shift notation by introducing
\begin{equation*}
    \beta\coloneqq\beta_1, \qquad \delta_c\coloneqq\beta_c-\beta,\ c\in\cA, \qquad \theta\coloneqq \matrix{\beta&\delta_2&\hdots&\delta_C}^\top
    \in\bbR^{Cd}.
\end{equation*}

\paragraph{Privatizing data.} Differential privacy (DP) relies on the concept of adjacent datasets: two datasets $D$ and $D'$ are adjacent, written $D\sim D'$, 
if one is obtained from the other by replacing a single input-label pair. We quantify DP via $\rho$-zero-concentrated differential privacy ($\rho$-zCDP) which uses the
$\alpha$-Rényi divergence to measure the distinguishability between two probability distributions $P\ll Q$: 
\[
  D_\alpha(P\|Q) :=\frac{1}{\alpha-1} \log\int\left(\frac{\mathrm dP}{\mathrm dQ}\right)^\alpha \mathrm dQ, \qquad \alpha>1.
\]

\begin{defn}[Zero-concentrated differential privacy
  \citep{bunConcentratedDifferentialPrivacy2016}]
  \label{def:zcdp}
  A randomized algorithm $\cM$ satisfies $\rho$-zero-concentrated differential privacy ($\rho$-zCDP) if, for every pair of adjacent datasets $D, D'$ and every $\alpha>1$,
  \[
    D_\alpha\rbr{\cM(D)\|\cM(D')}\leq\rho\alpha.
  \]
\end{defn}
Guarantees for $\rho$-zCDP can be translated to 
the standard notion of $(\epsilon,\delta)$-differential privacy 
by taking 
$\epsilon=\rho+2\sqrt{\rho\log(1/\delta)}$ \citep{bunConcentratedDifferentialPrivacy2016}. We achieve $\rho$-zCDP via the Gaussian mechanism \citep{dwork2014algorithmic} and provide the following two different privatization guarantees.

\emph{Label privatization:}  We privatize only the labels $y_c$ by releasing 
\begin{equation}
\label{eq:C-client-label-private}
    \tilde y_c= \mathrm{Clip}_{K_{y,c}}(y_c)+\xi_c, \qquad \xi_c\sim\cN(0,\sigma_{y,\mathrm{DP},c}^2 I_{n_c}), \qquad \sydpc^2 = \frac{\min(\Delta_{y,c}^2, K_{y, c}^2)}{2\rho_{y,c}}.
\end{equation}
Here, $\mathrm{Clip}_{K_{y,c}}(x) = x\min\{1,K_{y,c}/(2\abs{x})\}$ is the entry-wise clipping function with a constant $K_{y,c}$ that could depend on $n_c$ and $d$, and $\Delta_{y,c} = \max_{y_c'\sim y_c''}\norm{y_c' - y_c''}_2$ denotes the $l_2$-sensitivity of the labels of client $c$. The release in \eqref{eq:C-client-label-private} satisfies $\rho_{y,c}$-zCDP \citep{bunConcentratedDifferentialPrivacy2016}.


\emph{Feature and label privatization:} We privatize both features $X_c$ and labels $y_c$
by releasing
\begin{equation}\label{eq:C-client-feature-label-private}
    \tilde X_c=\mathrm{Clip}_{K_{X,c}}(X_c)+\Xi_c, \qquad \tilde y_c=\mathrm{Clip}_{K_{y,c}}(y_c)+\xi_c,
\end{equation}
where $\mathrm{Clip}_{K_{X,c}}(v) = v\min\{1,\sqrt{d}K_{X,c}/(2\norm{v}_2)\}$ is the row-wise clipping function with a constant $K_{X,c}$ that could depend on $n_c$ and $d$, entries of $\Xi_c$ are i.i.d.\ $\sim \cN(0,\sxdpc^2)$, $\sxdpc^2 = \min(\Delta_{X,c}^2, d K_{X,c}^2)/2\rho_{X,c}$,
$\xi_c$ is in \eqref{eq:C-client-label-private}, and $\Delta_{X,c}=\max_{X_c'\sim X_c''}\norm{X_c' - X_c''}_2$ is the sensitivity of the features of client $c$.
The release in \eqref{eq:C-client-feature-label-private} satisfies $(\rho_{X,c}+\rho_{y,c})$-zCDP \citep{bunConcentratedDifferentialPrivacy2016}.
All noises ($\{\varepsilon_c\}_{c\in\cC}$, $\{\xi_c\}_{c\in\cA}$,
$\{\Xi_c\}_{c\in\cA}$) are independent among themselves and of $(\theta,X_1,\ldots,X_C)$.

\paragraph{Assumptions.} \label{eq:aa} 

We consider the asymptotic regime where $d,n_1,\dots, n_C\to \infty$ and set 
$n\coloneqq\sum_{c=1}^C n_c,\ \pi_c\coloneqq\frac{n_c}{n}$, and $\gamma_c\coloneqq\frac{d}{n_c}$. We assume an underparameterized regime \myassump{ass:C-client-asymptotics}, so that 
  $\pi_c\to\bar\pi_c\in(0,1)$ and 
  $\gamma_c\to\bar\gamma_c\in(0,1)$ for $c\in\{1,\ldots,C\}$.
We further assume almost surely $\sup_{c,n,d}\norm{\beta_c}_2<\infty$ and all entries of $Z_c$ to be independent (between themselves and w.r.t.\ $\{\beta_1, \ldots, \beta_C\}$), identically distributed within each client, zero-mean, unit-variance and having uniformly bounded $8$-th moment \myassump{ass:z-moment-bounds}. 
The population covariances are uniformly well-conditioned  \myassump{ass:C-client-cov-bounds}, i.e., %
    $\omega_\Sigma I_d\preceq\Sigma_c\preceq \alpha_\Sigma I_d$ uniformly in $c$ and $d$, for constants $0<\omega_\Sigma<\alpha_\Sigma<\infty$. 

We assume that noise levels, $\sigma_c^2$, $\sydpc^2$ and $\sxdpc^2$, are constants independent of $d, n$ \myassump{ass:noise_const} and that clipping does not occur almost surely for large enough $n$ \myassump{ass:priv_budg_large}. This last assumption is technical and we defer the analysis of clipping to future work, see Section \ref{sec:conc}. We note that this covers several practical settings: if the data are normalized so that the entries of $X_c, y_c$ are bounded, then we can take $K_{y, c}, K_{X, c}$ of constant order and \eqref{eq:C-client-label-private}-\eqref{eq:C-client-feature-label-private} hold with $\rho_{y, c}$ of constant order and $\rho_{X, c}$ of order $d$; if the data has sub-Gaussian tails, then $K_{y, c}, K_{X, c}$ are of order $\sqrt{\log d}$, which adds a logarithmic factor to $\rho_{y, c}, \rho_{X, c}$ as well (see Appendix \ref{app:aux-clipping} for details).  Whereas taking $\rho_{y, c}$ of
constant order is standard in privacy, letting $\rho_{X, c}$ grow with $d$ is due to the fact that features are $d$-dimensional:
$\rho_{X, c}$ of order $d$ corresponds to a noise of constant order per each coordinate. 
Moreover, we adopt a strong privacy model, for which large budgets are known to
be needed \citep{Bhowmick2018ProtectionAR,Cyffers2025SettingI}, and recent evidence indicates that the
budget can grow with the dimension while protecting from reconstruction attacks
\citep{pmlr-v258-liu25b,3666122.3669564, pmlr-v162-guo22c}.



\paragraph{Weighted ridge estimator.}
For label privatization, let $\Xeff_c=X_c$ for every
$c\in\cA$.  For feature and label privatization, $\Xeff_c$ is the conditionally denoised design defined in Section~\ref{subsec:xy-private}.  In both cases, set
$\Xeff_1=X_1$ and $\tilde y_1=y_1$. 
We consider the weighted joint ridge estimator
\begin{equation}
\label{eq:delta-estimator}
    (\betahat,\deltahat_2,\ldots,\deltahat_C) \!=\!\argmin_{\hat{b},\hat h_2,\ldots,\hat h_C} \!\cbr{\!\frac{\norm{y_1\!-\!X_1\hat b}_2^2}{n}\! +\!\sum_{c\in\cA} \!w_c\!\frac{\norm{\tilde y_c\!-\!X_c^{\mathrm{eff}}(\hat b\!+\!\hat h_c)}_2^2}{n} \!+ \!\lambda_\beta\norm{\hat b}_2^2\! +\! \sum_{c\in\cA}\!\lambda_{\delta_c}\norm{\hat h_c}_2^2\!},
\end{equation}
where $\cbr{w_c}_{c=2}^C \eqcolon \mathbf{w}$ are the weights, $\lambda_\beta\geq 0$ the ridge corresponding to the target $\beta$, and $\cbr{\lambda_{\delta_c}}_{c=2}^C\eqcolon \boldsymbol{\lambda}_\delta\succeq 0$ the ridges corresponding to model shift. 

For $c\in \{1,\ldots,C\}$, we denote
the empirical covariance and cross-correlation terms as 
\begin{equation}\label{eq:emp_cov}
    \hat\Sigma_c\coloneqq (X_c^{\mathrm{eff}})^\top X_c^{\mathrm{eff}}/n,\qquad \vhat_c\coloneqq(X_c^{\mathrm{eff}})^\top\tilde y_c/n.
\end{equation}
Next, we define the deterministic block-selection matrices
$L_c\in\bbR^{d\times Cd}$ as
$L_1\theta=\beta$ and $L_c\theta=\beta+\delta_c$, for $c\in\cA$. Equivalently, $L_1=\matrix{I_d&0&\ldots&0}$, while $L_c$ has an $I_d$ in the
$\beta$ block and an $I_d$ in the $\delta_c$ block.  Hence, the
model for client $c$ is $$y_c=X_cL_c\theta+\varepsilon_c.$$ Since it will be useful to view the problem as a regression in $Cd$ dimensions, we introduce the joint estimator $\thetahat \coloneqq \matrix{\betahat^\top &\deltahat_2^\top& \hdots &\deltahat_C^\top}^\top$, as well as  empirical matrices and vectors
\[
    \Shat_c\coloneqq L_c^\top\hat\Sigma_cL_c \in \bbR^{Cd\times Cd}, \qquad \Uhat_c\coloneqq L_c^\top \vhat_c\in \bbR^{Cd}.
\]
Let $P_\beta$ be the projection onto the first block and
$P_{\delta_c}$ the projection onto the $\delta_c$ block.  Then, we can express the joint ridge term via
$\Lambda \coloneqq \lambda_\beta P_\beta + \sum_{c\in\cA}\lambda_{\delta_c}P_{\delta_c} = \operatorname{diag} \rbr{ \lambda_\beta I_d, \lambda_{\delta_2}I_d, \ldots, \lambda_{\delta_C}I_d}$.
With $\Shat\coloneqq\Shat_1+\sum_{c\in\cA}w_c\Shat_c,\ \Uhat\coloneqq\Uhat_1+\sum_{c\in\cA}w_c\Uhat_c,\ \Qhat\coloneqq(\Shat+\Lambda)^{-1}$,
we obtain 
\begin{equation}\label{eq:theta_closed_form}
    \hat\theta=\Qhat\Uhat.
\end{equation}

The formulation in \eqref{eq:delta-estimator} contains the main estimators considered throughout.
For example, setting $w_c=0$ corresponds to a regression setting without the auxiliary client $c$. Similarly, setting
$\lambda_{\delta_c}=0$ makes the contribution of client $c$ independent of
$b$ after minimizing over $h_c$, so that client $c$ does not affect
$\hat\beta$. On the other hand, sending $\lambda_{\delta_c}\to\infty$ forces
$\hat\delta_c\to0$ and pools client $c$ with the target model.  In particular,
if every $\lambda_{\delta_c}\to\infty$, then the estimator in \eqref{eq:delta-estimator} reduces to the weighted pooled ridge estimator considered in  \cite{song2024generalization,yang2025precise,rezaei2026high}, which ignores model shift and admits the closed-form solution 
\[
    \hat\beta \longrightarrow\rbr{\hat\Sigma_1+\sum_{c\in\cA}w_c\hat\Sigma_c + \lambda_\beta I_d}^{-1} \rbr{v_1+\sum_{c\in\cA}w_c\vhat_c}.
\]
Finally, ridgeless estimators are obtained by taking the corresponding
ridge parameters to zero.

\paragraph{Prediction risk.}
We evaluate all estimators on the distribution of the target client. If $x_1$ is an independent target covariate, then the conditional out-of-sample excess risk is
\begin{equation}
\label{eq:C-client-risk}
\begin{split}
    R(\hat\beta;\beta) \coloneqq
    \bbE\sbr{(x_1^\top\hat\beta-x_1^\top\beta)^2\,\middle|\,\Xeff,\theta}
    =\bbE\sbr{\norm{\thetahat-\theta}_{S_1}^2 \,\middle|\, \Xeff,\theta}= B(\thetahat;\theta)+V(\thetahat;\theta),
\end{split}
\end{equation}
with $\Xeff\!=\!(\Xeff_1\ldots\Xeff_C)$ and $B(\thetahat;\theta)\!:= \!\norm{\bbE[\thetahat \mid \Xeff,\theta] - \theta}_{S_1}^2$, $V(\thetahat;\theta)\! := \!\Tr[\mathrm{Cov}(\thetahat\mid\Xeff,\theta) S_1]$ denoting bias and variance of the risk. The expectation is taken over the noise in the data and the private noise. 
Define the augmented population covariance $
S_c\coloneqq L_c^\top\Sigma_cL_c\in\bbR^{Cd\times Cd}$,
and the total noise level of the privatized client $\stc^2 \coloneqq \sigma_c^2 + \sydpc^2$. Throughout the rest of the analysis, $B$, $V$, and $R$
refer to the estimator with clipping omitted. Then, under label privatization,
the following exact expressions (proved in Appendix \ref{app:aux-bias-variance}) hold
\begin{equation}
\label{eq:biasvarclosed}
\begin{split}
    B(\hat\theta;\theta) =\theta^\top\Lambda\Qhat S_1\Qhat\Lambda\theta, \qquad V(\hat\theta;\theta)=\frac1n\Tr\sbr{S_1\Qhat\rbr{\sigma_1^2\Shat_1+\sum_{c\in\cA}w_c^2\stc^2\Shat_c}\Qhat}.
\end{split}
\end{equation}

\vspace{-1em}

\section{Theoretical results}
\label{sec:theory}

\vspace{-.25em}

We derive asymptotic characterizations of the bias and variance in \eqref{eq:biasvarclosed}, under both label privatization (Section \ref{sec:labeldpresult}) and joint feature and label privatization (Section \ref{subsec:xy-private}), and we characterize the optimal tuning parameters of the estimator in \eqref{eq:delta-estimator}. We start with an additional assumption \myassump{ass:indep_signal}: 
\begin{equation}
\label{eq:C-client-random-signal}
    \beta=\frac{r_\beta}{\sqrt d}\,u_{\beta}, \qquad \delta_c=\frac{r_{\delta_c}}{\sqrt d}\,u_{\delta_c},
    \quad c\in\cA,
\end{equation}
where $u_\beta,\ u_{\delta_2}, \ldots, u_{\delta_C}$ have independent, centered, unit-variance coordinates with a
uniformly bounded $8$-th moment and $r_\beta, r_{\delta_c}>0$ are constants independent of $n, d$.  
This represents an average-case setting in which no feature direction is singled out a priori, and it yields a characterization depending on the magnitudes of target signal and model shifts, rather than their alignment. Let us further define $\tau_\beta\coloneqq nr_\beta^2/d,\ \tau_{\delta_c}\coloneqq nr_{\delta_c}^2/d$. We will also assume that there exist constants $w_{-},w_{+}$ such that $0<w_{-}<w_c<w_{+}<+\infty$ for all $c\in \cA$. 


\vspace{-.25em}

\subsection{Label privatization}
\label{sec:labeldpresult}

\vspace{-.25em}

\paragraph{General deterministic equivalent.} We start by privatizing only labels via \eqref{eq:C-client-label-private}. We introduce the total noise terms $s_1^2\coloneqq\sigma_1^2$ and $s_c^2 \coloneqq \sigma_c^2+\sydpc^2$, $c\in \cA$. 
Without loss of generality, we consider $w_1 =1$ and define $\boldsymbol{\pi} = (\pi_c)_{c=1}^C$ and $\mathbf{\Sigma}\coloneqq 
\{\Sigma_c\}_{c=1}^C$. Let $\bar Q=\bar Q(\mathbf w,\lambda_\beta,
\boldsymbol\lambda_\delta,\boldsymbol{\pi},\mathbf{\Sigma})$ be the unique positive-definite solution of 
\begin{equation}
\label{eq:delta-general-block-de}
    \bar Q = \sbr{\Lambda + \sum_{c=1}^C\frac{\pi_cw_cS_c}{1+e_c}}^{-1},
\end{equation}
where $e_c = e_c (\mathbf w,\lambda_\beta,
\boldsymbol\lambda_\delta,\boldsymbol{\pi},\mathbf{\Sigma})$ are the $C$ fixed point coefficients satisfying
\begin{equation}\label{eq:delta-general-fixed-point}
    e_c=\frac{w_c}{n}\Tr(S_c\bar Q), \qquad
    c\in\{1,\ldots,C\}.
\end{equation}
Moreover, we define the matrix $J\in \bbR^{C\times C}$, the vectors $b,g\in\R^C$ and the matrix $\cK\in \bbR^{Cd\times Cd}$ as
\begin{equation}\label{eq:J_b_g_def}
    J_{ck} =\frac{\pi_kw_cw_k}{n(1+e_k)^2}\Tr(S_c\bar{Q} S_k\bar{Q}),\quad b_c = \frac{w_c}{n}\Tr(S_c\bar{Q} S_1\bar{Q}),\ c,k\in \cC,\quad g = (I_C - J)^{-1}b.
\end{equation}
\begin{equation}\label{eq:ck_def}
     \cK=\bar{Q}\sbr{S_1+\sum_{c=1}^C\frac{\pi_cw_cg_c}{(1+e_c)^2}S_c}\bar{Q}.
\end{equation}
%
At this point, we can state our result which establishes a deterministic equivalent of the bias and variance terms, depending on signal parameters $\tau_\beta,\tau_{\delta_c}$, alignment of $\beta_c$'s weighted by the covariance, sample complexities $\boldsymbol{\pi}$, covariance structures $\mathbf{\Sigma}$, intrinsic noise levels $\{\sigma_c^2\}_{c=1}^C$, weights $\{w_c\}_{c=2}^C$, ridge parameters $\lambda_\beta,\ \{\lambda_{\delta_c}\}_{c=2}^C$, 
and privacy noise intensities $\{\sydpc^2\}_{c=2}^C$. 

\begin{theorem}[Deterministic equivalent for label privatization]
\label{thm:main-det-eq}
Under the Assumptions (\ref{ass:C-client-asymptotics}) - (\ref{ass:priv_budg_large}), 
it holds, almost surely, that 
\begin{equation}\label{eq:main_conv}
    \lim_{n\to \infty}\abs{B(\thetahat,\theta)-\cB} = 0, \qquad  \lim_{n\to \infty}\abs{V(\thetahat,\theta)-\cV} = 0,
\end{equation}
where 
\vspace{-.3em}
\begin{equation}\label{eq:C-client-bias-variance-DE}
    \cB =\theta^\top \Lambda \cK \Lambda \theta,\qquad  \cV = \sum_{c=1}^C \frac{\pi_cw_cg_cs_c^2}{(1+e_c)^2}.
\end{equation}
If we additionally assume that \eqref{ass:indep_signal} holds, then \eqref{eq:main_conv} holds with $\cV$ as in \eqref{eq:C-client-bias-variance-DE} and 
\vspace{-.3em}
\begin{equation}\label{eq:bias_simpler}
    \cB = \frac{\tau_\beta\lambda_\beta^2}{n} \Tr(P_\beta \cK)+\sum_{c\in\cA} \frac{\tau_{\delta_c}\lambda_{\delta_c}^2}{n}\Tr(P_{\delta_c}\cK).
\end{equation}
\end{theorem}

As an immediate corollary, $R(\hat\beta;\beta)$ converges almost surely to $\cB+\cV$, and it does so almost surely even without averaging over observation and privacy noises (Appendix \ref{app:realized-risk}). Thus, although the estimator $\betahat$ depends on random, high-dimensional samples, its asymptotic risk is characterized by the deterministic quantities $\cK$ and $g$. In  Figure~\ref{fig:synthetic}, we show that the deterministic equivalent closely approximates the empirical risk at moderate dimension. The plots further highlight how the risk can be substantially influenced by the sources of heterogeneity, such as \textbf{(a)} covariance shift, \textbf{(b)} model shift, and \textbf{(c)} privacy budget. We next sketch the argument deferring details to Appendix  \ref{app:proof-main-det-eq}.

\emph{Idea of the proof.} The argument relies on reducing both bias and variance to the scalar resolvent functional $F=n^{-1}\Tr(S_1\widehat Q)$ and its derivatives. 
The deterministic equivalent for resolvents of random matrices with heterogeneous covariance structure, such as $\Qhat$, is well studied and an application of \citep[Theorem~1]{wagner2012large} gives that $\abs{F-e_1} \to 0$, for $e_1$ defined by \eqref{eq:delta-general-block-de}-\eqref{eq:delta-general-fixed-point}. However, to be able to apply such result, one first has to introduce an additional small positive regularization, which can be later removed using the uniform lower spectral bound on sample covariance matrices in the underparameterized regime. Then, uniform bounds on the second derivatives of $F$ and $e_1$ allow the convergence to be transferred to their first derivatives. We finally lift assumption \eqref{ass:indep_signal} and using \citep[Theorem 1.6]{louart2026resolvent}, we get the directional concentration needed to show \eqref{eq:C-client-bias-variance-DE}. \qed

\begin{remark}
Theorem \ref{thm:main-det-eq} is related to the high-dimensional multi-task linear regression analysis of \citet{ilbert2024analysing}, who also consider a shared signal with task-specific deviations under heterogeneous feature covariances. The settings, though, differ in several respects. In particular, the assumptions on the features are different: whereas \citet{ilbert2024analysing} assume Lipschitz concentration of the features, we work under the bounded moment condition (\ref{ass:z-moment-bounds}). Most importantly, rather than studying multi-task prediction, we have a single target client and account for auxiliary weights and privacy mechanisms, which are central to the dataset selection problem we consider.
\end{remark}

\begin{figure}[t]
        \centering
        \includegraphics[width=\textwidth]{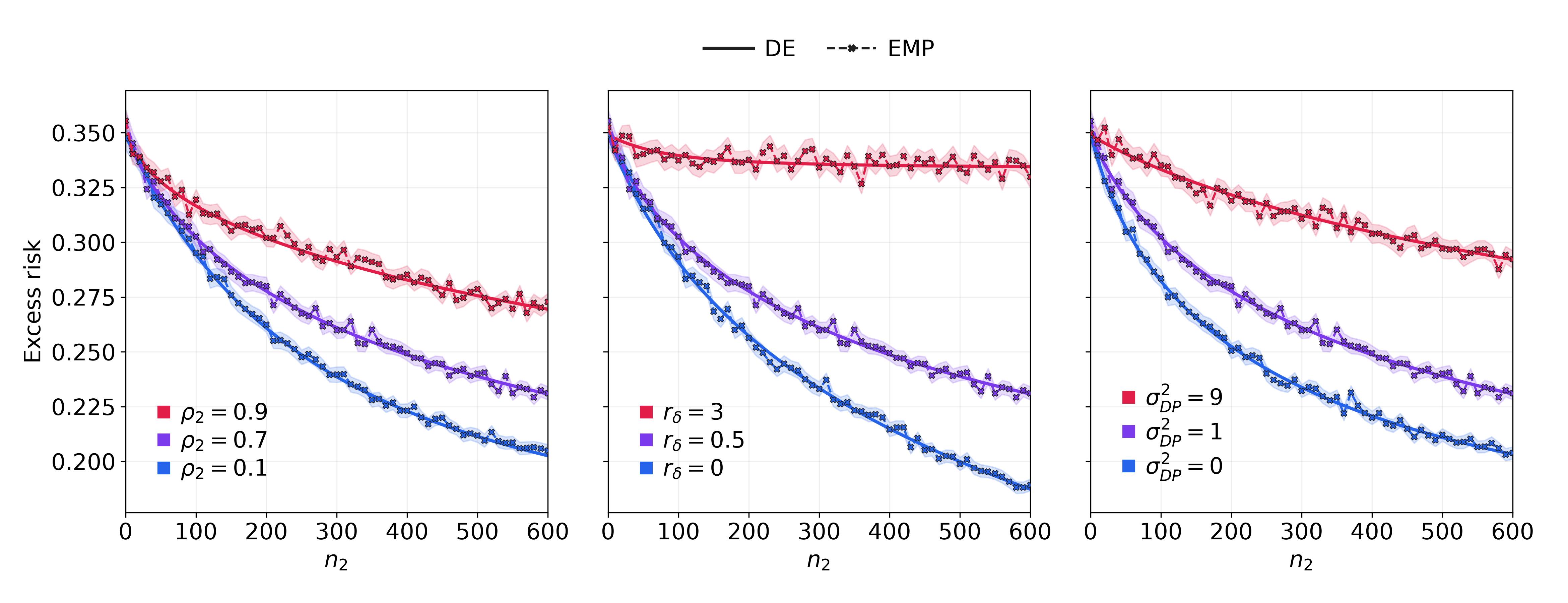}
        \makebox[\textwidth]{%
        \parbox{0.333\textwidth}{\centering\textbf{(a)} Varying covariance shift}%
        \parbox{0.333\textwidth}{\centering\textbf{(b)} Varying model shift}%
        \parbox{0.333\textwidth}{\centering\textbf{(c)} Varying privacy noise}%
        }
        \caption{
        Dependence of target excess risk on covariance shift, model shift and privacy noise in a synthetic linear regression experiment with two clients. We take $x_1 \sim \cN(0,\Sigma_1)$ and
        $x_2 \sim \cN(0,\Sigma_2)$, with $\Sigma_1$ and $\Sigma_2$ of the form 
        $[K(\rho)]_{ij}=\rho^{|i-j|}$.
        We set $\Sigma_1=K(\rho_1)$ with $\rho_1=0.1$ and
        $\Sigma_2=a_2K(\rho_2)$, so that
        $\Tr(\Sigma_2\Sigma_1^{-1})=d$.
        We take $d=80$, $n_1=200$ and plot $\norm{\betahat - \beta}_{\Sigma_1}^2$ as a function of $n_2$.
        Solid lines show the deterministic equivalent prediction and dashed lines with
        crosses show the empirical mean over $400$ i.i.d. repetitions, with
        shaded $90\%$ confidence intervals. 
        \textbf{(a)} Covariance shift: $\rho_2\in\{0.1,0.7,0.9\}$, with model shift
        $r_\delta=0.5$ and privacy noise variance $\sigma_{\mathrm{DP}}^2=1$.
        \textbf{(b)} Model shift: $r_\delta\in\{0,0.5,3\}$, 
        $\rho_2=0.7$ and $\sigma_{\mathrm{DP}}^2=1$.
        \textbf{(c)} Privacy noise:
        $\sigma_{\mathrm{DP}}^2\in\{0,1,9\}$, with $\rho_2=0.7$ and
        $r_\delta=0.5$.
        The middle curve is the same 
        in all three panels.
        }\label{fig:synthetic}
        \vspace{-1em}
\end{figure}


\vspace{-.4em}

\paragraph{Bayes-optimal tuning.}
If we impose further restrictions on the distribution of $\beta$ and $\delta$, then
we can show optimality of our estimator and characterize the parameters $(\mathbf w,\lambda_\beta,\boldsymbol\lambda_\delta)$ which achieve it. 

\begin{proposition}[Bayes-optimal privacy weights and shift regularization]
\label{prop:bayes-opt-param}
Consider the setting of Theorem~\ref{thm:main-det-eq} with the unclipped release $y_c+\xi_c$ for $c\in\mathcal A$ and Gaussian noise $\{\varepsilon_j\}_{j=1}^C$. Assume $\sigma_1^2>0$, $\sigma_{t,c}^2>0$
for $c\in\cA$ and 
\vspace{-.3em}
\begin{equation}\label{eq:ass_gaus_prop}
    \beta\sim\cN\left(0,\frac{r_\beta^2}{d}I_d\right), \qquad \delta_c\sim\cN\left(0,\frac{r_{\delta_c}^2}{d}I_d\right),\ c\in\cA,
\end{equation}
independently. Then, conditional on the designs, the posterior mean of the stacked parameter, i.e.\\ $\bbE\sbr{\theta \,\middle|\, X_1,\ldots,X_C,y_1,\tilde y_2,\ldots,\tilde y_C}$, is exactly the estimator in \eqref{eq:delta-estimator} with
\vspace{-.3em}
\begin{equation}\label{eq:Bayespar}
    w_c^\circ=\frac{\sigma_1^2}{\stc^2}, \qquad \lambda_\beta^\circ =\frac{\sigma_1^2}{\tau_\beta} =\frac{\sigma_1^2d}{nr_\beta^2},
    \qquad \lambda_{\delta_c}^\circ =\frac{\sigma_1^2}{\tau_{\delta_c}} =\frac{\sigma_1^2d}{nr_{\delta_c}^2}.
\end{equation}
This estimator minimizes $\mathbb E_\theta[R(\hat\beta;\beta)]$ 
among all measurable estimators.
\end{proposition}

In words, under Gaussian priors on target signal and model shifts, the weighted joint ridge estimator with parameters in \eqref{eq:Bayespar} is Bayes-optimal. 
The argument is deferred to Appendix~\ref{app:proof-bayes-opt-param}.


\subsection{Feature and label privatization}
\label{subsec:xy-private}

We next consider the case in which both features and labels are made private by the Gaussian mechanism in \eqref{eq:C-client-feature-label-private}. To analyze this case, we assume Gaussian noise $\{\varepsilon_j\}_{j=1}^C$ and strengthen Assumption (\ref{ass:z-moment-bounds}) by assuming \myassump{ass:features_gaussian} that the rows of $X_c$ are i.i.d.\ Gaussian with covariance $\Sigma_c$, i.e., $x_{ci}\sim\cN(0,\Sigma_c)$ for $c\in \cA$ and $i\in \{1, \ldots, n_c\}$.
The reduction to an equivalent regression problem can then be performed separately
for each auxiliary client. In particular, for $c\in \cA$ and $i\in \{1, \ldots, n_c\}$, we write
\begin{equation*}
    \tilde x_{ci}=x_{ci}+\xi_{X,ci}, \qquad \xi_{X,ci}\sim \cN(0,\sigma_{X,\mathrm{DP},c}^2I_d).
\end{equation*}
Gaussian conditioning gives
 $   x_{ci}\mid\tilde x_{ci}
    \sim
    \cN(K_c\tilde x_{ci},T_c)$,
where $K_c \coloneqq \Sigma_c(\Sigma_c+\sigma_{X,\mathrm{DP},c}^2I_d)^{-1}$ and $T_c \coloneqq \Sigma_c -\Sigma_c (\Sigma_c+\sigma_{X,\mathrm{DP},c}^2I_d)^{-1} \Sigma_c$. We then define the conditionally denoised design as
\[
   \Xeff_c = \tilde X_cK_c.
\]
Writing $X_c=\Xeff_c+E_c$, the rows of $E_c$ are independent
$\cN(0,T_c)$ vectors, which are also independent of $\Xeff$. Since
$y_c=X_c(\beta+\delta_c)+\varepsilon_c$, the privatized response satisfies
\[
    \tilde y_c =\Xeff_c(\beta+\delta_c)+\xieff_c, \qquad \xieff_c \coloneqq E_c(\beta+\delta_c)+\varepsilon_c+\xi_c.
\]
Moreover, conditional on $(\beta,\delta_c)$, it follows that
    $\xieff_c \sim\cN(0,\nu_c^2 I_{n_c})$ is independent of $\Xeff_c$, where $\nu_c^2=  \sigma_c^2+\sigma_{y,\mathrm{DP},c}^2 + (\beta+\delta_c)^\top T_c(\beta+\delta_c)$. The denoised population covariance is
\[
    \Sigma_c^{\mathrm{eff}} \coloneqq \operatorname{Cov}(x_c^{\mathrm{eff}}) = \Sigma_c (\Sigma_c+\sigma_{X,\mathrm{DP},c}^2I_d)^{-1} \Sigma_c.
\]
Thus, feature privatization has two effects: $\Sigma_c\longmapsto\Sigma_c^{\mathrm{eff}}$ and $\sigma_c^2+\sigma_{y,\mathrm{DP},c}^2 \longmapsto \nu_c^2$,
which, respectively, add covariance shift and increase effective noise. 

We now proceed to use the estimator in \eqref{eq:delta-estimator} with
$X_c^{\mathrm{eff}}$ as defined above. Let $g^{\mathrm{eff}}$, $\cK^{\mathrm{eff}}$, and $e_c^{\mathrm{eff}}$ denote the quantities $g$, $\cK$, and $e_c$ defined in \eqref{eq:delta-general-block-de}-\eqref{eq:ck_def} with $\mathbf{\Sigma}$ replaced by
$\mathbf{\Sigma}^{\mathrm{eff}}\coloneqq\{\Sigma_c^{\mathrm{eff}}\}_{c=1}^C$. Moreover, note that under the Assumption \eqref{ass:indep_signal}, it holds
\begin{equation}
\label{eq:C-client-nu-concentration}
    \lim_{n\to \infty} \max_{c\in\cA} \abs{\nu_c^2-\bar\nu_c^2}= 0, \qquad
    \bar\nu_c^2
    \coloneqq
    \sigma_c^2+\sigma_{y,\mathrm{DP},c}^2
    +\frac{r_\beta^2+r_{\delta_c}^2}{d}\Tr(T_c),
\end{equation}
almost surely. Thus, an application of Theorem \ref{thm:main-det-eq} gives the corollary below.

\begin{corollary}[Deterministic equivalent for label and feature privatization]
\label{cor:C-client-xy-private}
Under the Assumptions (\ref{ass:C-client-asymptotics}) - (\ref{ass:features_gaussian}), 
it holds, almost surely, that
\[
    \lim_{n\to\infty}\abs{R(\hat\beta;\beta)-\cR}=0,
\]
where $\cR = \theta^\top \Lambda \cK^{\mathrm{eff}} \Lambda \theta + \sum_{c=1}^C (\pi_cw_cg_c^{\mathrm{eff}}\bar\nu_c^2)/(1+e_c^{\mathrm{eff}})^2$.
\end{corollary}

\begin{figure}[t]
    \centering
    \includegraphics[width=0.32\textwidth]{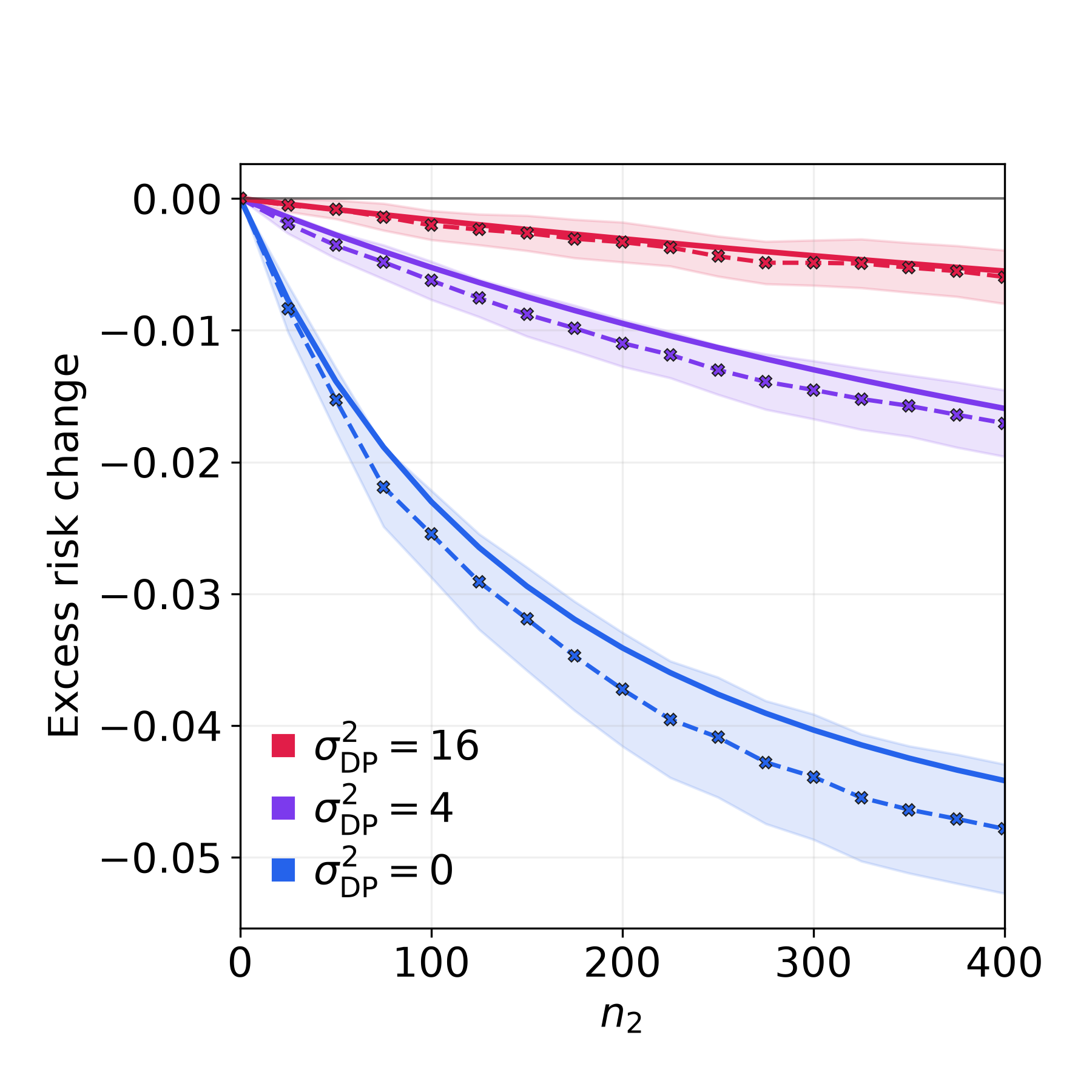}
    \hfill
    \includegraphics[width=0.32\textwidth]{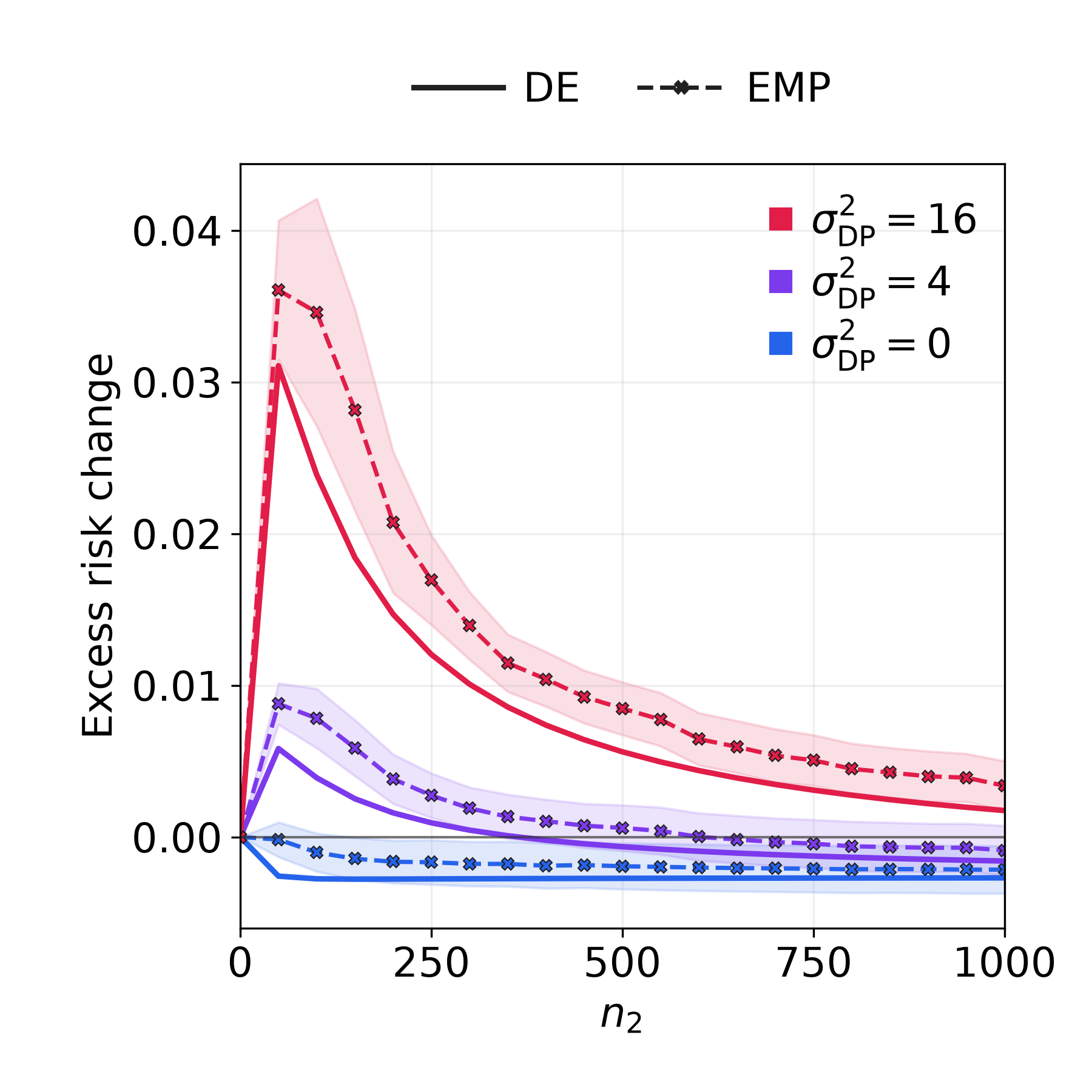}
    \hfill
    \includegraphics[width=0.32\textwidth]{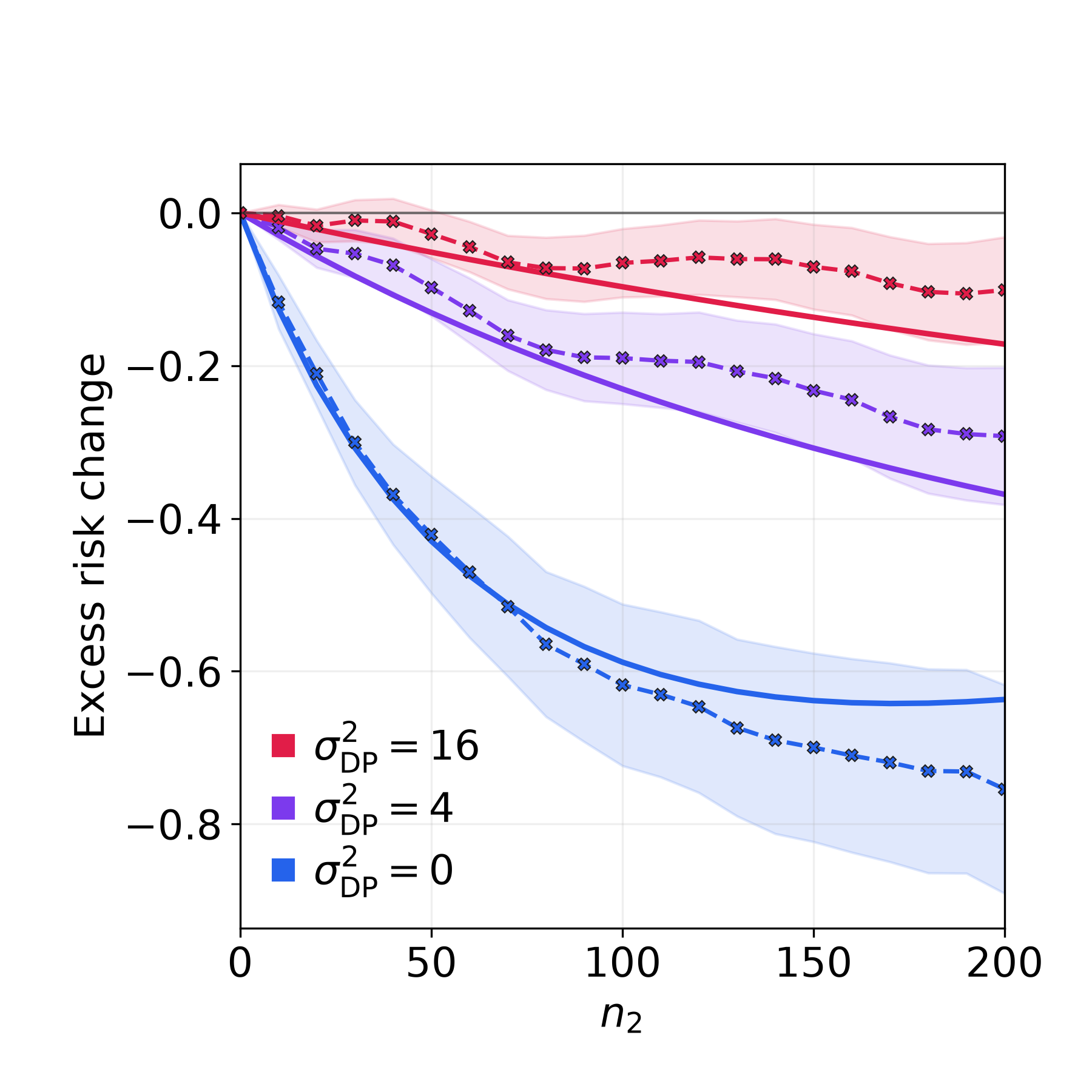}
    \makebox[\textwidth]{%
        \parbox{0.333\textwidth}{\centering\textbf{(a)} LSAC}%
        \parbox{0.333\textwidth}{\centering\textbf{(b)} Lianjia}%
        \parbox{0.333\textwidth}{\centering\textbf{(c)} Student Performance}%
        }

\caption{ Change in target risk relative to the target-only baseline as auxiliary data are added. \textbf{(a) LSAC:} Tier 3 is the target client and Tier 4 the auxiliary client. \textbf{(b) Lianjia:} District 8 is the target client and District 7 the auxiliary client. \textbf{(c) Student Performance:} Mousinho da Silveira is the target school and Gabriel Pereira the auxiliary school. }
\vspace{-1em}
    \label{fig:exact_fit}
\end{figure}

\vspace{-.75em}

\section{Experimental results} \label{sec:expe}

\vspace{-1mm}
In this section, we evaluate our theoretical predictions on three real-world regression datasets. We first assess how accurately the deterministic equivalents \eqref{eq:C-client-bias-variance-DE}-\eqref{eq:bias_simpler} predict changes in target risk as auxiliary sample size and privacy noise vary (Figure \ref{fig:exact_fit}). We then use these predictions to select beneficial auxiliary clients and compare against several heuristic baselines (Figure \ref{fig:source_selection}).
\vspace{-3mm}
\paragraph{Datasets.} We consider three publicly available regression datasets: \emph{(i)} in \emph{Student Performance}\footnote{\url{https://archive.ics.uci.edu/dataset/320/student+performance}} \citep{cortez2008using}, we predict the final Portuguese-course grade G3 and treat the two schools as clients; \emph{(ii)} in \emph{LSAC}\footnote {\url{https://www.kaggle.com/datasets/danofer/law-school-admissions-bar-passage}}\citep{wightman1998lsac}, we predict standardized first-year GPA and treat law-school tiers as clients; \emph{(iii)} in \emph{Lianjia}\footnote{\url{https://www.kaggle.com/datasets/ruiqurm/lianjia}}, we predict log housing transaction prices in Beijing and treat city districts as clients. Dataset preprocessing and additional statistics are given in Appendix \ref{app:experimental-details}.
\vspace{-3mm}
\paragraph{Experimental setup.} In each experiment, one client is designated as the target and the remaining clients as auxiliary with privacy requirements. We run $100$ evaluation repetitions. In each repetition, the target training size is fixed at $n_1=100$, approximately $20\%$ of the target observations are used for validation, and the remainder are used for testing. Auxiliary observations are sampled without replacement, with $n_2$ ranging up to $200$, $400$, and $1000$ for Student Performance, LSAC, and Lianjia, respectively. Following \eqref{eq:C-client-label-private}, independent Gaussian noise is added to auxiliary labels while covariates remain unchanged. Estimation of the distributional quantities entering the deterministic equivalent and the hyperparameter tuning procedures are described in Appendix~\ref{app:experimental-details}.
\vspace{-3mm}
\paragraph{Prediction of target risk.}
Figure \ref{fig:exact_fit} shows that the deterministic equivalent prediction accurately tracks empirical performance, as the auxiliary sample size $n_2$ and the privacy noise level $\sigma_{\mathrm{DP}}^2$ vary. We consider $\sigma_{\mathrm{DP}}^2\in\{0,4,16\}$ and plot the change in target risk relative to the target-only baseline (i.e., the difference between using $n_2$ client samples and not using client samples at all), as a function of $n_2$. Negative values indicate an improvement over the baseline using only target data.  Dashed lines with crosses are empirical averages, with shaded regions denoting $90\%$ confidence intervals for the empirical mean. Solid lines in plots show the simplified deterministic equivalent predictions \eqref{eq:bias_simpler}, illustrating that a more tractable expression gives accurate predictions in this setting. Similar results are obtained in a semi-synthetic setting in which a linear model is planted on the real-world design matrix, and we use the general deterministic equivalent predictions \eqref{eq:C-client-bias-variance-DE}, see Appendix \ref{app:semi_synth}.
   
\vspace{-3mm}
\paragraph{Dataset selection.}
\begin{figure}[t]
    \centering
    \includegraphics[width=0.32\textwidth]{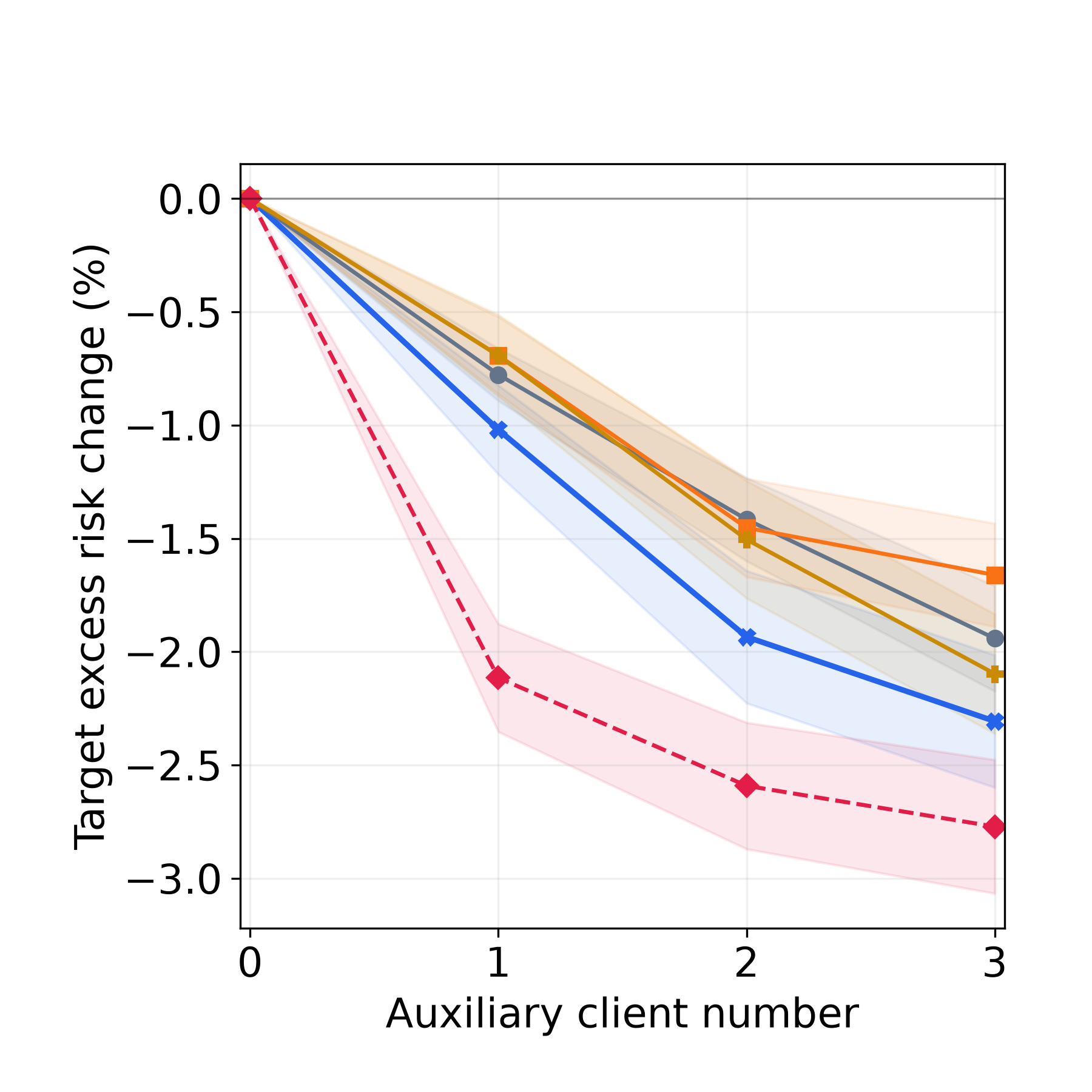}
    \hfill
    \includegraphics[width=0.32\textwidth]{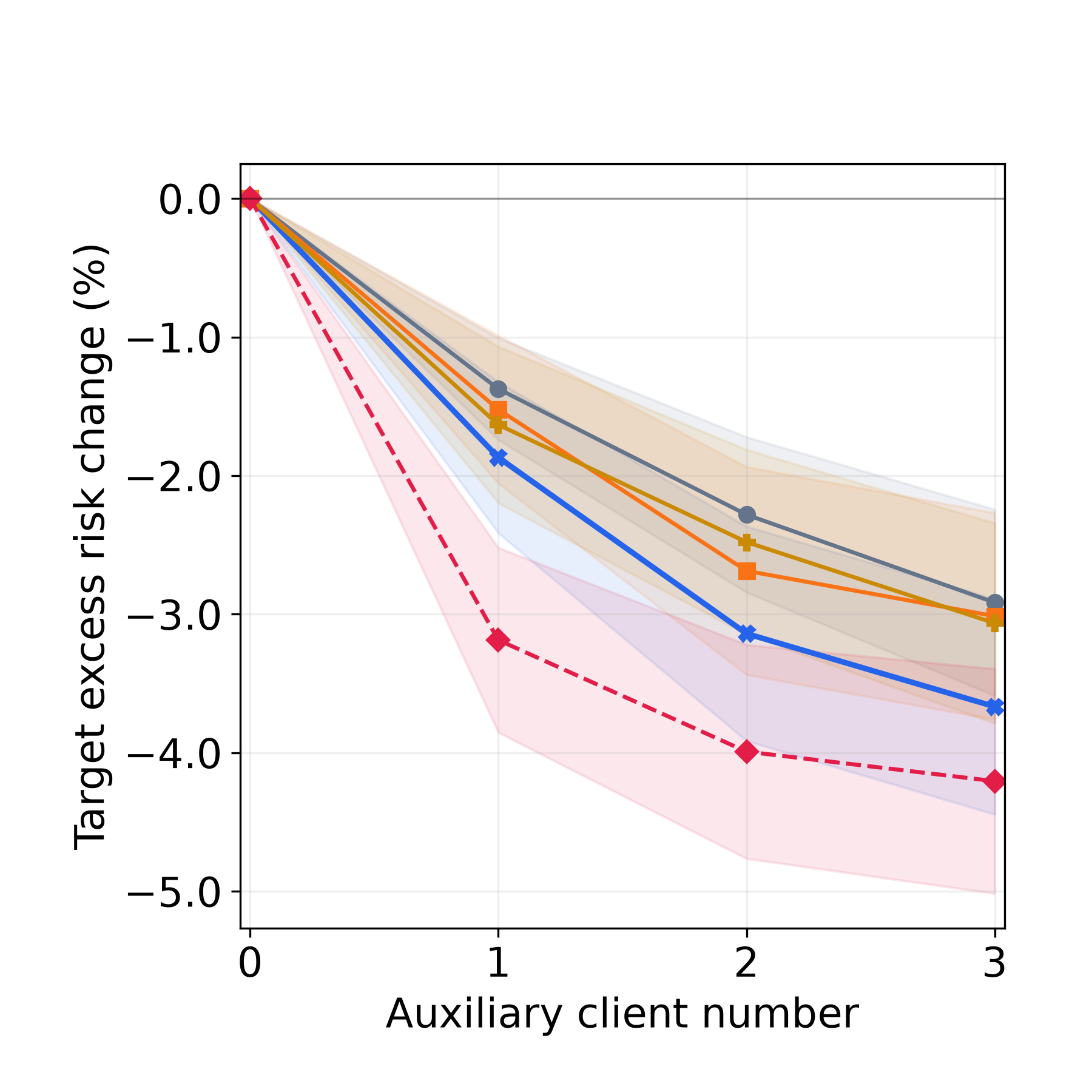}
    \hfill
    \includegraphics[width=0.32\textwidth]{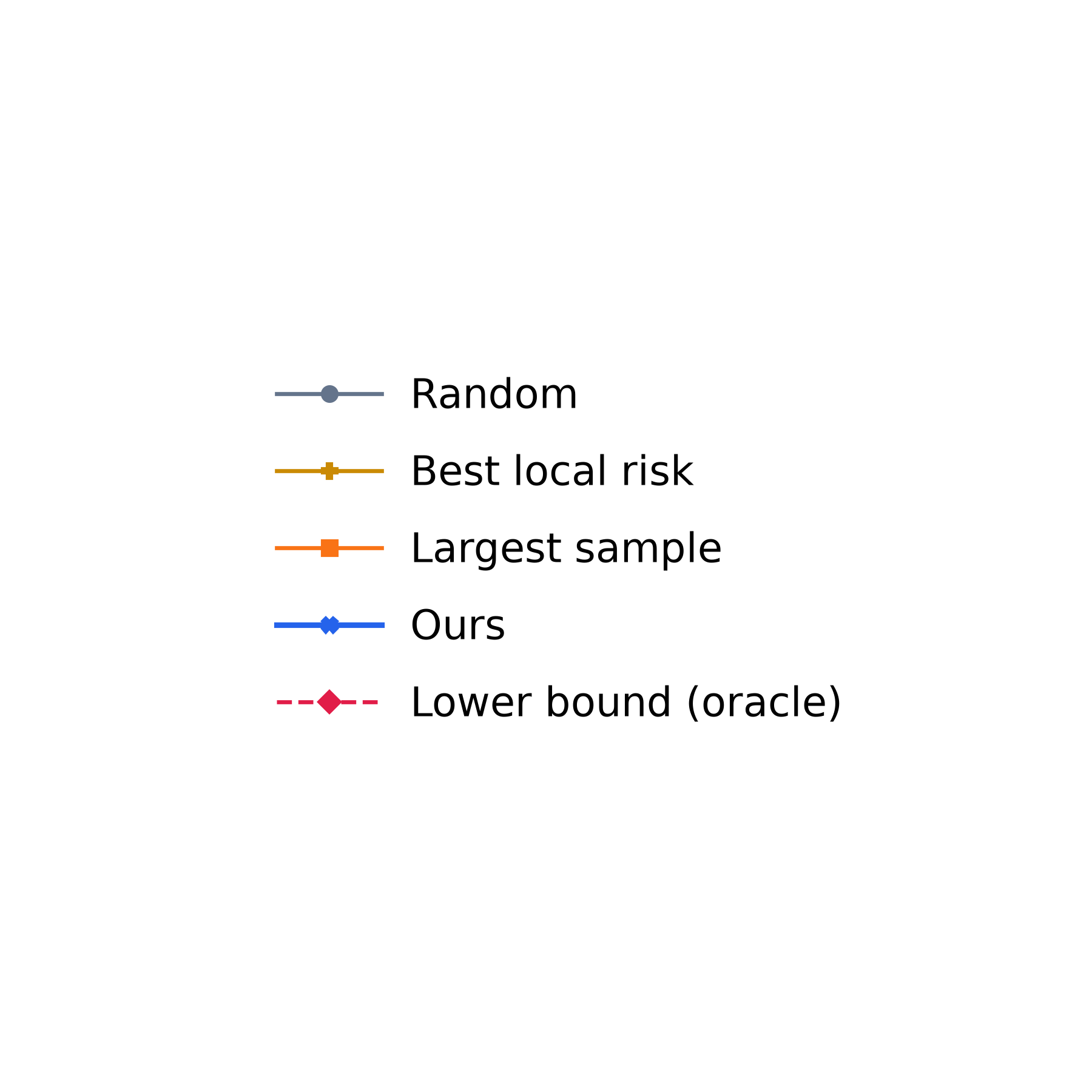}
    \makebox[\textwidth]{%
        \parbox{0.33\textwidth}{\centering\textbf{(a)} LSAC}%
        \parbox{0.33\textwidth}{\centering\textbf{(b)} Lianjia}%
        \parbox{0.33\textwidth}{\centering $\ $}%
        }
    \caption{ Percentage change in target test risk relative to the target-only baseline under sequential auxiliary client selection. \textbf{(a) LSAC:} Tier 3 is the target client and Tiers 1, 2, 4, 5, and 6 are candidate auxiliary clients. \textbf{(b) Lianjia:} District 8 is the target client and Districts 1, 2, 4, 6, 7, 10, 11 and 13 are candidate auxiliary clients. }
\vspace{-1em}
    \label{fig:source_selection}
\end{figure}

Figure \ref{fig:source_selection} demonstrates that the deterministic equivalent effectively guides auxiliary client selection in the multi-client LSAC and Lianjia datasets. Student Performance contains only two clients and is therefore excluded from this evaluation. Starting from the target-only estimator, each method sequentially selects auxiliary clients. All auxiliary clients have a fixed number of samples and receive an independent Gaussian privacy noise with variance $\sigma_{\mathrm{DP}}^2=4$. After selecting $k$ clients, we report the improvement over the baseline using only target data, expressed in percentage. As before, negative values indicate improvement over the target-only baseline.

The method \emph{Ours} greedily adds the auxiliary client that minimizes the deterministic equivalent prediction for the target risk from Theorem \ref{thm:main-det-eq} using hyperparameters from Proposition \ref{prop:bayes-opt-param}. Our approach employs only population-level quantities that do not require access to client data and could be estimated from a publicly available dataset. This is quite common, e.g., when dealing with data from  US Census~\citep{USCensus, Abowd2022} or UK Biobank \citep{Sudlow2015}. In our context, we assume the availability of a non-private dataset used to estimate the population quantities $\{\Sigma_c,\sigma_c\}_{c=1}^C$, together with either the relevant inner products between $\{\beta_c\}_{c=1}^C$ or $\tau_{\beta}$ and $\{\tau_{\delta_c}\}_{c=2}^C$, which are plugged into \eqref{eq:C-client-bias-variance-DE} or \eqref{eq:bias_simpler}, respectively. More precisely,  the non-private dataset consists of the remaining data not used for the experiment of Figure \ref{fig:source_selection}. Plots \textbf{(a)} and \textbf{(b)} use the deterministic equivalent \eqref{eq:C-client-bias-variance-DE}, although the simplified version \eqref{eq:bias_simpler} also improves LSAC client selection. 

We compare against three heuristic baselines: \emph{Random}, which selects clients uniformly at random; \emph{Largest sample}, which orders clients by decreasing size of their training set; and \emph{Best local risk}, which ranks clients by performance on their own local distributions. Finally, we report \emph{Lower bound (oracle)}, which at each step greedily adds the client that yields the lowest empirical target risk on the test set. Because it uses test data, this baseline cannot be achieved in practice and serves only as a reference. We report empirical averages (squares) and $90\%$ confidence intervals (shaded regions).
Figure \ref{fig:source_selection} shows that our approach outperforms the three baselines, and similar results for different privacy noise levels are reported in Appendix \ref{app:dataset_selection}. 

\section{Concluding remarks}\label{sec:conc}

In this work, we propose a principled way to decide which datasets will provide value for ridge regression, under differential privacy (DP) constraints. Our setting accommodates covariate and model shifts and provides useful predictions in experiments on real datasets, even though the underlying models are likely to be more complex than a simple linear model. Our results show that deterministic equivalents of the risk are accurate enough to guide dataset selection and identify the Bayes-optimal way to combine sources under our model assumptions.

An interesting future direction involves the analysis of nonlinear models. A first step would be to consider clipping of features and labels, which makes the models in \eqref{eq:C-client-label-private}-\eqref{eq:C-client-feature-label-private} nonlinear. In fact, the importance of clipping is well reported in practice~\citep{De2022UnlockingHD} and recent theoretical work has started to analyze the impact of gradient clipping in high-dimensional regression \citep{bombari2026highdimensionalprivatelinearregression,dworkhighd}. A natural next step would be then to consider multi-index models, for instance when early feature learning is followed by ridge regression \citep{ba2022highdimensionalasymptoticsfeaturelearning,pmlr-v235-cui24d}. In fact, combining linear probing with feature adaptation is an established approach to transfer learning \citep{pmlr-v235-collins24a,pmlr-v267-tahir25a,kumar2022finetuning}. Finally, more sophisticated privacy primitives, e.g., the private covariance estimators developed by \cite{amin2019differentially,kamath2019privately}, could be used for both data selection and privatization in order to improve the privacy-utility tradeoff.

\section*{Acknowledgments}

F.\ K.\ and M.\ M.\ are funded by the European Union (ERC, INF$^2$, project number 101161364). Views and opinions expressed are however those of the author(s) only and do not necessarily reflect those of the European Union or the European Research Council Executive Agency. Neither the European Union nor the granting authority can be held responsible for them. E. C. was supported by the National Research Agency under France 2030, reference “ANR-23-IACL-0008”. S.S.M. was supported by the Wallenberg AI, Autonomous Systems, and Software Program (WASP).
The authors thank Nikita Kalinin for helpful discussions and reference pointers. 





\bibliographystyle{plainnat}
\bibliography{references}

\appendix

\newpage

\section{Technical preliminaries} \label{app:auxiliary}

In this appendix we provide the privacy and clipping details from
Section~\ref{sec:preliminaries}, and derive the bias and variance formulas
in \eqref{eq:biasvarclosed}. 

\subsection{Sufficient conditions for inactive clipping}\label{app:aux-clipping}

First, let us recall that
\begin{equation}\label{eq:priv_var_def}
    \sydpc^2 = \frac{\min(\Delta_{y,c}^2, K_{y, c}^2)}{2\rho_{y,c}},\quad \sxdpc^2 = \frac{\min(\Delta_{X,c}^2,dK_{X,c}^2)}{2\rho_{X,c}},
\end{equation}
which under Assumption (\ref{ass:noise_const}) we can treat as fixed constants independent of $n$ and $d$.

Let $\cE_n$ denote the event that at least one auxiliary observation is clipped:
\[
 \cE_n = \bigcup_{c\in\mathcal A} \cbr{\max_{i\leq n_c}|y_{c,i}|>K_{y,c}/2 \ \text{or}\ \max_{i\leq n_c}\|X_{c,i}\|_2>\sqrt{d}K_{X,c}/2}.
\]
We give sufficient conditions under which $\sum_n\mathbb P(\mathcal E_n)<\infty$. The first Borel-Cantelli lemma then implies that, almost surely, for all sufficiently large $n$,
\begin{equation}
 \mathrm{Clip}_{K_{y,c}}(y_c)=y_c, \qquad \mathrm{Clip}_{K_{X,c}}(X_c)=X_c, \qquad c\in\mathcal A.
 \label{eq:aux-clipping-inactive}
\end{equation}
Note that for label privatization alone, the feature-clipping condition can simply be omitted.

\paragraph{Uniformly bounded samples.}
Suppose there are deterministic constants $B_{y,c},B_{X,c}>0$, independent of $n$ and $d$, such that
\[
 \abs{y_{c,i}}\leq B_{y,c}/2,\qquad \abs{X_{c,ij}}\leq B_{X,c}/2
\]
almost surely for every $i,j,n,d$. Then $\norm{x_{c,i}}_2\leq\sqrt{d}\,B_{X,c}/2$, so the choices
\[
 K_{y,c}=B_{y,c}, \qquad K_{X,c}=B_{X,c}
\]
make clipping inactive for every sample size. For fixed positive privacy noise variances,
\eqref{eq:priv_var_def} then guarantees $\rho_{y,c}$-zCDP under label-only privatization whenever
\begin{equation}\label{eq:cndrho}
 \rho_{y,c}\geq \frac{B_{y,c}^2}{2\sydpc^2}.
\end{equation}
Under joint feature and label privatization, composition gives a
$(\rho_{y,c}+\rho_{X,c})$-zCDP guarantee, provided that \eqref{eq:cndrho} holds together with
\[
 \rho_{X,c}\geq \frac{d\,B_{X,c}^2}{2\sxdpc^2}.
\]
Thus, privacy budgets of constant order for labels and of linear order in $d$ for features suffice to maintain fixed privacy noise variances while keeping clipping inactive, satisfying Assumptions (\ref{ass:noise_const}) and (\ref{ass:priv_budg_large}).

\paragraph{Uniformly sub-Gaussian samples.}
Suppose now that, for some constants $a_{y,c},a_{X,c}>0$ independent of $n,d,i,j$,
the actual responses and feature coordinates satisfy sub-Gaussian tail bounds
\begin{equation} \label{eq:aux-subgaussian-tails}
 \bbP(|y_{c,i}|>t) \leq 2\exp\!\left(-\frac{t^2}{2a_{y,c}^2}\right), \qquad \bbP(|X_{c,ij}|>t) \leq 2\exp\!\left(-\frac{t^2}{2a_{X,c}^2}\right), \quad t\geq0.
\end{equation}
By \citep[Proposition 2.6.1]{vershynin2018hdp2edit} we have that \eqref{eq:aux-subgaussian-tails} implies the bounded moments of Assumption \eqref{ass:z-moment-bounds}, and thus, can be viewed as a specialization of this assumption. Let us now fix $q>3$ and, for $n\geq2$, choose
\begin{equation} \label{eq:aux-subgaussian-thresholds}
 K_{y,c}=2a_{y,c}\sqrt{2q\log n}, \qquad K_{X,c}=2a_{X,c}\sqrt{2q\log n}.
\end{equation}
A union bound and the sub-Gaussian assumption \eqref{eq:aux-subgaussian-tails} then give
\[
 \bbP\rbr{\max_{i\leq n_c}|y_{c,i}|>K_{y,c}/2}
 \leq 2n_c n^{-q}.
\]
Moreover, a feature row with norm exceeding $\sqrt{d}\,K_{X,c}/2$ must have a
coordinate exceeding $K_{X,c}/2$ in magnitude. Hence another
union bound gives
\[
 \bbP\rbr{\max_{i\leq n_c}\norm{X_{c,i}}_2>\sqrt{d}\,K_{X,c}/2} \leq 2n_cd\,n^{-q},
\]
where $X_{c,i}\in \bbR^{d}$ denotes the $i$-th row of $X_c$. Since $C$ is fixed and $d$ is proportional to $n$,
\[
 \bbP(\cE_n)\leq 2\sum_{c\in\mathcal A}n_c(1+d)n^{-q} =O (n^{2-q}).
\]
The right-hand side is summable because $q>3$, proving \eqref{eq:aux-clipping-inactive}. Since $\log n$ and $\log d$ are asymptotically equivalent, the clipping constants have orders
\[
 K_{y,c}\asymp\sqrt{\log d}, \qquad K_{X,c}\asymp\sqrt{\log d}.
\]
Thus, from \eqref{eq:priv_var_def}, we have that constant variances for the privacy noise are compatible with privacy budgets of orders $\log d$ for labels and $d\log d$ for features.

\paragraph{General samples.}

Assumptions \eqref{ass:z-moment-bounds}-\eqref{ass:C-client-cov-bounds} combined with the assumption on bounded
noise moments imply that features and labels have bounded $8$-th moments. To see this, independence, centering, and the uniformly bounded $8$-th moments of the entries of $Z_c$, as assumed in \eqref{ass:z-moment-bounds}, imply that, for every deterministic
$a\in\mathbb R^d$,
\begin{equation} \label{eq:aux-linear-form-eighth-moment}
 \bbE\abs{\sum_{j=1}^d a_jZ_{c,ij}}^8 \leq K_8\rbr{\sum_{j=1}^d a_j^2}^4,
\end{equation}
where $K_8$ is a constant independent of $c,i,n,d$ and $a$. Indeed, expanding the $8$-th power eliminates every term containing a singleton factor and the remaining terms are bounded by a constant times $\norm{a}_2^8$. Taking $a=\Sigma_c^{1/2}e_j$, where $e_j$ is the $j$-th coordinate
vector, gives
\[
 \bbE\abs{X_{c,ij}}^8 \leq K_8(\Sigma_{c,jj})^4 \leq K_8\alpha_\Sigma^4.
\]
Moreover, Assumption \eqref{ass:z-moment-bounds} gives $\sup_{c,n,d}\norm{\beta_c}_2<\infty$ almost surely. Let us condition on this probability one event. Since $\beta_c$ is independent of the design, \eqref{eq:aux-linear-form-eighth-moment} yields
\[
 \bbE\sbr{\abs{x_{c,i}^\top\beta_c}^8\,\middle|\,\beta_c} \leq K_8\norm{\Sigma_c^{1/2}\beta_c}_2^8 \leq K_8\alpha_\Sigma^4\norm{\beta_c}_2^8.
\]
Using $y_{c,i}=X_{c,i}^\top\beta_c+\varepsilon_{c,i}$ and $|u+v|^8\leq2^7(|u|^8+|v|^8)$, we  obtain constants $M_y,M_X<\infty$ such that
\[
 \sup_{c,i,n,d}\mathbb E|y_{c,i}|^8\leq M_y, \qquad \sup_{c,i,j,n,d}\mathbb E|X_{c,ij}|^8\leq M_X.
\]
For any fixed $\eta>0$, choose
\begin{equation}\label{eq:aux-moment-thresholds}
 K_{y,c}=2n_c^{1/4+\eta}, \qquad K_{X,c}=2\,n_c^{3/8+\eta}.
\end{equation}
Markov's inequality and a union bound give
\begin{align*}
 \bbP\rbr{\max_{i\leq n_c}\abs{y_{c,i}}>K_{y,c}/2} &\leq M_y n_c^{-1-8\eta},\\
 \bbP\rbr{\max_{i\leq n_c}\norm{x_{c,i}}_2>\sqrt d\,K_{X,c}/2} &\leq\mathbb P\rbr{ \max_{i\leq n_c,\,j\leq d}\abs{X_{c,ij}}>n_c^{3/8+\eta}} \leq M_Xd\,n_c^{-2-8\eta}.
\end{align*}
Since $C$ is fixed and $n_c\asymp d\asymp n$,
\[
 \bbP(\cE_n) \leq\sum_{c\in\mathcal A}\rbr{M_y n_c^{-1-8\eta}+M_Xd\,n_c^{-2-8\eta}} = O(n^{-1-8\eta}).
\]
This bound is summable, so the first Borel-Cantelli lemma proves \eqref{eq:aux-clipping-inactive}. Hence, constant variances for the privacy noise are compatible with privacy budgets of order $O(n_c^{1/2+2\eta})$ for labels and $O(d\,n_c^{3/4+2\eta})$ for features.

\subsection{Proof of the bias-variance identities}
\label{app:aux-bias-variance}

For the unclipped estimator, let us set $w_1=1$ and $\tilde{\varepsilon}_1=\varepsilon_1$, $\tilde{\varepsilon}_c=\varepsilon_c+\xi_c$, for  $c\in\cA$, with $s_1^2=\sigma_1^2$, $s_c^2=\sigma_{t,c}^2$, for $c\in\mathcal A$.
Then
\[
 \tilde{y}_c=X_cL_c\theta+\tilde{\varepsilon}_c
\]
for every client, where $\tilde y_1=y_1$. By assumption, all entries of the vectors $\tilde\varepsilon_c$ ($c\in\cC$) are independent between themselves and w.r.t.\ the designs and the signal, they are centered, they have variances $s_c^2$ and uniformly bounded $8$-th moments.
Define
\[
 \hat{u}\coloneqq\frac1n\sum_{c=1}^C w_cL_c^\top X_c^\top \tilde{\varepsilon}_c.
\]
Since $\Uhat=\Shat\theta+\hat{u}$, the closed form \eqref{eq:theta_closed_form} gives
\begin{equation}\label{eq:thetahat_theta}
    \hat\theta-\theta =-\Qhat\Lambda\theta+\Qhat \hat{u}.
\end{equation}
Taking its conditional expectation and covariance yields
\begin{align}
 \mathbb E[\hat\theta\mid X,\theta]-\theta &=-\Qhat\Lambda\theta, \notag\\
 \operatorname{Cov}(\hat\theta\mid X,\theta) &=\frac1n\Qhat \rbr{\sigma_1^2\Shat_1 + \sum_{c\in\mathcal A}w_c^2\sigma_{t,c}^2\Shat_c} \Qhat, \label{eq:aux-estimator-covariance}
\end{align}
where $X=(X_1,\ldots,X_C)$. Then, we have
\[
 B(\thetahat;\theta) = \norm{\Qhat\Lambda\theta}_{S_1}^2 = \theta^\top\Lambda\Qhat S_1\Qhat\Lambda\theta.
\]
Taking the trace of \eqref{eq:aux-estimator-covariance} against $S_1$ similarly gives
\[
 V(\thetahat;\theta) = \frac{1}{n}\Tr\sbr{ S_1\Qhat\rbr{\sigma_1^2\Shat_1 +\sum_{c\in\cA}w_c^2\sigma_{t,c}^2\Shat_c}\Qhat},
\]
thus proving \eqref{eq:biasvarclosed}. \qed
\section{Main theoretical results}
\label{sec:allproofs}

\subsection{Proof of Theorem~\ref{thm:main-det-eq}}
\label{app:proof-main-det-eq}

The proof proceeds in six steps. First, we express the variance and, under Assumption \eqref{ass:indep_signal}, the bias in terms of the scalar resolvent functional
\[
F=\frac1n\Tr(S_1\Qhat)
\]
and its derivatives with respect to the ridge parameters and source weights. Second, a
heterogeneous-column resolvent theorem gives a deterministic equivalent for \(F\) after introducing a strictly positive auxiliary regularization. Third, the regularization is
removed. Fourth, the convergence of \(F\) is transferred to all first derivatives entering the bias and variance formulas. Fifth, we combine these results to obtain the variance limit and the bias limit under Assumption~\eqref{ass:indep_signal}. Finally, in the sixth step, we identify the resulting derivative expressions with the quantities $g$ and $\cK$ appearing in the theorem, and apply Proposition \ref{prop:mc-directional} to the bias quadratic form to remove Assumption \eqref{ass:indep_signal} and obtain the general bias characterization.



\paragraph{Step 1: Reduction to one resolvent
trace.}
As it will play an important role in the proof, we will introduce the notation $\cF=e_1$.
Let $\cJ = \cbr{\beta,\delta_2\dots, \delta_C}$ and, for $j \in \cJ$, define the projection $ P_j:\bbR^{Cd} \to \bbR^{Cd}$ onto the $d$ coordinates of the $j$-th block, i.e.,
\[
    P_j v = \begin{cases}
        \matrix{v_{1}& \dots& v_{d}& 0& \dots & 0 }, & \mbox{if}\ j=\beta,\\[2mm]
        \matrix{0& \dots& 0& v_{(k-1)d+1}& \dots& v_{kd}& 0& \dots & 0}, & \mbox{if}\ j = \delta_k,
    \end{cases}
\]
for $v = \matrix{v_{1}& v_{2} & \dots& v_{Cd} }\in \bbR^{Cd}$.
We can rewrite the joint ridge term using these projections as $\Lambda=\sum_{j\in\cJ}\lambda_jP_j$. For convenience, we recall \[
    \Shat=\Shat_1+\sum_{c\in\cA}w_c\Shat_c,
    \qquad
    \Qhat=(\Shat+\Lambda)^{-1},
\]
and we refer to Section \ref{sec:preliminaries} for the rest of the notation. The central empirical resolvent functional for bias is defined as
\begin{equation} \label{eq:proof-F}
    F(\mathbf w,\lambda_\beta,\boldsymbol{\lambda}_\delta)
    \coloneqq\frac1n\Tr(S_1\Qhat),
\end{equation}
and the main formula for variance will be obtained from it
and its first derivatives. Indeed, using the chain rule for $\vartheta \in \cbr{\lambda_\beta,\lambda_{\delta_2}, \dots, \lambda_{\delta_c}, w_2, \dots, w_C}$, we get
\[
    \partial_{\vartheta}\Qhat = -\Qhat\rbr{\partial_\vartheta \Qhat^{-1}}\Qhat. 
\]
Hence, for $j\in\cJ$ and $c\in\cA$, we have
\begin{align}
    \partial_{\lambda_j}F
    &=-\frac1n\Tr(S_1\Qhat P_j\Qhat),
    \label{eq:proof-F-lambda-C}\\
    \partial_{w_c}F
    &=-\frac1n\Tr(S_1\Qhat\Shat_c\Qhat).
    \label{eq:proof-F-w-C}
\end{align}
We first use these identities to rewrite the variance. Note that
\[
    \sigma_1^2\Shat_1 + \sum_{c\in\cA}w_c^2\stc^2\Shat_c =  \sigma_1^2\Shat + \sum_{c\in\cA} w_c(w_c\stc^2-\sigma_1^2)\Shat_c, \qquad \Qhat\Shat\Qhat
    =\Qhat-\Qhat\Lambda\Qhat.
\]
Substituting this into \eqref{eq:biasvarclosed}, followed by
\eqref{eq:proof-F-lambda-C} - \eqref{eq:proof-F-w-C} yields
\begin{equation}
\label{eq:proof-variance-F}
    V = \sigma_1^2\rbr{F+\sum_{j\in\cJ} \lambda_j\partial_{\lambda_j}F} - \sum_{c\in\cA}w_c(w_c\stc^2-\sigma_1^2)\partial_{w_c}F.
\end{equation}
We next identify a similar representation for the bias. First, introduce $M\coloneqq\Lambda\Qhat S_1\Qhat\Lambda$, and partition $M$ into $d\times d$ blocks $M_{jk}$ indexed by $j,k\in\cJ$.  Defining $\theta_\beta=\beta$ and
$\theta_{\delta_c}=\delta_c$ for $c\in\cA$, we can expand the
expression for the bias in  \eqref{eq:biasvarclosed} as
\begin{equation}
\label{eq:proof-bias-block-C}
    B=\sum_{j,k\in\cJ}\theta_j^\top M_{jk}\theta_k.
\end{equation}
The signal blocks are independent and centered, so, conditionally on all data feature matrices,
\begin{align*}
    \bbE_\theta[B\mid X_1,\ldots,X_C]
    &= \sum_{j\in \cJ} \frac{r_j^2}{d}\Tr(M_{jj})\\
    &= \sum_{j\in \cJ}\frac{r_j^2\lambda_j^2}{d}\Tr(S_1\Qhat P_j\Qhat)\\
    &= - \sum_{j\in \cJ}\tau_j\lambda_j^2\partial_{\lambda_j}F,
\end{align*}
where last two equalities follow from the definition of $M$, cyclicity of the trace and \eqref{eq:proof-F-lambda-C}. Therefore
\begin{align}
    \bbE_\theta[B\mid X_1,\ldots,X_C]
    &=- \sum_{j\in \cJ}\tau_j\lambda_j^2\partial_{\lambda_j}F\notag\\
    &=-\tau_\beta\lambda_\beta^2\partial_{\lambda_\beta}F-\sum_{c\in\cA}\tau_{\delta_c}\lambda_{\delta_c}^2\partial_{\lambda_{\delta_c}}F \label{eq:proof-mean-bias-F}.
\end{align} 

\paragraph{Step 2: Deterministic equivalent for $F$ for strictly positive
regularization.}

The matrix $\Shat + \Lambda$, showing up in definition of $F$,
may have zero ridge parameters. We therefore first add an auxiliary scalar regularization, as it is technically simpler to work with. For $\eta>0$, define
\[
    \Lambda_\eta\coloneqq\Lambda+\eta I_{Cd}, \qquad \Qhat_\eta\coloneqq(\Shat+\Lambda_\eta)^{-1}, \qquad F_\eta\coloneqq\frac1n\Tr(S_1\Qhat_\eta).
\]
We will apply the heterogeneous-column resolvent theorem \citep[Theorem~1]{wagner2012large}. To this end, set $w_1\coloneqq1$, and let
\begin{equation}
\label{eq:proof-Qbar-eta}
    \bar Q_\eta
    =\sbr{\Lambda_\eta + \sum_{c=1}^C \frac{\pi_cw_cS_c}{1+e_{c,\eta}}}^{-1},
\end{equation}
where
\begin{equation}
\label{eq:proof-e-eta}
    e_{c,\eta} = \frac{w_c}{n}\Tr(S_c\bar Q_\eta), \qquad c=1,\ldots,C.
\end{equation}
Next, notice that all of the assumptions of \citep[Theorem~1]{wagner2012large} are satisfied, as they follow from Assumptions \eqref{ass:C-client-asymptotics}, \eqref{ass:z-moment-bounds} and \eqref{ass:C-client-cov-bounds}. Thus, there is a unique solution to the fixed point system defined by \eqref{eq:proof-Qbar-eta}-\eqref{eq:proof-e-eta} such that, for every deterministic $A_d\in\mathbb R^{Cd\times Cd}$ with uniformly bounded operator norm,
\[
    \frac{1}{n}\Tr\sbr{A_d(\Qhat_\eta-\bar Q_\eta)} \asconv 0.
\]
Taking $A_d=S_1$ and $\cF_\eta\coloneqq\frac1n\Tr(S_1\bar Q_\eta)$ yields
\begin{equation}
\label{eq:proof-F-DE-eta}
    \abs{F_\eta-\cF_\eta}\asconv0.
\end{equation}

\paragraph{Step 3: Removing the regularization.}

We prove that \eqref{eq:proof-F-DE-eta} holds for $\eta = 0$, that is
\begin{equation}
    \label{eq:proof-f-cf-convergence}
    \abs{F-\cF}\asconv0.
\end{equation}
We will first show that the objects $F$ and $\cF$ are indeed well-defined. 
Note that 
\[
    \Sigmahat_c = \pi_c\Sigma_c^{1/2}
    \rbr{\frac1{n_c}Z_c^\top Z_c}\Sigma_c^{1/2}.
\]
By the Bai-Yin theorem \citep[Theorem 5.11]{bai2010spectral}, 
there exist
constants $0<a<A<\infty$ such that, almost surely, for all sufficiently large $n$,
\begin{equation}
\label{eq:proof-client-spectrum}
    aI_d\preceq\hat\Sigma_c\preceq AI_d,
    \qquad c=1,\ldots,C.
\end{equation}
We will now prove that this implies that $\Shat$ is positive definite. Let, $u = [u_1^\top,u_2^\top,\ldots,u_C^\top]^\top \in \bbR^{Cd}$ be an arbitrary non-zero vector, where $u_i \in \bbR^{d}$, $i \in \cC$. Then 
\[
    u^\top\Shat u = u_1^\top\Sigmahat_1u_1 +\sum_{c\in\cA}w_c(u_1+u_c)^\top\Sigmahat_c(u_1+u_c).
\]
On the event \eqref{eq:proof-client-spectrum}, it holds
\begin{equation*}
    u^\top\Shat u \geq a\norm{u_1}_2^2
    +aw_{-}\sum_{c\in\cA}\norm{u_1+u_c}_2^2>0.
\end{equation*}
Moreover, as $\norm{u_c}_2^2 \leq 2\norm{u_1}_2^2+2\norm{u_1+u_c}_2^2$, it holds
\[
    \norm{u}_2^2 \leq (2C-1)\rbr{\norm{u_1}_2^2+\sum_{c\in \cA}\norm{u_1+u_c}_2^2}.
\]
This implies that $\Shat$ is positive definite, so there exists a constant $\kappa_{\mathrm{E}}>0$, independent of $n$, such that
\begin{equation}
\label{eq:proof-S-gap}
    \Shat\succeq\kappa_{\mathrm{E}}I_{Cd}
\end{equation}
almost surely for all large enough $n$. As $\Qhat_\eta = (\Shat+\Lambda_\eta)^{-1}$, \eqref{eq:proof-S-gap} implies that almost surely
\begin{equation}
\label{eq:proof-Q-bound}
\sup_{\eta\geq 0}\norm{\Qhat_\eta}_{2} \leq \norm{\Shat^{-1}}_{2}\leq \kappa_{\mathrm{E}}^{-1}
\end{equation}
even when some or all ridge parameters are
zero. This means that the matrix $\Qhat$, and so $F$ as well, is indeed well defined. Next, we turn our attention to $\bar{Q}$ and the corresponding fixed point equations. Fix an arbitrary $\eta>0$. Then, 
\eqref{eq:proof-Qbar-eta} implies that, for $c\in \cC$,
\[
    \bar Q_\eta^{-1} \succeq \sum_{c=1}^C \frac{\pi_cw_cS_c}{1+e_{c,\eta}} \succeq \frac{\pi_cw_cS_c}{1+e_{c,\eta}}.
\]
Taking the trace on the range of $S_c$ gives
\[
    \frac{\pi_cw_c}{1+e_{c,\eta}}
    \Tr(S_c\bar Q_\eta)
    \leq\operatorname{rank}(S_c)=d.
\]
As $e_{c,\eta}$ satisfies \eqref{eq:proof-e-eta} and $n\pi_c=n_c$,
\begin{equation}
\label{eq:proof-e-ratio}
    \frac{e_{c,\eta}}{1+e_{c,\eta}}
    \leq\frac d{n_c}.
\end{equation}
Let us choose $\gamma_*\in(\max_c\bar\gamma_c,1)$.  For all sufficiently large $n$,
\begin{equation*}
    \frac{1}{1+e_{c,\eta}}\geq1-\gamma_*,
    \qquad c=1,\ldots,C,
\end{equation*}
and therefore
\begin{equation*}
    \bar Q_\eta^{-1}
    \succeq
    (1-\gamma_*)\sum_{c=1}^C\pi_cw_cS_c.
\end{equation*}
We can prove in the same way as for $\Shat$ that the matrix on the right is uniformly positive definite. Namely, there exists $\kappa_{\mathrm{DE}}>0$, independent of $n$, such that
\begin{equation}\label{eq:proof-qbar-inverse-bound}
    \bar Q_\eta^{-1} \succeq
    (1-\gamma_*)\sum_{c=1}^C\pi_cw_cS_c\succeq\kappa_{\mathrm{DE}}I_{Cd}
\end{equation}
uniformly in $\eta>0$ for all sufficiently large $n$.  This further implies that
\begin{equation}
\label{eq:proof-Q-bar-bound}
\sup_{\eta>0}\norm{\bar{Q}_\eta}_{2} \leq\kappa_{\mathrm{DE}}^{-1}.
\end{equation}
By \eqref{eq:proof-e-ratio}, for all sufficiently large $n$,
\[
0 \le e_{c,\eta} \le \frac{\gamma_*}{1-\gamma_*},
\qquad c\in\cC,\quad \eta>0.
\]
Fix such an $n$ and an arbitrary sequence $\eta_k\downarrow0$.
By compactness in $\mathbb R^C$, there exists a subsequence
$\eta_{k_i}$ along which the vector
$(e_{1,\eta_{k_i}},\ldots,e_{C,\eta_{k_i}})$ converges.
We define
\begin{equation}\label{eq:proof-eq-limit}
    e_c \coloneqq \lim_{i\to\infty} e_{c,\eta_{k_i}}.
\end{equation}
Note that by \eqref{eq:proof-Qbar-eta} - \eqref{eq:proof-e-eta} it holds
\begin{align*}
    \lim_{i\to\infty} e_{c,\eta_{k_i}} &=  \lim_{i\to\infty} \frac{w_c}{n}\Tr \sbr{S_c\rbr{\Lambda_{\eta_{k_i}} + \sum_{c=1}^C\frac{\pi_cw_cS_c}{1+e_{c,\eta_{k_i}}}}^{-1}} \\
    &= \frac{w_c}{n}\Tr\sbr{S_c\rbr{\lim_{i\to\infty}\Lambda_{\eta_{k_i}} + \sum_{c=1}^C\frac{\pi_cw_cS_c}{1+\lim_{i\to\infty}e_{c,\eta_{k_i}}}}^{-1}}\\
    &= \frac{w_c}{n}\Tr\sbr{S_c\rbr{\Lambda + \sum_{c=1}^C\frac{\pi_cw_cS_c}{1+e_{c}}}^{-1}},
\end{align*}
where the second equality comes from continuity and existence of the limits of a converging sequence. Thus, $e_c$ solves the fixed point equations
\[
    \bar Q =\sbr{\Lambda + \sum_{c=1}^C\frac{\pi_cw_cS_c}{1+e_c}}^{-1}, \qquad e_c=\frac{w_c}{n}\Tr(S_c\bar Q),\qquad \text{ for } c\in \cC.
\]
The limiting matrix being inverted is bounded below by
$\kappa_{\mathrm{DE}}I$ by \eqref{eq:proof-qbar-inverse-bound}, so $\bar Q$ is
well-defined and positive definite. Moreover, the fixed point
equations imply
\[
e_c=\frac{w_c}{n}\operatorname{Tr}(S_c\bar Q)>0,
\qquad c\in\cC,
\]
because $w_c>0$, $S_c$ is positive semidefinite and nonzero,
and $\bar Q$ is positive definite.

Next, let us now prove that the positive fixed point of this equation is unique. To this end, let us define the fixed point map $\Phi:\bbR^C \to \bbR^{C}$ coordinate-wise as $\Phi \coloneqq (\Phi_1,\ \dots,\ \Phi_C)$ where each coordinate function $\Phi_c$, $c\in \{1,\ \dots,\ C\}$, is defined as 
\begin{equation}\label{eq:proof-phic-def}
    \Phi_c(\mathbf{e}) = \frac{w_c}{n}\Tr\sbr{S_c\rbr{\Lambda+ \sum_{c=1}^C\frac{\pi_cw_cS_c}{1+e_c}}^{-1}},
\end{equation}
for $\mathbf{e} = (e_1,\ \dots,\ e_C)$. Note that $\Phi$ is componentwise increasing and $\Phi(\alpha \mathbf{e})<\alpha\Phi(\mathbf{e})$, for every $\alpha>1$. Assume by contradiction that there exist two solutions $\mathbf{e}^{(1)}> 0$ and $\mathbf{e}^{(2)}> 0$ such that
\[
    \Phi(\mathbf{e}^{(1)}) = \mathbf{e}^{(1)},\qquad  \Phi(\mathbf{e}^{(2)}) = \mathbf{e}^{(2)}.
\]
Suppose, without loss of generality, that $\alpha = \max_{c\in \cC} \frac{e_c^{(1)}}{e_c^{(2)}}>1$. Then, for every $c\in \cC$, $
e_c^{(1)}\leq \alpha e_c^{(2)}$ and, by monotonicity, $
\Phi(\mathbf{e}^{(1)})\leq  \Phi(\alpha\mathbf{e}^{(2)})$.
As $\mathbf{e}^{(1)}$ is a fixed point, $
\mathbf{e}^{(1)}\leq  \Phi(\alpha\mathbf{e}^{(2)})$.
Next, using that $\Phi(\alpha \mathbf{e}^{(2)})<\alpha\Phi(\mathbf{e}^{(2)})$ and that $\mathbf{e}^{(2)}$ is a fixed point we get that
$
    \mathbf{e}^{(1)}< \alpha\mathbf{e}^{(2)}$
holds componentwise. However, by definition of $\alpha$, there exists $c$ s.t.\ $e^{(1)}_c= \alpha e^{(2)}_c$. This gives a contradiction and proves uniqueness.

Having established existence and uniqueness of the fixed point at \(\eta=0\), fix a deterministic sequence \(\eta_m=1/m\), independent of \(n\), for the remainder of the convergence argument. Using the triangle inequality and taking $\limsup_{n\to \infty}$, we have
\begin{equation} \label{eq:limsup-f-bound}
    \limsup_{n\to \infty}\abs{F-\cF} \leq \limsup_{n\to \infty}\abs{F-F_{\eta_{m}}} + \limsup_{n\to \infty}\abs{F_{\eta_{m}} - \cF_{\eta_{m}}} + \limsup_{n\to \infty}\abs{\cF_{\eta_{m}} - \cF}.
\end{equation}
For the first term on the RHS, note that $\Qhat_\eta-\Qhat=-\eta\Qhat_\eta\Qhat$, so it follows 
\begin{equation*}
    \abs{F_\eta - F} = \abs{\frac{1}{n}\Tr\sbr{S_1(\Qhat_\eta - \Qhat)}} = \frac{\eta}{n}\abs{\Tr\sbr{S_1\Qhat_\eta\Qhat}}
\end{equation*}
Then, we can give a bound
\begin{align}
    \abs{F_{\eta_{m}} - F} &\leq \frac{\eta_{m} d}{n}\,\norm{S_1}_2\,\norm{\Qhat_{\eta_{m}}}_2\,\norm{\Qhat}_2 \notag\\
    &\leq \eta_{m} \gamma_* \alpha_{\Sigma} \kappa_{\mathrm{E}}^{-2},\label{eq:proof-feta-f-bound}
\end{align}
where the last inequality holds almost surely for large enough $n$. For the second term of RHS, we have that from \eqref{eq:proof-F-DE-eta} it holds that almost surely 
\begin{equation}\label{eq:proof-feta-cfeta-limsup}
    \limsup_{n\to\infty} \abs{F_{\eta_{m}} - \cF_{\eta_{m}}}= 0,
\end{equation}
for each fixed $\eta_{m}$. 

Lastly, we give a bound for the third term of the RHS. To this end, let us extend the definition of $\Phi$ from \eqref{eq:proof-phic-def} to account for an arbitrary $\eta$ parameter. Namely, we introduce $\Phi_\eta: \bbR^C\to \bbR^C$ as $\Phi_\eta = (\Phi_{\eta,1},\ \dots,\ \Phi_{\eta,C})$ where each coordinate function $\Phi_{\eta,c}$, $c\in \{1,\ \dots,\ C\}$, is defined as  
\[
    \Phi_{\eta,c}(\mathbf {e})\coloneqq \frac{w_c}{n}\Tr\sbr{S_c \rbr{\Lambda+\eta I + \sum_{k=1}^C \frac{\pi_kw_kS_k}{1+e_k}}^{-1}},
\]
for $\mathbf{e} = (e_1,\ \dots,\ e_C)$. One can readily verify that $\Phi_\eta$ is componentwise increasing as well as decreasing in $\eta$, that is, $\Phi_\eta(\mathbf{e}) \leq \Phi_0(\mathbf{e})=\Phi(\mathbf{e})$. As before for $\Phi$, we can prove the uniqueness of the fixed point for $\Phi_\eta$. 

Now, let us denote by $\mathbf {e}^{(0)}=(e_1^{(0)},\ldots,e_C^{(0)})$ and $\mathbf {e_\eta}=(e_{1,\eta},\ldots,e_{C,\eta})$ the fixed point of $\Phi$ and $\Phi_\eta$, respectively, such that
\[
    \Phi_{0}(\mathbf {e}^{(0)}) = \Phi (\mathbf {e}^{(0)}) = \mathbf {e}^{(0)}, \qquad \Phi_{\eta}(\mathbf {e_\eta}) = \mathbf {e_\eta}.
\]
Since $\Phi_\eta$ is componentwise increasing and decreasing in $\eta$, it holds
\[
\Phi_\eta(\mathbf {e}^{(0)})\leq \Phi (\mathbf {e}^{(0)}) = \mathbf {e}^{(0)}.
\]
Thus, we define a sequence $\{\mathbf {e}^{(i)}\}_{i=0}^\infty$ such that
\[
    \mathbf {e}^{(i+1)} \coloneqq \Phi_\eta(\mathbf{e}^{(i)}),
\]
which is decreasing, due to the componentwise monotonicity of $\Phi_\eta$. As it is bounded below by $0$, the sequence $\mathbf {e}^{(i)}$ converges componentwise to some
$\mathbf {e}^*\geq0$. By continuity of $\Phi_\eta$,
$\mathbf e^*=\Phi_\eta(\mathbf e^*)$, so uniqueness of the fixed
point gives $\mathbf e^*=\mathbf{e_\eta}$. Therefore, from the monotonicity of the sequence, it holds 
\[
    0\leq e_{c,\eta} = \lim_{i\to\infty} e_c^{(i)}\leq e_c^{(0)},
    \qquad c\in \cC.
\]
Let us now define the difference we would like to bound as
\[
    r_\eta \coloneqq \max_{1\leq c\leq C} \frac{e_c^{(0)}-e_{c,\eta}}{1+e_c^{(0)}}.
\]
Since $e_{c,\eta}\geq0$, one has $0\leq r_\eta<1$.  By definition it holds $1+e_{c,\eta} \geq (1-r_\eta)(1+e_c^{(0)})$, and therefore
\[
    \frac{1}{1+e_{c,\eta}} \leq \frac{1}{1-r_\eta}\frac{1}{1+e_c^{(0)}}.
\]
Using the fixed point representations of $\bar Q_\eta$ and $\bar Q$,
we obtain
\begin{align*}
    \bar Q_\eta^{-1} &= \Lambda+\eta I + \sum_{c=1}^C \frac{\pi_cw_cS_c}{1+e_{c,\eta}}\\
    &\preceq \Lambda+\eta I + \frac{1}{1-r_\eta} \sum_{c=1}^C \frac{\pi_cw_cS_c}{1+e_c^{(0)}}\\
    &\preceq \frac{1}{1-r_\eta}\bar Q^{-1}+\eta I.
\end{align*}
By \eqref{eq:proof-Q-bar-bound}, and the fact that $\bar Q$ is well defined as the limit of $\bar{Q}_{\eta_{k_i}}$ where $\eta_{k_i}\downarrow 0$, it holds $\bar Q^{-1}\succeq\kappa_{\mathrm{DE}}I$, for all sufficiently large $n$.  Therefore we get
\[
    \eta I \preceq \frac{\eta}{\kappa_{\mathrm{DE}}}\bar Q^{-1},
\]
so that
\[
    \bar Q_\eta^{-1} \preceq \rbr{\frac{1}{1-r_\eta} + \frac{\eta}{\kappa_{\mathrm{DE}}}}\bar Q^{-1}.
\]
If we denote by $q_\eta$ the reciprocal of the bracket coefficient, after inversion we get
\begin{equation}\label{eq:qbar_bound}
    \bar Q_\eta \succeq q_\eta\bar Q, \qquad q_\eta \coloneqq \frac{1-r_\eta} {1+\frac{\eta}{\kappa_{\mathrm{DE}}}(1-r_\eta)}.
\end{equation}
Then it directly follows that
\[
    e_{c,\eta}\geq q_\eta e_c^{(0)},
\]
and so
\[
    \frac{e_c^{(0)}-e_{c,\eta}}{1+e_c^{(0)}} \leq (1-q_\eta)\frac{e_c^{(0)}}{1+e_c^{(0)}}.
\]
From \eqref{eq:proof-e-ratio}, and the fact that $e_c^{(0)}$ is well defined as the limit of $e_{c,\eta_{k_i}}$ where $\eta_{k_i}\downarrow 0$, for all sufficiently large $n$,
\[
    \frac{e_c^{(0)}}{1+e_c^{(0)}} \leq \frac{d}{n_c} \leq\gamma_*<1.
\]
Taking the maximum over $c$ gives
\[
    r_\eta\leq\gamma_*(1-q_\eta).
\]
Writing $a_\eta\coloneqq\eta/\kappa_{\mathrm{DE}}$, we have
\[
    1-q_\eta = \frac{r_\eta+a_\eta(1-r_\eta)} {1+a_\eta(1-r_\eta)} \leq r_\eta+a_\eta.
\]
Therefore
\[
    r_\eta \leq \gamma_*(r_\eta+a_\eta),
\]
which implies
\begin{equation}\label{eq:proof-r-eta-bound}
    r_\eta \leq \frac{\gamma_*}{1-\gamma_*} \frac{\eta}{\kappa_{\mathrm{DE}}}.
\end{equation}
By definition it holds that $\cF=e_1^{(0)}$ and $\cF_\eta=e_{1,\eta}$. Moreover, \eqref{eq:proof-e-ratio} implies
\[
    1+e_1^{(0)}\leq\frac{1}{1-\gamma_*}.
\]
Combining this with \eqref{eq:proof-r-eta-bound} yields
\begin{equation}    \label{eq:proof-F-eta-deterministic}
    0\leq\abs{\cF-\cF_\eta} \leq \eta\ \frac{\gamma_*} {\kappa_{\mathrm{DE}}(1-\gamma_*)^2}.
\end{equation}

Finally, plugging into \eqref{eq:limsup-f-bound} the bounds from \eqref{eq:proof-feta-f-bound}, \eqref{eq:proof-feta-cfeta-limsup} and \eqref{eq:proof-F-eta-deterministic}  we get that for each $\eta_{m}$ it holds almost surely
\[
     \limsup_{n\to \infty}\abs{F-\cF} \leq \eta_{m} \gamma_* \alpha_{\Sigma} \kappa_{\mathrm{E}}^{-2}+ \eta_{m}\ \frac{\gamma_*} {\kappa_{\mathrm{DE}}(1-\gamma_*)^2}.
\]
As a countable intersection of probability $1$ sets is of probability $1$, the previous inequality holds with probability $1$ over all $\eta_{m}$ simultaneously. Taking $\lim_{m\to\infty}$ finally gives
\[
     \limsup_{n\to \infty}\abs{F-\cF} = \lim_{m\to\infty}\limsup_{n\to \infty}\abs{F-\cF} \leq \lim_{m\to\infty}\eta_{m}\rbr{ \gamma_* \alpha_{\Sigma} \kappa_{\mathrm{E}}^{-2}+\frac{\gamma_*} {\kappa_{\mathrm{DE}}(1-\gamma_*)^2}} = 0,
\]
almost surely, proving the desired claim.

\paragraph{Step 4: Deterministic equivalents for first derivatives of $F$.}

We now transfer the convergence $F-\cF\asconv 0$ to the first
derivatives appearing in \eqref{eq:proof-variance-F} and
\eqref{eq:proof-mean-bias-F}.  Let
\[
    \cP \coloneqq \{\lambda_\beta\} \cup \{\lambda_{\delta_c}:c\in\cA\} \cup \{w_c:c\in\cA\}.
\]
We will prove that, for every $\vartheta\in\cP$,
\begin{equation}
    \label{eq:proof-derivative-convergence}
    \partial_\vartheta F-\partial_\vartheta\cF
    \asconv 0.
\end{equation}
When $\vartheta$ is a ridge parameter equal to zero, the derivative is
understood as the right derivative.

Take $\vartheta\in\cP$ and keep all other parameters fixed.
We write $t_0$ for the value of $\vartheta$ at the parameter tuple under
consideration, and regard $F$ and $\cF$, and any other objects depending on $\vartheta$, as functions of the scalar
variable $t$ obtained by replacing $\vartheta$ by $t$. Let us choose an arbitrary
$h_0>0$ and consider the one-sided interval
\[
    \cI_\vartheta=[t_0,t_0+h_0].
\]
The interval can be chosen sufficiently small that all weights remain
bounded above and bounded away from zero.

We first establish a uniform second-derivative bound for the empirical
quantity $F$.  On the almost-sure event
\eqref{eq:proof-S-gap}, there exists a constant $\kappa_{\mathrm{E}}>0$, independent of $n$,
such that, almost surely for all sufficiently large $n$,
\[
    \Shat+\Lambda(t)\succeq\Shat\succeq\kappa_{\mathrm{E}} I_{Cd}.
\]
Note that the same lower bound holds uniformly for $t\in \cI_\vartheta$: increasing
a ridge parameter adds a positive semidefinite matrix, while increasing
a weight adds a positive multiple of $\Shat_c$.  Hence,
\[
    \sup_{t\in \cI_\vartheta} \norm{\Qhat(t)}_{2} \leq \kappa_{\mathrm{E}}^{-1}
\]
almost surely for all sufficiently large $n$.

Differentiating the resolvent identity gives, for a ridge parameter
$\lambda_j$ and a weight $w_c$, respectively,
\begin{align*}
    \partial_{\lambda_j}^2F
    &=
    \frac{2}{n}
    \Tr\!\left(
        S_1\Qhat P_j\Qhat P_j\Qhat
    \right),\qquad
    \partial_{w_c}^2F
    =
    \frac{2}{n}
    \Tr\!\left(
        S_1\Qhat\Shat_c\Qhat\Shat_c\Qhat
    \right).
\end{align*}
As $S_1$, $P_j$, $\Shat_c$, and $\Qhat(t)$
are uniformly bounded in operator norm, there exists a constant
$K_1<\infty$ such that
\begin{equation}\label{eq:proof-random-second-derivatives}
    \sup_{t\in \cI_\vartheta}\abs{\partial_t^2F(t)}\leq K_1
\end{equation}
almost surely for all sufficiently large $n$.

We next establish an analogous bound for $\cF$.  Recall that $\cF=\cF(t)=e_1(t)$, where $\mathbf e(t)=(e_1(t),\ldots,e_C(t))$ is the unique fixed
point of
\[
    \mathbf e(t)=\Phi(\mathbf e(t),t),
\]
for $\Phi$ defined in \eqref{eq:proof-phic-def} extended to include dependence on parameter $t$. We will prove that the deterministic fixed point $\mathbf{e}(t)$ is locally smooth in the parameter $t$.
To make this precise, let us make the dependence on $t$ explicit by denoting for an arbitrary $\mathbf e=(e_1,\ldots,e_C)$ 
\begin{equation}\label{eq:Qet_phiet_def}
    Q(\mathbf e,t) \coloneqq \rbr{\Lambda(t) + \sum_{k=1}^C \frac{\pi_k w_k(t)S_k}{1+e_k}}^{-1},
\end{equation}
so  $\Phi_c(\mathbf e,t) = \frac{w_c(t)}{n} \Tr\rbr{S_cQ(\mathbf e,t)}$.
Note that, along the fixed point solution $\mathbf e(t)$, it holds
\[
    \bar Q(t)=Q(\mathbf e(t),t), \qquad 
\]
We will also occasionally drop dependence on $t$ and $\mathbf{e}$ to simplify exposition. Recall that only the parameter corresponding to $\vartheta$ varies with $t$,
all remaining ridge parameters and weights are kept fixed. 
The argument leading to \eqref{eq:proof-qbar-inverse-bound}, applied uniformly over
$t\in \cI_\vartheta$, gives a constant $\kappa_{\mathrm{DE}}>0$,
independent of $n$, such that, for all sufficiently large $n$,
\begin{equation}
    \label{eq:proof-local-Qbar-bound}
\sup_{t\in \cI_\vartheta}
\|Q(\mathbf{e}(t),t)\|_{2}
=
\sup_{t\in \cI_\vartheta}
\|\bar Q(t)\|_{2}
\le \kappa_{\mathrm{DE}}^{-1}.
\end{equation}
All bounds on partial derivatives below are evaluated at
$(\mathbf{e},t)=(\mathbf{e}(t),t)$, with $\mathbf{e}$ and $t$ treated as independent
variables when taking these derivatives.
Moreover, from \eqref{eq:proof-e-ratio},
\[
    \frac{e_c(t)}{1+e_c(t)}
    \leq
    \frac{d}{n_c}
    \leq\gamma_*<1,
\]
uniformly for $t\in\cI_\vartheta$ and all sufficiently large $n$.
Hence
\begin{equation}
    \label{eq:proof-local-e-bound}
    0\leq e_c(t)
    \leq
    \frac{\gamma_*}{1-\gamma_*},
    \qquad
    c=1,\ldots,C.
\end{equation}
In particular, the factors $(1+e_c(t))^{-1}$ and their first two
derivatives with respect to $e_c$ are uniformly bounded.

We next verify that the fixed point map $\Phi$ has uniformly bounded
partial derivatives up to second order.  The basic identity used
throughout is
\begin{equation}
    \label{eq:proof-Qbar-partial}
    \partial_x Q
    =
    -Q(\partial_x Q^{-1})Q,
\end{equation}
for any scalar variable $x$.  For example,
\[
    \partial_{e_k} Q^{-1} = -\frac{\pi_k w_kS_k}{(1+e_k)^2},
\]
and therefore
\[
    \partial_{e_k}Q = Q \frac{\pi_k w_kS_k}{(1+e_k)^2} Q.
\]
Similarly,
\[
    \partial_{e_k e_k} Q^{-1} = \frac{2\pi_k w_kS_k}{(1+e_k)^3}, \qquad \partial_{e_k e_\ell}Q^{-1}=0, \quad k\neq \ell.
\]
If $t$ is a ridge parameter $\lambda_j$, then, with $\mathbf e$ kept
fixed,
\[
    \partial_t Q^{-1}=P_j, \qquad \partial_{tt} Q^{-1}=0.
\]
If $t$ is a weight $w_r$, then
\[
    \partial_tQ^{-1} = \frac{\pi_rS_r}{1+e_r}, \qquad \partial_{tt}Q^{-1}=0,
\]
again with $\mathbf e$ held fixed.  Mixed derivatives such as
$\partial_{te_k}Q^{-1}$ are either zero or a scalar multiple of $S_k$
with a uniformly bounded coefficient. Repeated use of \eqref{eq:proof-Qbar-partial} therefore shows that every first or second partial derivative of $Q(\mathbf e,t)$ with
respect to $(\mathbf e,t)$ is a finite sum of products of the forms
$Q A_1 Q A_2 Q$ and $QA_1Q$, where each $A_i$ is one of
$P_j$, $S_k$, or a uniformly bounded scalar multiple thereof.
The matrices $S_k$ have uniformly bounded operator norms by
\eqref{ass:C-client-cov-bounds}, while
\eqref{eq:proof-local-Qbar-bound} uniformly bounds $Q(\mathbf{e},t)$.
Consequently, all these first and second derivatives of $Q(\mathbf{e},t)$ have
uniformly bounded operator norm. The same is true for the derivatives of $\Phi$.  Indeed,
\[
    \Phi_c(\mathbf e,t) = \frac{w_c(t)}{n} \Tr(S_cQ(\mathbf e,t)).
\]
Since $\rank(S_c)\leq d$, for any matrix $A$ with
uniformly bounded operator norm,
\[
    \frac1n\abs{\Tr(S_cA)} \leq \frac{d}{n} \norm{S_c}_{2} \norm{A}_{2} \le K_2,
\]
for some constant $K_2<\infty$ independent of $n$.
Thus, for $t\in \cI_\vartheta$,
\begin{equation}
    \label{eq:proof-Phi-derivative-bounds}
    \norm{\partial_t\Phi(\mathbf{e},t)}_2 + \norm{\partial_{tt}\Phi(\mathbf{e},t)}_{2} + \norm{\partial_{t\mathbf e}\Phi(\mathbf{e},t)}_{2} + \norm{\partial_{\mathbf e\mathbf e}\Phi(\mathbf{e},t)}_{2} \leq K_3,
\end{equation}
for a constant $K_3<\infty$ independent of $n$,
where $\partial_{\mathbf{e}}$ denotes differentiation with respect to the vector $\mathbf{e} = (e_1,\ \dots,\ e_C)$ applied entrywise to the object that follows.

Define the Jacobian evaluated at the fixed point by
\[
J(t):=\partial_{\mathbf e}\Phi(\mathbf e(t),t).
\]
Its entries are
\begin{equation}\label{eq:jacobian_entries}
J_{ck}(t) = \frac{\pi_k w_c(t)w_k(t)} {n(1+e_k(t))^2} \operatorname{Tr} \rbr{S_c\bar Q(t)S_k\bar Q(t)}\ge 0.
\end{equation}
We suppress the dependence on $t$ when no confusion can arise. Let $\mathbf{v} = (v_1,\ \dots,\ v_C)\in \bbR^C$ be such that
\[
    v_c\coloneqq1+e_c, \qquad c=1,\ldots,C.
\]
Then, we have
\begin{equation}\label{eq:jacobian_bound}
    (Jv)_c = \frac{w_c}{n}\Tr\sbr{S_c\bar Q \rbr{\sum_{k=1}^C \frac{\pi_kw_kS_k}{1+e_k}}\bar Q}= \frac{w_c}{n}\Tr\sbr{S_c\bar Q \rbr{\bar Q^{-1}-\Lambda} \bar Q}= e_c - \frac{w_c}{n} \Tr(S_c\bar Q\Lambda\bar Q)\leq e_c.
\end{equation}
Therefore, in the weighted maximum norm $\norm{x}_v \coloneqq \max_{1\leq c\leq C}\frac{\abs{x_c}}{1+e_c}$, we have that, for all sufficiently large $n$,
\begin{equation*}
    \norm{J(t)}_v = \max_c\frac{(J(t)v)_c}{1+e_c} \leq \max_c\frac{e_c(t)}{1+e_c(t)} \leq \max_c\frac{d}{n_c} \leq\gamma_*<1,
\end{equation*}
where the bound on $e_c(t)/(1+e_c(t))$ is due to \eqref{eq:proof-e-ratio}.  Consequently, $I_C-J(t)$ is invertible, as
\begin{equation}
    \label{eq:proof-I-minus-J-bound}
    \norm{(I_C-J(t))^{-1}}_v \leq\frac{1}{1-\gamma_*},
\end{equation}
for $t\in \cI_\vartheta$. By \eqref{eq:proof-local-e-bound}, $1\le v_c(t)\le(1-\gamma_*)^{-1}$.
Consequently, uniformly in $n$ and $t$,
\[
\|x\|_{v(t)}
\le \|x\|_2
\le \frac{\sqrt C}{1-\gamma_*}\|x\|_{v(t)}.
\]
Combining this with \eqref{eq:proof-I-minus-J-bound} yields
\begin{equation}\label{eq:opnormb}
\sup_{t\in \cI_\vartheta}
\|(I_C-J(t))^{-1}\|_{2}
\le \frac{\sqrt C}{(1-\gamma_*)^2}.
\end{equation}
We can now control the derivatives of the fixed point itself. Let us write
\[
    G(\mathbf e,t) \coloneqq \mathbf e-\Phi(\mathbf e,t).
\]
Then, the fixed point relation is
\[
    G(\mathbf e(t),t)=0,
\]
and its Jacobian with respect to $\mathbf e$ is
\[
    \partial_{\mathbf e}G = I_C-J,
\]
Fix $n$ sufficiently large and $t_*\in \cI_\vartheta$.
At $(\mathbf e(t_*),t_*)$, all denominators $1+e_c(t_*)$ are
positive and the matrix defining $Q(\mathbf e(t_*),t_*)$ is
positive definite. Hence $Q$, $\Phi$, and $G$ are smooth
on an open neighborhood of this point in
$\mathbb R^C\times\mathbb R$, including when $t_*=0$. Note that $\partial_{\mathbf e}G$ is invertible, so the implicit function theorem
gives a local twice continuously differentiable branch
of solutions. Since $\mathbf e(t_*)$ is componentwise positive,
this branch remains positive after shrinking the
neighborhood. Uniqueness of the positive fixed point
identifies it with $\mathbf e(t)$ for nearby $t\in \cI_\vartheta$. Thus, $\mathbf e(t)$ is twice continuously differentiable on
$\cI_\vartheta$, with derivatives at the endpoints
interpreted one-sidedly. Differentiating $\mathbf e(t)=\Phi(\mathbf e(t),t)$ gives
\begin{equation}
    \label{eq:proof-e-prime-equation}
    (I_C-J)\mathbf e'(t) = \partial_t\Phi.
\end{equation}
Therefore,
\[
    \norm{\mathbf e'(t)}_2 \leq \norm{(I_C-J)^{-1}}_{2} \norm{\partial_t\Phi}_2.
\]
The inverse is uniformly bounded by
\eqref{eq:opnormb}, while
$\partial_t\Phi$ is uniformly bounded by
\eqref{eq:proof-Phi-derivative-bounds}.  Hence there exists
$K_4<\infty$, independent of $n$, such that
\begin{equation}
    \label{eq:proof-e-prime-bound}
    \sup_{t\in\cI_\vartheta} \norm{\mathbf e'(t)} \leq K_4.
\end{equation}
Differentiating \eqref{eq:proof-e-prime-equation} once more gives 
\begin{equation}
    \label{eq:proof-e-second-equation}
    (I_C-J)\mathbf e''(t) = \partial_{tt}\Phi + 2\,\partial_{t\mathbf e}\Phi\,\mathbf e'(t) + \partial_{\mathbf e\mathbf e}\Phi [\mathbf e'(t),\mathbf e'(t)].
\end{equation}
Note that here $\partial_{\mathbf e\mathbf e}\Phi[x,y]$ denotes the vector obtained
by applying the Hessian of each component of $\Phi$ to the pair
$(x,y)$.  By \eqref{eq:proof-Phi-derivative-bounds} and
\eqref{eq:proof-e-prime-bound}, every term on the right-hand side of
\eqref{eq:proof-e-second-equation} is uniformly bounded.  Applying
again the uniform bound on $(I_C-J)^{-1}$ coming from \eqref{eq:opnormb} yields a constant
$K_5<\infty$, independent of $n$, such that
\begin{equation}\label{eq:proof-e-secundum-bound}
    \sup_{t\in\cI_\vartheta} \norm{\mathbf e''(t)} \leq K_5.
\end{equation}
Since $\cF(t)=e_1(t)$, we conclude that, for some constant $K_6<\infty$
independent of $n$,
\begin{equation}
    \label{eq:proof-DE-second-derivatives}
    \sup_{t\in\cI_\vartheta}
    \abs{\partial_t^2\cF(t)}
    \leq K_6.
\end{equation}

We can now transfer the convergence of $F$ and $\cF$ to their first
derivatives.  Let $0<h<h_0$.  By Taylor's theorem, for some
$\xi_F,\xi_{\cF}\in(t_0,t_0+h)$, it holds
\begin{align*}
    F(t_0+h) &= F(t_0) + h\,\partial_tF(t_0) + \frac{h^2}{2}\partial_t^2F(\xi_F),\\
    \cF(t_0+h) &= \cF(t_0) + h\,\partial_t\cF(t_0) + \frac{h^2}{2}\partial_t^2\cF(\xi_{\cF}).
\end{align*}
Dividing by $h$ and using
\eqref{eq:proof-random-second-derivatives} and
\eqref{eq:proof-DE-second-derivatives} gives
\begin{align*}
    \abs{\partial_tF(t_0) - \frac{F(t_0+h)-F(t_0)}{h}} &\leq \frac{K_1}{2}h,\\
    \abs{\partial_t\cF(t_0) - \frac{\cF(t_0+h)-\cF(t_0)}{h}} &\leq \frac{K_6}{2}h.
\end{align*}
By adding and subtracting the two finite-difference quotients, we obtain
\[
    \abs{\partial_tF(t_0) - \partial_t\cF(t_0)} \leq \frac{K_1+K_6}{2}h + \frac{\abs{F(t_0+h)-\cF(t_0+h)} + \abs{F(t_0)-\cF(t_0)}}{h}.
\]
For every fixed $h>0$, both parameter values $t_0$ and $t_0+h$ are
fixed.  Hence the already established deterministic equivalent
convergence gives
\[
    F(t_0)-\cF(t_0)\asconv0, \qquad F(t_0+h)-\cF(t_0+h)\asconv0.
\]
Therefore, for fixed $h$,
\[
    \limsup_{n\to\infty} \abs{\partial_tF(t_0) - \partial_t\cF(t_0)} \leq \frac{K_1+K_6}{2}h
\]
almost surely.

Finally, let us take a deterministic sequence $h_m\downarrow0$.  Since this
sequence is countable, we may intersect the probability-one events
corresponding to all parameter values $t_0+h_m$.  On this common event,
the preceding inequality holds for every $m$.  Sending $m\to\infty$
then  proves \eqref{eq:proof-derivative-convergence}.


\paragraph{Step 5: Completing variance limit and bias limit under Assumption~\eqref{ass:indep_signal}.} For the variance, combining
\eqref{eq:proof-variance-F}, \eqref{eq:proof-f-cf-convergence} and
\eqref{eq:proof-derivative-convergence} immediately gives
\begin{equation}
    \label{eq:proof-variance-convergence-final}
    \abs{V-\bar{\cV}}\asconv 0,
\end{equation}
where $\bar{\cV} = \sigma_1^2\rbr{\cF+\sum_{j\in\cJ}\lambda_j\partial_{\lambda_j}\cF} - \sum_{c\in\cA}w_c(w_cs_c^2-\sigma_1^2)\partial_{w_c}\cF$.  

For the bias, recall from \eqref{eq:proof-bias-block-C} that
\[
    B = \sum_{j,k\in\cJ} \theta_j^\top M_{jk}\theta_k, \qquad M=\Lambda\Qhat S_1\Qhat\Lambda,
\]
where $M_{jk}$ denotes the $(j,k)$ block of $M$.  By
\eqref{eq:proof-S-gap}, $\Qhat$ is uniformly bounded almost surely for
all sufficiently large $n$. Since $\Lambda$ and $S_1$ are uniformly
bounded, there exists a deterministic constant $K_M<\infty$ such that,
almost surely for all sufficiently large $n$,
\begin{equation} \label{eq:proof-M-block-bound}
    \max_{j,k\in\cJ}\norm{M_{jk}}_{2}
    \leq K_M.
\end{equation}
For each $j\in\cJ$, by the assumption in \eqref{eq:C-client-random-signal}, we can write
\[
    \theta_j=\frac{r_j}{\sqrt d}x_j,
\]
where the coordinates of $x_j$ are independent, centered, have unit
variance, and have uniformly bounded eighth moments.  Moreover, the
vectors $x_j$, $j\in\cJ$, are independent of one another and of the
design matrices.

Let us condition on the designs.  For a diagonal block $j=k$, define
\[
    D_j \coloneqq \theta_j^\top M_{jj}\theta_j - \frac{r_j^2}{d}\Tr(M_{jj}).
\]
An application of 
\citep[Lemma~B.26]{bai2010spectral} gives
\[
    \bbE_\theta\sbr{\abs{x_j^\top M_{jj}x_j-\Tr(M_{jj})}^4 \,\middle|\,X_1,\ldots,X_C}\leq K_1\left[\Tr(M_{jj}M_{jj}^\top)\right]^2,
\]
for a constant $K_1$.
Using \eqref{eq:proof-M-block-bound}, we get
\[
    \Tr(M_{jj}M_{jj}^\top) \leq d\norm{M_{jj}}_{2}^2 \leq dK_M^2.
\]
Hence, since the $r_j$ are fixed,
\begin{equation}
    \label{eq:proof-diagonal-bias-fourth}
    \bbE_\theta \sbr{\abs{D_j}^4 \,\middle|\, X_1,\ldots,X_C} \leq \frac{K_2}{d^2}
\end{equation}
almost surely for all sufficiently large $n$ for a constant $K_2$. For two distinct blocks $j\neq k$, we have
\[
    \bbE_\theta \sbr{\abs{x_j^\top M_{jk}x_k}^4 \,\middle|\, X_1,\ldots,X_C} \leq K_3\sbr{\Tr(M_{jk}M_{jk}^\top)}^2.
\]
Therefore,
\begin{equation}
    \label{eq:proof-offdiagonal-bias-fourth}
    \bbE_\theta \sbr{\abs{\theta_j^\top M_{jk}\theta_k}^4 \,\middle|\, X_1,\ldots,X_C} \leq \frac{K_4}{d^2}, \qquad j\neq k.
\end{equation}
Now, by the independence and centering of the signal blocks,
\[
    \bbE_\theta[B\mid X_1,\ldots,X_C] = \sum_{j\in\cJ} \frac{r_j^2}{d}\Tr(M_{jj}),
\]
and consequently
\[
    B-\bbE_\theta[B\mid X_1,\ldots,X_C] = \sum_{j\in\cJ}D_j + \sum_{\substack{j,k\in\cJ\\j\neq k}} \theta_j^\top M_{jk}\theta_k.
\]
Since $C$ is fixed, the number of terms in the two sums is uniformly
bounded.  Thus, using \eqref{eq:proof-diagonal-bias-fourth} and
\eqref{eq:proof-offdiagonal-bias-fourth},
\[
    \bbE_\theta \sbr{\abs{B-\bbE_\theta[B\mid X_1,\ldots,X_C]}^4\,\middle|\,X_1,\ldots,X_C}\leq\frac{K_5}{d^2}
\]
almost surely for all sufficiently large $n$.

Let us fix $\varepsilon>0$.  Markov's inequality then gives
\[
    \bbP_\theta \rbr{\abs{B-\bbE_\theta[B\mid X_1,\ldots,X_C]}>\varepsilon \,\middle|\, X_1,\ldots,X_C} \leq \frac{K_6}{\varepsilon^4d^2}.
\]
As $\sum_{n=1}^\infty \frac{1}{d^2}<\infty$, for almost every realization of the design sequence,
the conditional Borel-Cantelli lemma gives
\begin{equation}
    \label{eq:proof-bias-concentration-final}
    B-\bbE_\theta[B\mid X_1,\ldots,X_C]
    \asconv 0.
\end{equation}
Since the signal vectors are independent of the designs, this conditional almost-sure statement is also an almost-sure statement with respect to the joint probability space generating both the designs and the signal vectors.

Finally, from \eqref{eq:proof-mean-bias-F},
\[
    \bbE_\theta[B\mid X_1,\ldots,X_C] = -\sum_{j\in\cJ} \tau_j\lambda_j^2\partial_{\lambda_j}F.
\]
By \eqref{eq:proof-derivative-convergence}, we obtain that
\begin{equation} \label{eq:proof-mean-bias-convergence-final}
    \bbE_\theta[B\mid X_1,\ldots,X_C]-\bar{\cB} \asconv0,
\end{equation}
where $\bar{\cB}=-\sum_{j\in\cJ}\tau_j\lambda_j^2 \partial_{\lambda_j}\cF$.  Combining \eqref{eq:proof-bias-concentration-final} and \eqref{eq:proof-mean-bias-convergence-final} yields
\[
    \abs{B-\bar{\cB}} \leq \abs{B-\bbE_\theta[B\mid X_1,\ldots,X_C]} + \abs{\bbE_\theta[B\mid X_1,\ldots,X_C]-\bar{\cB}} \asconv0,
\]
which, together with \eqref{eq:proof-variance-convergence-final}, gives the deterministic equivalent formulas for bias and variance under Assumption \eqref{ass:indep_signal}. In the last step that follows, we simplify deterministic equivalent formulas and extend the bias characterization to the general signal setting without Assumption \eqref{ass:indep_signal}.


\paragraph{Step 6: Extension to general signals and identification of the deterministic equivalents.}
Let $\bar Q_{S_1}(t)$ and
$\mathbf e(t)=(e_{c,S_1}(t))_{c=1}^C$ be the deterministic quantities in
Proposition~\ref{prop:mc-directional} with $A=S_1$. At $t=0$, these are exactly 
$\bar Q$ and $\mathbf e=(e_c)_{c=1}^C$ from \eqref{eq:delta-general-block-de}-\eqref{eq:delta-general-fixed-point}. Differentiating the fixed point equations \eqref{eq:qa-det-equiv-eq} gives
\[
    e'_{c,S_1}(0) =-\frac{w_c}{n}\Tr(S_c\bar Q S_1\bar Q) + \sum_{k=1}^C\frac{\pi_k w_cw_k}{n(1+e_k)^2} \Tr(S_c\bar Q S_k\bar Q)e_{k,S_1}'(0).
\]
Thus, $(I_C-J)\mathbf e'(0)=-b$, with $J$ and $b$ as in
\eqref{eq:J_b_g_def}. By the invertibility established in \eqref{eq:opnormb} in Step 4,
\[
    \mathbf e'(0)=-(I_C-J)^{-1}b=-g.
\]
Differentiating the matrix inverse $\bar{Q}_{S_1}(t)$ therefore yields
\begin{equation}\label{eq:proof-directional-K}
    - \frac{d}{dt}\bar Q_{S_1}(t)\Big|_{t=0}  =\bar Q\left[ S_1+\sum_{c=1}^C\frac{\pi_cw_cg_c}{(1+e_c)^2}S_c\right]\bar Q =\cK, \qquad g_c=\frac{w_c}{n}\Tr(S_c\cK).
\end{equation}

Let us first fix a deterministic, uniformly bounded signal sequence $\theta$, and set $a=\Lambda\theta$. The uniform bounds on $\Lambda$ and $\theta$ give $\sup_n\norm{a}_2<\infty$. Moreover, $S_1\succeq0$ is deterministic and
uniformly bounded. Thus, we can apply Proposition \ref{prop:mc-directional} with $a=\Lambda\theta$ and $A=S_1$. Namely, for
\[
    H_{S_1}(t)=a^\top(\Shat+\Lambda+tS_1)^{-1}a, \qquad \cH_{S_1}(t)=a^\top\bar Q_{S_1}(t)a,
\]
it follows from \eqref{eq:mc-H-derivative-convergence} that
\begin{equation}\label{eq:proof-general-bias-convergence}
    \frac{d}{dt}H_{S_1}(t)\Big|_{t=0}-\frac{d}{dt}\cH_{S_1}(t)\Big|_{t=0}\asconv0.
\end{equation}
On the other hand, the exact bias identity \eqref{eq:biasvarclosed} and
\eqref{eq:proof-directional-K} give
\begin{equation}\label{eq:bias_det_bias}
    B=- \frac{d}{dt}H_{S_1}(t)\Big|_{t=0}, \qquad -\frac{d}{dt}\cH_{S_1}(t)\Big|_{t=0}=\theta^\top\Lambda\cK\Lambda\theta,
\end{equation}
and thus, we have proved the convergence of bias in \eqref{eq:C-client-bias-variance-DE}. Recall that $\theta$ is almost surely uniformly bounded by Assumption \eqref{ass:z-moment-bounds}. Hence, we can apply Proposition \ref{prop:mc-directional} for almost every realization, which yields that  \eqref{eq:proof-general-bias-convergence} holds under the law of $\theta$. 

It remains to identify the scalar deterministic equivalents obtained in
the preceding five steps with the simplified expressions in the theorem. Define
\[
    D=\operatorname{diag}\rbr{\frac{\pi_c}{(1+e_c)^2}}_{c=1}^C, \qquad T_{ck}=\frac{w_cw_k}{n}\Tr(S_c\bar Q S_k\bar Q),
\]
and let $\mathbf u_c$ be the $c$-th coordinate vector of $\mathbb R^C$. Then $T=T^\top$, $J=TD$, and $b=T\mathbf u_1$, since $w_1=1$. The identity
\[
    (I_C-TD)^{-1}T=T(I_C-DT)^{-1}
\]
shows that $(I_C-J)^{-1}T$ is symmetric. Its first column is $g$, so
\begin{equation}\label{eq:proof-fixed-point-reciprocity}
    \mathbf u_1^\top(I_C-J)^{-1} =\mathbf u_1^\top+g^\top D.
\end{equation}
For $j\in\cJ$, set $q_c^{(j)}=w_c\Tr(S_c\bar Q P_j\bar Q)/n$. Differentiating the fixed point equations with respect to a ridge parameter and an auxiliary weight, respectively, gives
\[
    (I_C-J)\partial_{\lambda_j}\mathbf e=-q^{(j)}, \qquad (I_C-J)\partial_{w_c}\mathbf e  =\frac{1}{w_c}\sbr{e_c\mathbf u_c-(1+e_c)J\mathbf u_c}, \quad c\in\cA.
\]
Taking the first coordinate and using \eqref{eq:proof-fixed-point-reciprocity}, cyclicity of the trace, and \eqref{eq:ck_def}, we obtain
\begin{align}
    -\partial_{\lambda_j}\cF &=(\mathbf u_1^\top+g^\top D)q^{(j)} = \frac{1}{n}\Tr(P_j\cK), \qquad j\in\cJ, \label{eq:proof-ridge-derivative-K}\\
    -\partial_{w_c}\cF & =\frac{[(I_C-J)^{-1}]_{1c}}{w_c} = \frac{\pi_cg_c}{w_c(1+e_c)^2}, \qquad c\in\cA. \label{eq:proof-weight-derivative-g}
\end{align}
Consequently, under Assumption~\eqref{ass:indep_signal}, the scalar bias
equivalent established in Step~5 is
\begin{equation}\label{eq:C-client-bias-DE}
    \bar{\cB} =-\sum_{j\in\cJ}\tau_j\lambda_j^2 \partial_{\lambda_j}\cF =\frac{1}{n}\sum_{j\in\cJ}\tau_j\lambda_j^2\Tr(P_j\cK)=\cB,
\end{equation}
which is exactly the bias formula \eqref{eq:bias_simpler} in the theorem.

Finally, evaluating $n^{-1}\Tr(\bar Q^{-1}\cK)$ using, respectively,
\eqref{eq:ck_def} and \eqref{eq:delta-general-block-de}, together with
\eqref{eq:proof-directional-K}, gives
\[
    \frac{1}{n}\Tr(\bar Q^{-1}\cK) = \cF+\sum_{c=1}^C\frac{\pi_c e_cg_c}{(1+e_c)^2} = \frac{1}{n}\Tr(\Lambda\cK) + \sum_{c=1}^C\frac{\pi_cg_c}{1+e_c}.
\]
Hence, by \eqref{eq:proof-ridge-derivative-K},
\[
    \cF+\sum_{j\in\cJ}\lambda_j\partial_{\lambda_j}\cF = \cF-\frac1n\Tr(\Lambda\cK) = \sum_{c=1}^C\frac{\pi_cg_c}{(1+e_c)^2}.
\]
Substituting this identity and \eqref{eq:proof-weight-derivative-g} into
the variance equivalent from Step~5 gives
\[
    \bar{\cV} = \sigma_1^2\rbr{\cF+\sum_{j\in\cJ}\lambda_j\partial_{\lambda_j}\cF} - \sum_{c\in\cA}w_c(w_cs_c^2-\sigma_1^2)\partial_{w_c}\cF = \sum_{c=1}^C\frac{\pi_cw_cs_c^2g_c}{(1+e_c)^2}=\cV,
\]
matching the formula in \eqref{eq:C-client-bias-variance-DE}. Since the variance argument does not use assumption \eqref{ass:indep_signal}, this proves the claimed variance limit of \eqref{eq:main_conv} in full generality. Together with
\eqref{eq:proof-general-bias-convergence}-\eqref{eq:bias_det_bias} and, when applicable,
\eqref{eq:C-client-bias-DE}, this completes the proof of
Theorem~\ref{thm:main-det-eq}.
\qed


\subsection{Proof of the full bias term}

We state an auxiliary result that allows us to concentrate the bias term for general fixed hidden signals. Recall from \eqref{eq:biasvarclosed} that $B(\thetahat,\theta) = \theta^\top \Lambda \Qhat S_1 \Qhat \Lambda \theta$. We will first rewrite the bias term in terms of matrices. Namely, let us  set $a_\theta=\Lambda\theta$, and introduce
\begin{equation} \label{eq:mc-H}
 \Qhat_{S_1}(t)=(\Shat+\Lambda+tS_1)^{-1}, \qquad H_{S_1}(t)=a_\theta^\top\Qhat_{{S_1}}(t)a_\theta.
\end{equation}
The vector $a_\theta$ is deterministic, uniformly bounded, and held fixed when $t$ varies. Differentiating the inverse gives
\[
    \frac{d}{dt}H_{S_1}(t) = -a_\theta^\top\Qhat_{S_1}(t)S_1\Qhat_{S_1}(t)a_\theta,
\]
so the exact bias identity is
\[
 B=- \frac{d}{dt}H_{S_1}(t)\Big|_{t=0}.
\]
Thus, we state and prove the following proposition which gives deterministic equivalents for $\frac{d}{dt}H_{S_1}(t)$.

\begin{proposition}\label{prop:mc-directional}
Under the design and parameter assumptions of Theorem \ref{thm:main-det-eq}, let
$a\in\R^{Cd}$ and $A\succeq0$ belong to deterministic sequences with
$\sup_n\norm{a}_2<\infty$ and
$\sup_n\norm{A}_{2}<\infty$. Let us choose an arbitrary $h_0>0$ and for $t\in[0,h_0]$, define empirical matrix $\Qhat_A(t)=(\Shat+\Lambda+tA)^{-1}$ and its corresponding deterministic equivalent
\begin{equation}\label{eq:qa-det-equiv-eq}
     \bar{Q}_A(t) =\left[\Lambda+tA+ \sum_{c=1}^C\frac{\pi_cw_cS_c}{1+e_{c,A}(t)}\right]^{-1}, \qquad e_{c,A}(t)=\frac{w_c}{n}\Tr(S_c\bar{Q}_A(t)).
\end{equation}
Set $H_A(t)=a^\top\Qhat_A(t)a$ and $\cH_A(t)=a^\top\bar{Q}_A(t)a$.
Then for every fixed $t\in[0,h_0]$,
\begin{equation} \label{eq:mc-H-convergence}
    H_A(t)-\cH_A(t)\asconv0, 
\end{equation}
and
\begin{equation} \label{eq:mc-H-derivative-convergence}
 \frac{d}{dt}H_A(t)\Big|_{t=0}-\frac{d}{dt}\cH_{A}(t)\Big|_{t=0}\asconv0.
\end{equation}
\end{proposition} 

\begin{proof}
We extend the existing scalar argument to one deterministic quadratic
form. In the proof of this proposition, where obvious, we omit the subscript $A$.

To begin with, we note that the matrix in \eqref{eq:qa-det-equiv-eq} is indeed well-defined with unique positive solutions $e_{c,A}>0$ as proved Step 2 and Step 3 of the argument. Namely, this follows since every positivity, existence, uniqueness, and comparison argument in
Step 3 remains valid with $\Lambda$ replaced by $\Lambda+tA$,
because $tA\succeq0$.

The main idea is to generalize the approach of using the resolvent theorem treating average direction \citep[Theorem 1]{wagner2012large} to using a more fine-grained result \citep[Theorem 1.6]{louart2026resolvent} allowing per-direction control.

To begin with, we fix $t\geq0$ and $\eta>0$ and set
\[
 D_\eta = \Lambda+tA+\eta I, \qquad \Qhat_\eta(t)=(\Shat+D_\eta)^{-1}.
\]
As before, we introduce $\eta$ as it will allow us to make the argument technically simpler, with the aim to later obtain the results for $\eta = 0$. We can rewrite the $\Qhat_\eta(t)$ in terms of a resolvent function,
\begin{equation}\label{eq:qhat_r_def}
    \Qhat_\eta(t)= D_\eta^{-1/2}R(-1)D_\eta^{-1/2},
\end{equation}
where $R(z)=(D_\eta^{-1/2}\Shat D_\eta^{-1/2}-zI)^{-1}$. We can view $D_\eta^{-1/2}\Shat D_\eta^{-1/2}$ as the empirical covariance matrix of the whitened $n$ samples $\breve{x}_{c,i}\in \bbR^{Cd}$ defined as 
\[
    \breve{x}_{c,i} = \sqrt{w_c}\,D_\eta^{-1/2}L_c^\top\Sigma_c^{1/2}Z_{c,i}^\top,\quad c\in \cC, \,\,i\in [n_c],
\]
where $Z_{c,i}$ denotes the transpose of row $i$ of $Z_c$. Moreover, we have that the population covariance matrix of each sample belonging to client $c$ is
\[
    \breve{\Sigma}_c:=\E[\breve{x}_{c,i}\breve{x}_{c,i}^\top] = w_cD_\eta^{-1/2}S_cD_\eta ^{-1/2}.
\]
Let us introduce the quantity $M_{\mathrm{F}}^{(q)}$ as
\[
    M_{\mathrm{F}}^{(q)} \coloneqq \sup_{\substack{c\in\cC\\ i\in [n_c]\\ \norm{B}_{\mathrm{F}}\leq 1}} \rbr{\bbE \abs{\breve{x}_{c,i}^\top B\breve{x}_{c,i} - \Tr( \breve{\Sigma}_c B)}^q}^{1/q},
\]
where $\norm{\cdot}_{\mathrm{F}}$ is the Frobenius norm, such that $\norm{B}_{\mathrm{F}} = \sqrt{\Tr(BB^*)}$, where $B^*$ is the complex conjugate of $B$. By Assumption \eqref{ass:C-client-cov-bounds}, it holds
\[
    \norm{\breve{\Sigma}_c}_{\mathrm{F}} \leq K_{c,\eta} \sqrt{d} ,
\]
where the constant $K_{c,\eta}$ depends on the fixed $\eta$. By Assumption \eqref{ass:C-client-asymptotics}, we have $\sqrt{d}\norm{\breve{\Sigma}_c}_{\mathrm{F}}= O(n)$, and thus the assumption of Theorem 1.6 of \citep{louart2026resolvent} is satisfied. We then apply this theorem with $q=4$, to get the deterministic equivalent of the resolvent function $R(z)$. Namely, for $v=D_\eta^{-1/2}a$, we have that, at every fixed complex $z \in \bbC \setminus [0,+\infty[$
\begin{equation} \label{eq:mc-directional-L4}
 \rbr{\bbE\abs{v^\top(R(z)-\bar R(z))v}^4}^{\frac{1}{4}} \!\leq \!K_{7}\norm{vv^\top}_{\mathrm{F}}\sbr{\frac{M_{\mathrm{F}}^{(4)}}{\sqrt{n}}+(1+\max_c\norm{\breve{\Sigma}_c}_2)\frac{\max_c\norm{\breve{\Sigma}_c}_\mathrm{F}+(M_\mathrm{F}^{(2)})^2\sqrt{Cd}}{n}},
\end{equation}
where $K_{7}$ is a constant independent of $n$ and $d$, and $\bar R$ is the deterministic resolvent defined by equations
\begin{equation}\label{eq:Rbar-def}
    \bar R(z) = -\frac{1}{z}\rbr{\sum_{c=1}^Cd_c(z)\pi_c\breve{\Sigma}_c +I}^{-1},\qquad d_c(z) = -\frac{1}{z\rbr{1+\frac{1}{n}\Tr\rbr{\breve{\Sigma}_c\bar{R}(z)}}}.
\end{equation}
Note that the objects in the above equations \eqref{eq:Rbar-def} are well-defined and have unique guaranteed solutions for every $z \in \bbC \setminus [0,+\infty[$, by \citep[Theorem 1.1]{louart2026resolvent}.

Let us bound the values on the LHS of \eqref{eq:mc-directional-L4}. By previous calculations, we get
\[
    \max_c\norm{\breve{\Sigma}_c}_{\mathrm{F}} \leq K_{\eta} \sqrt{d},
\]
for some constant $K_{\eta}$ depending on $\eta$. Moreover, by Assumption \eqref{ass:C-client-cov-bounds}, it holds
\[
    \max_c\norm{\breve{\Sigma}_c}_2 \leq K_\eta\alpha_{\Sigma}.
\]
We now give bounds for $M_{\mathrm{F}}^{(4)}$  and $M_\mathrm{F}^{(2)}$. For a deterministic matrix $B$, the fourth-moment inequality for quadratic forms of independent coordinates with bounded eighth moments \citep[Lemma~B.26]{bai2010spectral} gives
\begin{equation} \label{eq:mc-QF4}
\rbr{\bbE \abs{\breve{x}_{c,i}^\top B\breve{x}_{c,i} - \Tr( \breve{\Sigma}_c B)}^4}^{1/4} \leq K_\eta\norm{B}_{F},
\end{equation}
which follows by writing $\breve{x}_{c,i} = \breve{U}_c Z_{c,i}^\top$ and the assumption \eqref{ass:z-moment-bounds} on the bounded eighth moments of coordinates of $Z_c$. As the supremum is taken over $\norm{B}_{F}\leq 1$, this in turn directly implies that $M_{\mathrm{F}}^{(4)}\leq K_{\eta}$. Similarly, by applying \citep[Lemma~B.26]{bai2010spectral} for second moments, we can get $M_{\mathrm{F}}^{(2)}\leq K_{\eta}$ using the fact that bounded eighth moments of coordinates of $Z_c$ imply bounded fourth moments.

Plugging these bounds back into \eqref{eq:mc-directional-L4} we get
\[
     \rbr{\bbE\abs{v^\top(R(z)-\bar R(z))v}^4}^{\frac{1}{4}} \leq \frac{K_\eta}{\sqrt{n}}.
\]

By applying Markov's inequality we have
\[
    \bbP\rbr{\abs{v^\top(R(z)-\bar R(z))v}\geq t}\leq \frac{K_\eta}{t^4n^2}.
\]
Thus, by picking $t=n^{-1/20}$ and applying Borel-Cantelli, we get almost sure convergence to $0$, as the right hand side is summable in $n$. Namely, we get for $z \in \bbC \setminus [0,+\infty[$
\begin{equation}\label{eq:r_rbar_conv}
    \abs{v^\top(R(z)-\bar R(z))v} \asconv 0.
\end{equation}

From the definition of $v$ and the identity in \eqref{eq:qhat_r_def}, we have $v^\top R(-1)v = a^\top \hat{Q}_\eta(t)a$. Moreover, we get 
\[ 
    D_\eta^{-1/2} \bar{R}(-1)D_\eta^{-1/2} = \bar{Q}_\eta(t), \qquad d_c(-1) = 1/(1+e_{c,\eta}),
\] 
where $\bar Q_\eta(t)$ and $e_{c,\eta}$ are the solutions to \eqref{eq:qa-det-equiv-eq}, with $\Lambda+tA$ replaced by $D_\eta$. Thus, by taking $z=-1$ in \eqref{eq:r_rbar_conv} we get
\begin{equation} \label{eq:mc-H-eta-convergence}
 a^\top(\Qhat_\eta(t)-\bar Q_\eta(t))a  \asconv0.
\end{equation}
We move onto removing $\eta$ from the previous identity. To this end, note that every positivity, existence, uniqueness, and comparison argument in Step 3 of the proof of Theorem \ref{thm:main-det-eq} remains valid with $\Lambda$ replaced by $\Lambda+tA$. This is  because $tA\succeq0$. In particular, the bounds \eqref{eq:proof-S-gap}-\eqref{eq:proof-r-eta-bound} hold uniformly for $0\leq t\leq h_0$. Thus, arguing as in the derivations of \eqref{eq:proof-feta-f-bound} and \eqref{eq:proof-r-eta-bound}, we get
\begin{align}
 \sup_{0\leq t\leq h_0} \norm{\Qhat_\eta(t)-\Qhat(t)}_{2}&\leq K\eta, \label{eq:mc-Q-eta-random}\\
 \sup_{0\leq t\leq h_0}\max_c \abs{e_{c,\eta}(t)-e_c(t)}&\leq K\eta. \label{eq:mc-e-eta}
\end{align}
On the other hand, subtracting the deterministic inverse matrices gives
\[
 \bar Q_\eta(t)^{-1}-\bar Q(t)^{-1} = \eta I+\sum_{c=1}^C\pi_cw_cS_c \rbr{\frac1{1+e_{c,\eta}(t)}- \frac1{1+e_c(t)}}.
\]
Thus, using \eqref{eq:mc-e-eta}, and uniform boundedness of $\pi_c$, $w_c$ and $S_c$, it holds
\[
     \sup_{0\leq t\leq h_0}\norm{\bar Q_\eta(t)^{-1}-\bar Q(t)^{-1}}_2 \leq K \eta. 
\]
Another inverse matrix identity and the bounds on $\norm{\bar Q_\eta(t)}_2$ and $\norm{\bar Q(t)}_2$, derived in the same way as \eqref{eq:proof-Q-bar-bound} and \eqref{eq:qbar_bound}, give
\begin{equation} \label{eq:mc-Q-eta-deterministic}
 \sup_{0\leq t\leq h_0} \norm{\bar Q_\eta(t)-\bar Q(t)}_{2} \leq K\eta.
\end{equation}
Then, it holds
\begin{align*}
    \abs{H(t) - \cH(t)} &= \abs{a^\top(\Qhat(t)-\bar Q(t))a}  \\
                &\leq\abs{a^\top(\Qhat(t)-\Qhat_\eta(t))a} +\abs{a^\top(\Qhat_\eta(t)-\bar Q_\eta(t))a} +\abs{a^\top(\bar Q_\eta (t)-\bar Q(t))a}\\
                &\leq K\eta \norm{a}_2^2 + \abs{a^\top(\Qhat_\eta(t)-\bar Q_\eta(t))a},
\end{align*}
where we have used \eqref{eq:mc-Q-eta-random} and \eqref{eq:mc-Q-eta-deterministic} in the last inequality. We then apply \eqref{eq:mc-H-eta-convergence}, and take a countable sequence $\eta\downarrow0$ as done in Step 3 of the previous section, thus proving \eqref{eq:mc-H-convergence} for every fixed $t$.

We next turn to the derivative convergence. We can rely on Step 4 of the previous section, specialized to the case of $\Lambda(t)=\Lambda+tA$. The goal is to give bounds on second derivatives of $H_A(t)$ and $\cH_A(t)$, so that we can use Taylor's theorem to transfer the convergence to the derivatives.

Firstly, expanding the second derivative of the empirical term, we have
\[
 \frac{d^2}{dt^2}H_A(t)=2a^\top\Qhat_A(t)A\Qhat_A(t)A\Qhat_A(t)a.
\]
Using the bounds on $\Qhat_A(t)$ from previous section, we get eventually almost surely,
\[
 \sup_{0\leq t\leq h_0}\abs{\frac{d^2}{dt^2}H_A(t)} \leq2\norm{a}_2^2\norm{A}_{2}^2\kappa_{\rm E}^{-3} \leq K_8.
\]
We next aim to obtain a bound for the second derivative of $\cH_A(t)$. To this end, we use the same map $\Phi_c(\mathbf e,t)=w_c\Tr(S_cQ(\mathbf e,t))/n$ as in \eqref{eq:Qet_phiet_def} of Step 4 of the previous section, and all the accompanying derivations. Note that, in the case we are considering, we can calculate the partial derivatives $t$ explicitly as
\[
 \partial_t Q(\mathbf e,t)^{-1}=A, \qquad \partial_{tt}Q(\mathbf e,t)^{-1}=0.
\]
The Jacobian has exactly the form \eqref{eq:jacobian_entries} and moreover by \eqref{eq:jacobian_bound} it holds
\[
 (J(\mathbf{1}+\mathbf{e}))_c=e_c(t)-\frac{w_c}{n} \Tr\bigl(S_c\bar Q_A(t)(\Lambda+tA) \bar Q_A(t)\bigr) \leq e_c(t),
\]
where $\mathbf{1}$ denotes the all-ones vector. Since $A$ is
bounded, the partial derivative bounds of Step 4 and the two implicit function
identities \eqref{eq:proof-e-prime-equation} and \eqref{eq:proof-e-second-equation} still apply. Namely, through \eqref{eq:proof-e-prime-bound} and \eqref{eq:proof-e-secundum-bound}, they imply
\begin{equation}\label{eq:new_e_bound}
     \sup_{0\leq t\leq h_0} \rbr{\norm{\mathbf e'(t)}_2+ \norm{\mathbf e''(t)}_2}<K.
\end{equation}
For completeness, let $M(t)=\bar Q_A(t)^{-1}$. Then
\begin{align*}
 M'(t)&=A-\sum_c\frac{\pi_cw_ce_c'(t)}{(1+e_c(t))^2}S_c,\\
 M''(t)&=-\sum_c\pi_cw_c \sbr{\frac{e_c''(t)}{(1+e_c(t))^2} -\frac{2(e_c'(t))^2}{(1+e_c(t))^3}}S_c,\\
 \bar Q_A''(t) &=2\bar Q_A(t)M'(t)\bar Q_A(t)M'(t)\bar Q_A(t) -\bar Q_A(t)M''(t)\bar Q_A(t).
\end{align*}
By \eqref{eq:new_e_bound}, all these matrices are uniformly bounded in operator norm. Therefore
\begin{equation} \label{eq:mc-H-second-deterministic}
 \sup_{0\leq t\leq h_0}\abs{\frac{d^2}{dt^2}\cH_A(t)} \leq\norm{a}_2^2 \sup_{0\leq t\leq h_0}\norm{\bar Q_A''(t)}_{2} \leq K_9.
\end{equation}
Lastly, for every fixed $0<h<h_0$, Taylor's theorem gives precisely the bound used in the original Step~4
\begin{align*}
 \abs{\frac{d}{dt}H_A(t)\Big|_{t=0}-\frac{d}{dt}\cH_A(t)\Big|_{t=0}}\leq\frac{K_8+K_9}{2}h+\frac{|H_A(h)-\cH_A(h)|+|H_A(0)-\cH_A(0)|}{h}.
\end{align*}
Since $h$ is fixed, the already established deterministic equivalent
convergence of \eqref{eq:mc-H-convergence} gives
\[
    H(0)-\cH(0)\asconv0, \qquad H(h)-\cH(h)\asconv0.
\]
Therefore, for fixed $h$, it holds almost surely 
\[
    \limsup_{n\to\infty}  \abs{\frac{d}{dt}H_A(t)\Big|_{t=0}-\frac{d}{dt}\cH_A(t)\Big|_{t=0}} \leq \frac{K_8+K_9}{2}h.
\]

Finally, let us take a deterministic sequence $h_m\downarrow0$.  Since this sequence is countable, we may intersect the probability-one events corresponding to all parameter values $h_m$.  On this common event, the preceding inequality holds for every $m$.  Sending $m\to\infty$ then  proves \eqref{eq:mc-H-derivative-convergence}.
\end{proof}



\subsection{Concentration of the realized risk} \label{app:realized-risk}

We strengthen the convergence statement of Theorem \ref{thm:main-det-eq}. To this end, let us define 
\begin{equation}\label{eq:hatR}
\hat{R}(\hat\beta;\beta) \coloneqq \norm{\hat\theta-\theta}_{S_1}^2.
\end{equation}
We will prove the slightly stronger statement that $\hat{R}(\hat\beta;\beta)$ almost surely concentrates around $B(\thetahat;\theta)+ V(\thetahat;\theta)$, and thus around $\cB+ \cV$.

By Assumption~(\ref{ass:priv_budg_large}),
clipping is almost surely inactive for all sufficiently large $n$,
and \eqref{eq:thetahat_theta} therefore gives
\begin{equation}\label{eq:aux-realized-risk-expansion}
\norm{\hat\theta-\theta}_{S_1}^2 = \theta^\top\Lambda\Qhat S_1\Qhat\Lambda\theta -2\theta^\top\Lambda\Qhat S_1\Qhat \hat{u} + \hat{u}^\top\Qhat S_1\Qhat \hat{u},
\end{equation}
where the first term is precisely $B(\thetahat;\theta)$. We next show that the second term vanishes almost surely and that the third concentrates around the variance expression in \eqref{eq:biasvarclosed}. Let
\[
 \tilde{X}\coloneqq \matrix{ X_1L_1\\ w_2X_2L_2\\ \vdots\\ w_CX_CL_C}, \qquad \tilde{\varepsilon}\coloneqq
 \matrix{\tilde{\varepsilon}_1\\\tilde{\varepsilon}_2\\\vdots\\\tilde{\varepsilon}_C},
\]

so that $\hat{u}=n^{-1}\tilde{X}^\top \tilde\varepsilon$. Note that under the spectral bounds established
in the proof of Theorem~\ref{thm:main-det-eq},
\[
 \norm{\Qhat}_{2}=O(1),\qquad
 \norm{S_1}_{2}=O(1),\qquad
 \left\|\frac{1}{n}\tilde{X}^\top\tilde{X}\right\|_{2}=O(1)
\]

almost surely for all sufficiently large $n$. Moreover, by Assumption \eqref{ass:z-moment-bounds} it holds that $\norm{\theta}_2$ is uniformly bounded almost surely.

Consider now the first the cross term in
\eqref{eq:aux-realized-risk-expansion} and define $a_n\coloneqq \frac{2}{n}\tilde{X}\Qhat S_1\Qhat\Lambda\theta$, so that
\[
 2\theta^\top\Lambda\Qhat S_1\Qhat \hat{u}=\tilde{\varepsilon}^\top a_n.
\]
The preceding spectral bounds imply, for some constant $K$,
\[
 \norm{a_n}_2^2 \leq \frac{K}{n^2} \theta^\top\Lambda\Qhat S_1\Qhat \tilde{X}^\top\tilde{X} \Qhat S_1\Qhat\Lambda\theta = O(n^{-1})
\]
almost surely. Since the entries of $\tilde{\varepsilon}$ are independent, centered, and have uniformly bounded $4$-th moments, the standard fourth-moment bound for linear forms gives
\[
 \bbE\sbr{|\tilde{\varepsilon}^\top a_n|^4 \,\middle|\,X,\theta} \leq K\|a_n\|_2^4 = O(n^{-2}).
\]
Hence, by Markov's inequality and the Borel-Cantelli lemma,
\begin{equation}\label{eq:aux-cross-noise-concentration}
\theta^\top\Lambda\Qhat S_1\Qhat \hat{u} \asconv0.
\end{equation}

We now consider the quadratic noise term. Namely we express it as 
\[
 \hat{u}^\top\Qhat S_1\Qhat \hat{u}=\tilde{\varepsilon}^\top M\tilde{\varepsilon},
\]
for the matrix term $M\coloneqq \frac1{n^2}\tilde{X}\Qhat S_1\Qhat\tilde{X}^\top$. Moreover, let $\Omega\coloneqq \operatorname{diag} \rbr{ \sigma_1^2I_{n_1}, \sigma_{t,2}^2I_{n_2}, \ldots, \sigma_{t,C}^2I_{n_C} }$ denote the covariance of $\tilde{\varepsilon}$. Then, the deterministic centering term associated with this quadratic form is
\[
\Tr(M\Omega) =  V(\thetahat;\theta).
\]
Furthermore,
\[
 \norm{M}_{2} \leq \frac{\norm{\tilde{X}}_{2}^2}{n^2}
 \norm{\Qhat S_1\Qhat}_{2} = O(n^{-1}),
\]
and, since $\operatorname{rank}(M)\leq Cd=O(n)$,
\[
 \norm{M}_{F}^2 \leq \operatorname{rank}(M)\norm{M}_{2}^2 = O(n^{-1}).
\]
For independent, centered entries with uniformly bounded $8$-th moments, the fourth-moment bound for quadratic forms yields
\[
 \bbE\sbr{ \abs{\tilde\varepsilon^\top M\tilde\varepsilon-\Tr(M\Omega)}^4 \,\middle|\,X,\theta} \leq K\norm{M}_{F}^4 = O(n^{-2}).
\]
Another application of Markov's inequality and Borel-Cantelli therefore gives
\begin{equation} \label{eq:aux-quadratic-noise-concentration}
\hat{u}^\top\Qhat S_1\Qhat \hat{u}- V(\thetahat;\theta) \asconv0.
\end{equation}
Combining \eqref{eq:aux-realized-risk-expansion}, \eqref{eq:aux-cross-noise-concentration}, and \eqref{eq:aux-quadratic-noise-concentration}, we conclude that
\begin{equation}\label{eq:aux-realized-risk-concentration}
\norm{\hat\theta-\theta}_{S_1}^2
-\rbr{B(\thetahat;\theta)+V(\thetahat;\theta)}
\asconv 0.
\end{equation}
Finally, using Theorem \ref{thm:main-det-eq} we get
\[
    \abs{\hat{R}(\betahat,\beta) - (\cB+\cV)}\asconv 0.
\]


\subsection{Proof of Proposition~\ref{prop:bayes-opt-param}}
\label{app:proof-bayes-opt-param}

We prove that the choice of weights and ridge parameters in
\eqref{eq:Bayespar} makes the estimator in \eqref{eq:delta-estimator}
equal to the posterior mean, and that its target component minimizes the prediction risk. Throughout the proof, we condition on the design matrices $X_1,\ \ldots,\ X_C$, which are therefore treated as fixed.

Recall that, under label privatization and Gaussian observation noise assumption, the target observations satisfy
\[
    y_1=X_1L_1\theta+\varepsilon_1, \qquad \varepsilon_1\sim\cN(0,\sigma_1^2I_{n_1}),
\]

while, for each $c\in\cA$,
\[
    \tilde y_c = X_cL_c\theta+\tilde\varepsilon_c, \qquad \tilde\varepsilon_c \coloneqq \varepsilon_c+\xi_c \sim \cN(0,\stc^2I_{n_c}),
\]
For simplicity let us introduce $\sigma_{t,1} \coloneqq \sigma_1$ and $\tilde{y}_1 = y_1$. Thus, conditional on $\theta$ and the designs, we have
\[
\tilde{y}_c\mid\theta,X_c \sim \cN\rbr{X_cL_c\theta,\,\stc^2I_{n_c}}.
\]
Since the observation and privacy noises are independent across clients, the responses are conditionally independent given $\theta$ and the designs. Hence the likelihood factorizes as
\begin{equation}\label{eq:bayes-likelihood}
    \cL(\theta) \coloneqq p\rbr{\tilde y_1,\ldots,\tilde y_C \middle| \theta,X_1,\ldots,X_C} = \prod_{c\in\cC} \frac{1}{(2\pi\stc^2)^{n_c/2}} \exp\rbr{-\frac{1}{2\stc^2} \norm{\tilde y_c-X_cL_c\theta}_2^2}.
\end{equation}
Suppressing factors that do not depend on $\theta$, we get
\[
    \cL(\theta) \propto \exp\rbr{- \sum_{c\in\cC} \frac{1}{2\stc^2} \norm{\tilde y_c-X_cL_c\theta}_2^2 }.
\]
Thus, up to an additive constant independent of $\theta$, the negative log-likelihood is
\begin{equation}
\label{eq:bayes-negative-log-likelihood}
-\log\cL(\theta) =  \sum_{c\in\cC} \frac{1}{2\stc^2} \norm{\tilde y_c-X_cL_c\theta}_2^2 + \mathrm{const}.
\end{equation}
On the other hand, by Assumption \eqref{eq:ass_gaus_prop}, the parameter $\theta$ has the Gaussian prior
\[
     \theta\sim\cN(0,\Gamma),\text{ where }  \quad \Gamma = \frac{r_\beta^2}{d}P_\beta + \sum_{c\in\cA} \frac{r_{\delta_c}^2}{d}P_{\delta_c}.
\]
Its precision matrix is therefore
\begin{equation}
\label{eq:bayes-prior-precision}
\Gamma^{-1} = \frac{d}{r_\beta^2}P_\beta + \sum_{c\in\cA} \frac{d}{r_{\delta_c}^2}P_{\delta_c}.
\end{equation}

Let $u = \matrix{b^\top & h_2^\top & \cdots & h_C^\top}^\top$ denote a generic value of the stacked parameter. By Bayes' rule, the posterior density of $u$, conditional on the designs and responses, is
proportional to the product of its prior density and the likelihood. Thus, up to an additive constant independent of \(u\), the negative
log-posterior is
\begin{equation}\label{eq:bayes-negative-log-posterior}
-\log p\rbr{u \mid X_1,\ldots,X_C, \tilde y_1,\ldots,\tilde y_C}
= \sum_{c\in\cC} \frac{1}{2\stc^2} \norm{\tilde y_c-X_cL_cu}_2^2 + \frac{1}{2}u^\top\Gamma^{-1}u + \mathrm{const}.
\end{equation}
Using \eqref{eq:bayes-prior-precision}, the prior term can equivalently
be written as
\[
    \frac{1}{2}u^\top\Gamma^{-1}u = \frac{d}{2r_\beta^2}\norm{b}_2^2 + \sum_{c\in\cA} \frac{d}{2r_{\delta_c}^2}\norm{h_c}_2^2.
\]
We next identify the posterior distribution explicitly. Expanding the
quadratic terms in \eqref{eq:bayes-negative-log-posterior}, and
discarding terms that do not depend on $u$, gives
\[
    -\log p(u \mid X_1,\ldots,X_C, \tilde y_1,\ldots,\tilde y_C) = \frac12u^\top\Omega u-u^\top q+\mathrm{const},
\]
where
\begin{align}
    \Omega &= \Gamma^{-1} + \sum_{c\in\cC} \frac{1}{\stc^2} L_c^\top X_c^\top X_cL_c,\label{eq:bayes-posterior-precision} \\
    q &= \sum_{c\in\cC} \frac{1}{\stc^2} L_c^\top X_c^\top\tilde y_c. \label{eq:bayes-posterior-linear}
\end{align}
Since \(\Gamma^{-1}\) is positive definite, $\Omega$ is also positive
definite and hence invertible. Setting $m\coloneqq\Omega^{-1}q$, we complete the square:
\[
\frac{1}{2}u^\top\Omega u-u^\top q = \frac{1}{2}(u-m)^\top\Omega(u-m) - \frac{1}{2}q^\top\Omega^{-1}q.
\]
Thus it holds that
\[
    p\rbr{u \mid X_1,\ldots,X_C, \tilde y_1,\ldots,\tilde y_C} \propto \exp\rbr{- \frac{1}{2}(u-m)^\top\Omega(u-m)},
\]
from which it directly follows that
\[
    \theta \,\mid\, X_1,\ldots,X_C,\tilde y_1,\ldots,\tilde y_C \sim \cN(m,\Omega^{-1}),
\]
and therefore
\begin{equation}
\label{eq:bayes-posterior-mean-general}
\bbE\sbr{\theta \middle|  X_1,\ldots,X_C, \tilde y_1,\ldots,\tilde y_C} = m = \Omega^{-1}q.
\end{equation}
We now show that this posterior mean is precisely the estimator
\eqref{eq:delta-estimator} evaluated at the parameters in
\eqref{eq:Bayespar}. Recall that
\[
    \Shat_c = \frac{1}{n}L_c^\top X_c^\top X_cL_c, \qquad \Uhat_c = \frac{1}{n}L_c^\top X_c^\top\tilde y_c.
\]
Let us introduce the ridge matrix as
\[
    \Lambda^\circ \coloneqq \lambda_\beta^\circ P_\beta + \sum_{c\in\cA} \lambda_{\delta_c}^\circ P_{\delta_c},
\]
with $\lambda_\beta^\circ$, $\lambda_{\delta_c}^\circ$ defined in \eqref{eq:Bayespar}. From the definition of the ridge parameters, it holds
\begin{equation}\label{eq:bayes-prior-ridge-match}
\frac{n}{\sigma_1^2}\Lambda^\circ = \frac{d}{r_\beta^2}P_\beta + \sum_{c\in\cA}
\frac{d}{r_{\delta_c}^2}P_{\delta_c} =\Gamma^{-1}.
\end{equation}
Similarly, using $w_c^\circ$ defined in \eqref{eq:Bayespar}, we get
\[
\frac{n}{\sigma_1^2} \rbr{\Shat_1 + \sum_{c\in\cA}w_c^\circ\Shat_c} = \frac{1}{\sigma_1^2} L_1^\top X_1^\top X_1L_1 + \sum_{c\in\cA} \frac{1}{\stc^2} L_c^\top X_c^\top X_cL_c. \]
Combining this identity with \eqref{eq:bayes-posterior-precision} and \eqref{eq:bayes-prior-ridge-match} yields
\begin{equation}
\label{eq:bayes-Omega-ridge}
\Omega = \frac{n}{\sigma_1^2} \rbr{\Shat_1 + \sum_{c\in\cA}w_c^\circ\Shat_c + \Lambda^\circ}.
\end{equation}
Likewise, from \eqref{eq:bayes-posterior-linear},
\begin{equation}
\label{eq:bayes-q-ridge}
q = \frac{n}{\sigma_1^2} \rbr{\Uhat_1 + \sum_{c\in\cA}w_c^\circ\Uhat_c}.
\end{equation}
Substituting \eqref{eq:bayes-Omega-ridge} and \eqref{eq:bayes-q-ridge} into \eqref{eq:bayes-posterior-mean-general} gives
\[
\bbE\sbr{\theta \mid X_1,\ldots,X_C, y_1,\tilde y_2,\ldots,\tilde y_C} =
\rbr{\Shat_1 + \sum_{c\in\cA}w_c^\circ\Shat_c + \Lambda^\circ}^{-1}
\rbr{\Uhat_1 + \sum_{c\in\cA}w_c^\circ\Uhat_c}.
\]
The right-hand side is exactly the closed-form expression for the
estimator in \eqref{eq:delta-estimator}, as expressed in \eqref{eq:theta_closed_form} with parameters $(\mathbf w^\circ,\lambda_\beta^\circ, \boldsymbol\lambda_\delta^\circ)$. Hence if we denote this estimator by $\hat\theta^\circ$, we have
\[
\hat\theta^\circ = \bbE\sbr{\theta\mid X_1,\ldots,X_C,\tilde y_1,\ldots,\tilde y_C}.
\]
Taking the first $d$-dimensional block gives
\[
\hat\beta^\circ = \bbE\sbr{\beta \mid X_1,\ldots,X_C, \tilde y_1,\ldots,\tilde y_C}.
\]
It remains to prove that \(\hat\beta^\circ\) minimizes $\mathbb E_\theta[R(\hat\beta;\beta)]$. Let us introduce the sigma algebra $\cD \coloneqq \sigma\rbr{X_1,\ldots,X_C, \tilde y_1,\ldots,\tilde y_C}$, generated by the observed data, and let
\[
    m_\beta \coloneqq \bbE[\beta\mid\cD] = \hat\beta^\circ.
\]
Consider any other $\cD$-measurable estimator $\tilde{\beta}$. We have the decomposition
\[
    \tilde{\beta}-\beta = (\tilde\beta-m_\beta) + (m_\beta-\beta).
\]
Therefore,
\begin{equation} \label{eq:bayes-risk-decomposition-pre}
    \norm{\tilde\beta-\beta}_{\Sigma_1}^2 = \norm{\tilde\beta-m_\beta}_{\Sigma_1}^2 + \norm{m_\beta-\beta}_{\Sigma_1}^2 + 2(\tilde\beta-m_\beta)^\top \Sigma_1(m_\beta-\beta).
\end{equation}
Taking the conditional expectation given $\cD$, the first term is
unchanged because both $\tilde{\beta}$ and $m_\beta$ are
$\cD$-measurable. For the cross term,
\[
    \bbE[m_\beta-\beta\mid\cD] = m_\beta-\bbE[\beta\mid\cD] = 0.
\]
Hence \eqref{eq:bayes-risk-decomposition-pre} gives
\begin{equation}\label{eq:bayes-posterior-risk-decomposition}
\bbE\sbr{\norm{\tilde\beta-\beta}_{\Sigma_1}^2 \mid \cD}= 
\norm{\tilde\beta-\hat\beta^\circ}_{\Sigma_1}^2 +
\bbE\sbr{\norm{\hat\beta^\circ-\beta}_{\Sigma_1}^2 \mid \cD}.
\end{equation}
Since $\Sigma_1\succ0 $, the first term on the right-hand side is
nonnegative, and it is equal to zero if and only if
\(\tilde\beta=\hat\beta^\circ\). Therefore, for every realization of
the observed data,
\[
    \hat\beta^\circ = \bbE[\beta\mid\cD]
\]
uniquely minimizes the posterior expected target prediction loss.

Finally, recall that
\[
   \mathbb E_\theta [R(\tilde\beta;\beta)] = \bbE\sbr{\norm{\tilde\beta-\beta}_{\Sigma_1}^2 \mid X_1,\ldots,X_C}.
\]
Applying the tower property to
\eqref{eq:bayes-posterior-risk-decomposition} gives
\[
\mathbb E_\theta[R(\tilde\beta;\beta)] = \mathbb E_\theta[R(\hat\beta^\circ;\beta)] + \bbE\sbr{\norm{\tilde\beta-\hat\beta^\circ}_{\Sigma_1}^2\mid X_1,\ldots,X_C}.
\]
The second term is nonnegative, which proves that $\hat\beta^\circ$ is the desired minimizer and concludes the argument.
\qed

\section{Experimental details and additional numerical results}\label{app:experiments}

\subsection{Experimental details} \label{app:experimental-details} 

\paragraph{Dataset preprocessing.} For \emph{Student Performance}, we use $649$ observations from the two schools and $38$ encoded features after removing the intermediate grades G1 and G2 and constant columns. For \emph{LSAC}, removing observations with missing outcomes, client identifiers, or selected predictors leaves $20{,}970$ observations across $6$ law-school tiers and $10$ encoded features. For \emph{Lianjia}, we use $300{,}284$ observations across $9$ Beijing districts and the $6$ selected features. Missing or invalid numerical values are imputed with medians.


\paragraph{Estimation of distributional quantities.}

The deterministic equivalents \eqref{eq:C-client-bias-variance-DE} and \eqref{eq:bias_simpler} depend on the unknown distributional quantities $\{\Sigma_c,\sigma_c\}_{c=1}^C$, together with either the relevant inner products between $\{\beta_c\}_{c=1}^C$ or the magnitudes $\tau_{\beta}$ and $\{\tau_{\delta_c}\}_{c=2}^C$, respectively. In both experiments of Figure \ref{fig:exact_fit} and Figure \ref{fig:source_selection}, using designated non-private data, we estimate the relevant regression coefficients and noise variances by five-fold cross-fitting. To stabilize covariance estimation, we apply $10\%$ isotropic shrinkage to the covariance estimate, i.e. we mix the regular empirical covariance estimate with an equal variance isotropic covariance matrix, outputting $0.9\Sigmahat_c+0.1(\Tr(\Sigmahat_c)/d)I$, where $\Sigmahat_c$ is the standard empirical covariance matrix as in \eqref{eq:emp_cov}. The two experiments differ, however, in which non-private observations are available for estimating these quantities. For the experiment on the prediction of the risk change in Figure \ref{fig:exact_fit}, our goal is to test how accurately the deterministic equivalent predicts empirical changes in target risk with data from real datasets, when its distributional quantities are estimated as accurately as possible. We therefore estimate the distributional quantities in each repetition using all de-privatized training and validation observations available for each client. We then evaluate the simplified deterministic equivalent \eqref{eq:bias_simpler}, which averages over signal and model shift directions. On the other hand, for the dataset selection experiment in Figure~\ref{fig:source_selection}, we instead simulate a setting in which expected dataset benefit must be predicted before accessing the observations used for model fitting or evaluation. We therefore split each of the datasets into two roughly same-sized partitions, one designated for estimating distributional quantities, and the other for the actual experiment. The distributional quantities are thus estimated once from this fixed non-private reference partition that is disjoint from all training, validation, and test observations. These estimates are held fixed across different privacy noise levels and experimental repetitions. Client selection uses the general, direction-aware deterministic equivalent \eqref{eq:C-client-bias-variance-DE}.  In both experiments, the used training sample sizes and the auxiliary privacy noise variances of all clients are assumed to be known when evaluating the deterministic equivalent.

\paragraph{Hyperparameter tuning.}
The experiments in Figures~\ref{fig:exact_fit} and~\ref{fig:source_selection} also use different hyperparameter-selection procedures. For the experiment on the prediction of the risk change in  Figure~\ref{fig:exact_fit}, in order to select hyperparameters, we first do a pilot experiment with $100$ repetitions. Each pilot repetition uses a randomized target training-validation split and all available auxiliary observations, with the corresponding privacy noise added to the auxiliary responses. For each client pair and privacy level, we jointly select the ridge penalties and auxiliary weight that minimize the average target-validation MSE across these pilot repetitions. The selected parameters are then fixed across all auxiliary sample sizes and all main experiment $100$ evaluation repetitions, with identical values used by the empirical estimator and the deterministic-equivalent prediction. This allows us to isolate how accurately the deterministic equivalent predicts empirical changes in target risk with data from real datasets. On the other hand, for the dataset selection experiment in Figure \ref{fig:source_selection}, the empirical ridge penalties and auxiliary weights are chosen using target-validation MSE in each evaluation repetition, privacy condition and candidate client subset. Hyperparameter tuning uses that repetition's sampled training observations, including privatized auxiliary responses. All selection methods share the same hyperparameters for a given client subset and repetition. To rank candidate clients, the deterministic equivalent used by \emph{Ours} instead uses the population-based hyperparameter formulas from Proposition~\ref{prop:bayes-opt-param}, evaluated using estimates from the independent non-private reference partition. These selection hyperparameters remain fixed across repetitions for each privacy level and may differ from those used by the empirically evaluated estimators. This reflects a setting in which the expected benefit of auxiliary data must be predicted before accessing the observations used for model fitting or evaluation.

\subsection{Semi-synthetic experiment}\label{app:semi_synth}

\begin{figure}[t]
    \centering
     \includegraphics[width=0.32\textwidth]{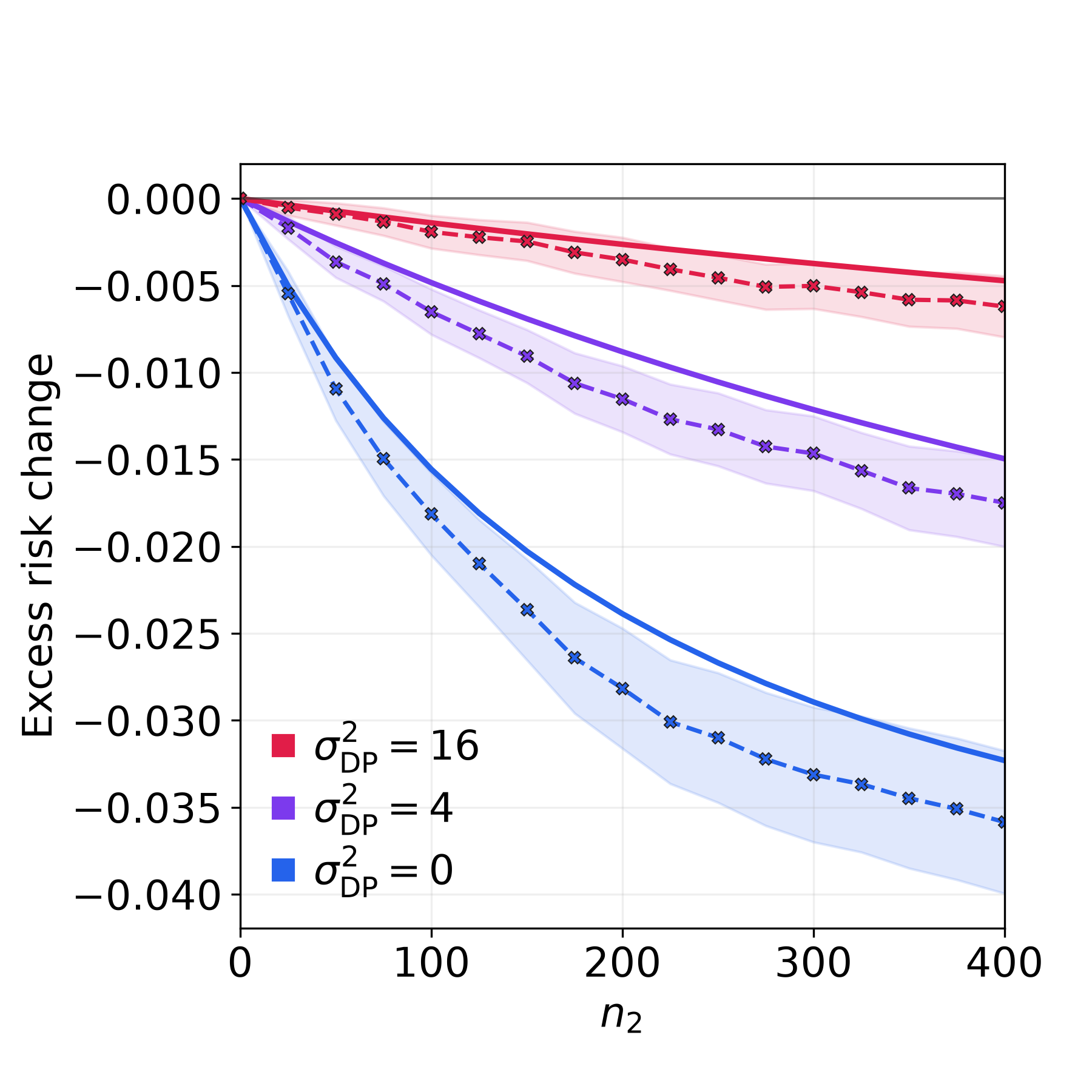}
    \hfill
    \includegraphics[width=0.32\textwidth]{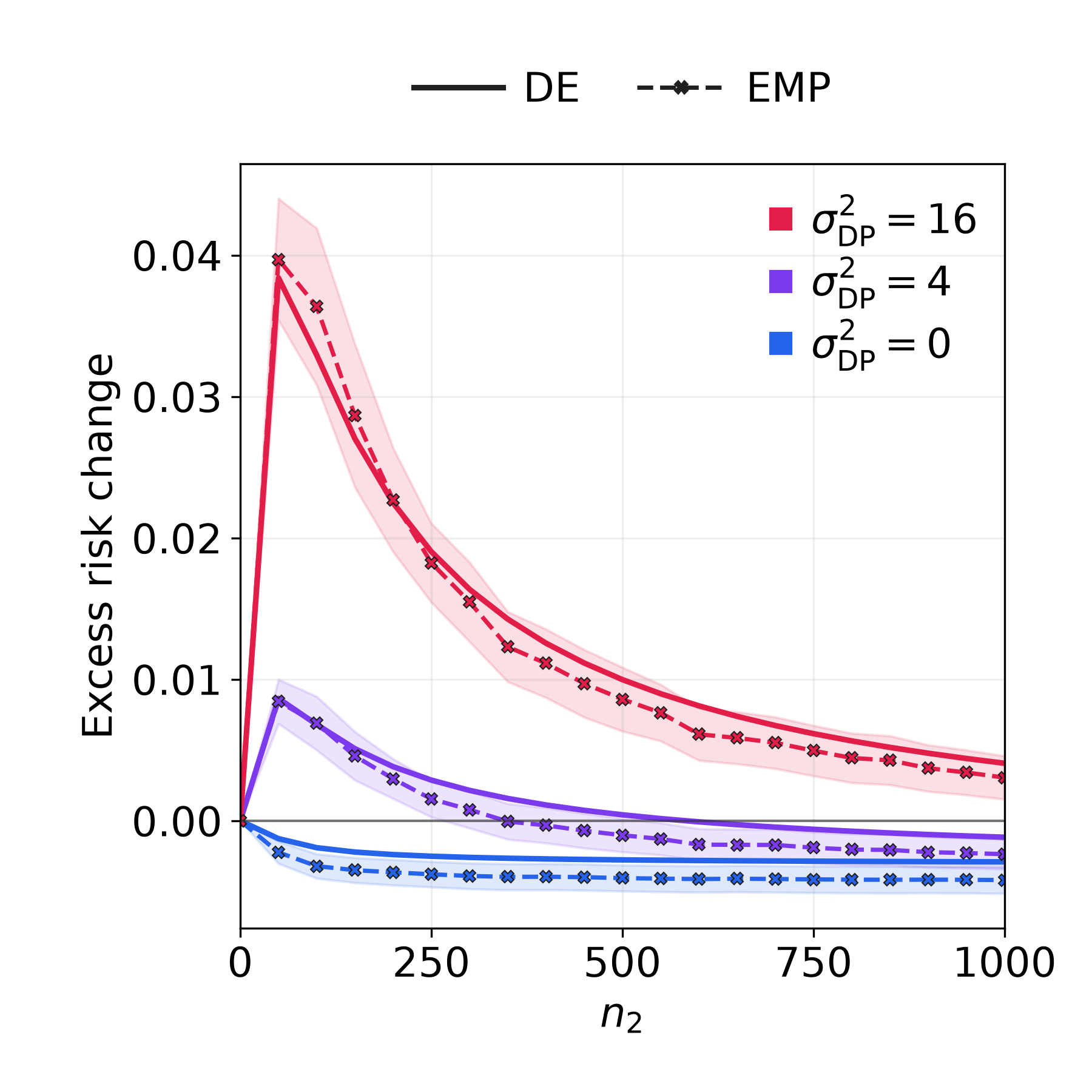}
    \hfill
    \includegraphics[width=0.32\textwidth]{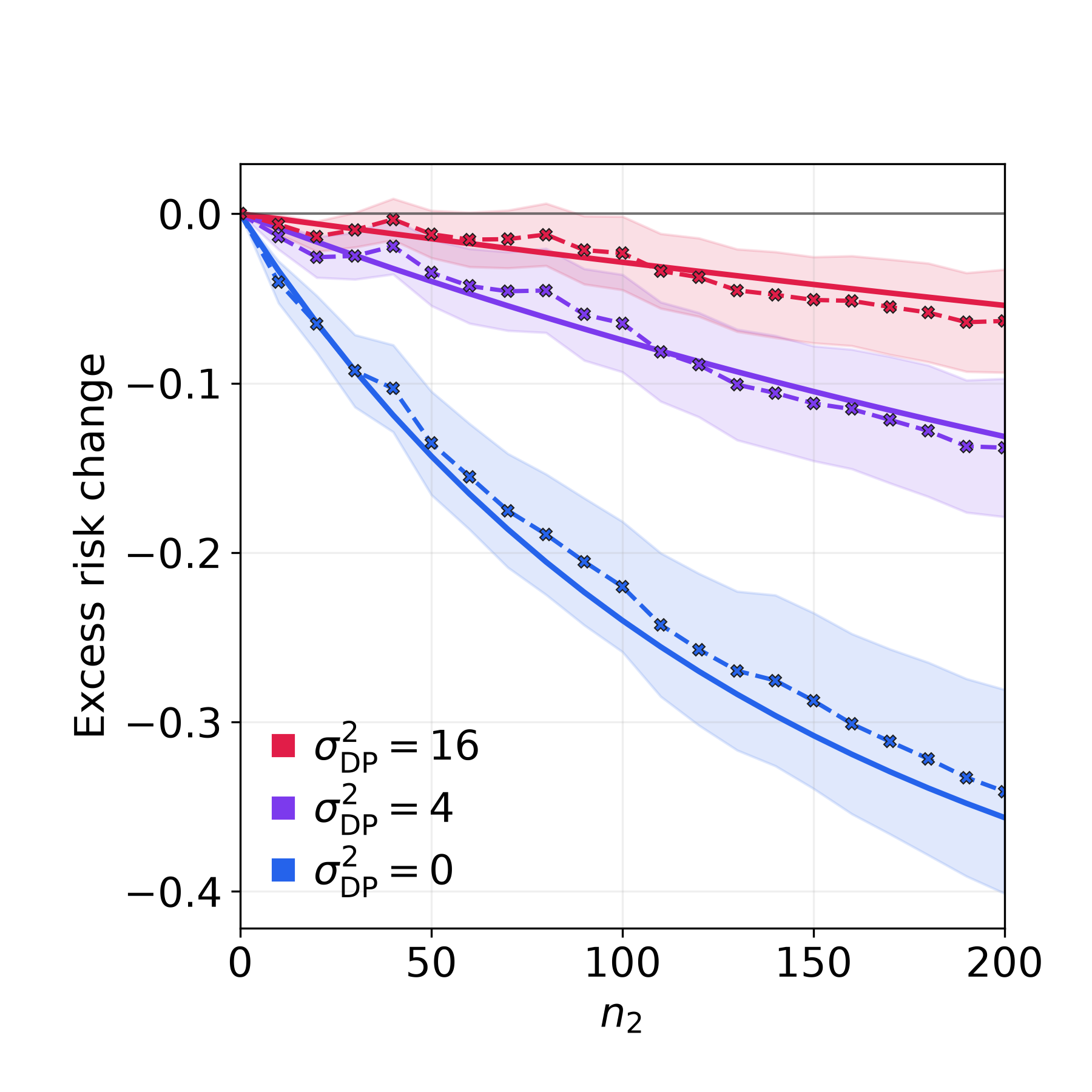}
    \makebox[\textwidth]{%
        \parbox{0.333\textwidth}{\centering\textbf{(a)} LSAC}%
        \parbox{0.333\textwidth}{\centering\textbf{(b)} Lianjia}%
        \parbox{0.333\textwidth}{\centering\textbf{(c)} Student Performance}%
        }
    \caption{Change in target excess risk as auxiliary data are added in two-client semi-synthetic regression experiments. Lower values correspond to a better risk. Original covariates are retained, while responses are generated from planted linear models calibrated to each client.  \textbf{(a) LSAC:}  Tier 3 as the target client, and Tier 4 as the auxiliary client. \textbf{(b) Lianjia:} District 8 as the target client, and District 7 as the auxiliary client. \textbf{(c) Student Performance:} Mousinho da Silveira (MS) as the target school, and Gabriel Pereira (GP) as the auxiliary school. Solid lines show direction-aware deterministic equivalent predictions (DE) using the known generating parameters; dashed lines with crosses show empirical results (EMP). }
\label{fig:semi_synthetic_direction_aware}
\end{figure}    

To isolate the accuracy of the deterministic equivalent from the error in estimating population-level quantities, we construct a semi-synthetic version of the experiment in Figure \ref{fig:exact_fit}. We use the original covariates and client definitions, but replace the responses with data generated from planted linear models. Namely, using all data, we first estimate all hidden signals $\beta_c$ and noise variances $\sigma_c^2$ by five-fold cross-fitting. We then use these estimates, denoted by $\beta_c^*$ and ${\sigma_c^*}^2$, to generate responses
\[
y_c^*=X_c\beta_c^*+\varepsilon_c^*, \qquad \varepsilon_c^*\sim\mathcal N(0,{\sigma_c^*}^2 I).
\]
The covariates and planted signals remain fixed, while observation noise, data partitions, and privacy noise are resampled in each repetition.

As in the original experiment from Figure \ref{fig:exact_fit}, we fix the target training size at $n_1=100$ and vary the auxiliary sample size up to $200$, $400$, and $1000$ for Student Performance, LSAC, and Lianjia, respectively. We then add to the auxiliary responses independent Gaussian noise with variance $\sigma_{\mathrm{DP}}^2\in\{0,4,16\}$. We then repeat the hyperparameter tuning setup of the experiment in Figure \ref{fig:exact_fit}. Namely, we run a pilot experiment in which the ridge penalties and auxiliary weights are jointly selected by minimizing the average target-validation MSE over $100$ repetitions using all available auxiliary observations. The selected hyperparameters are then fixed across auxiliary sample sizes and the $100$ evaluation repetitions, and are used identically by the empirical estimator and the deterministic equivalent.

 The deterministic equivalent is evaluated using the known planted signals, noise variances, and empirical covariance computed with $10\%$ isotropic shrinkage, as in the original experiment. As we have exact access to planted signals, we use the direction-aware deterministic equivalent expression of \eqref{eq:C-client-bias-variance-DE}. Figure \ref{fig:semi_synthetic_direction_aware} reports the change in target risk relative to the target-only estimator. Dashed lines with crosses show empirical means and pointwise $90\%$ confidence intervals, while solid lines show the deterministic-equivalent predictions. 

Removing the error in estimating population-level quantities and imposing a correctly specified linear response model improves the accuracy of theoretical predictions for the datasets Student Performance and Lianjia, as can be seen by comparing plots \textbf{(b)}-\textbf{(c)} in Figure \ref{fig:exact_fit} and Figure \ref{fig:semi_synthetic_direction_aware}. For the LSAC dataset, the discrepancy between empirical behaviour and theoretical predictions changes less between the two figures, as it is driven primarily by the variance term, rather than by estimation error or model misspecification. This is consistent with LSAC covariates having large fourth moments, which cause stronger finite sample fluctuations in the empirical covariance matrices than are captured by the asymptotic approximation via the deterministic equivalent.

\subsection{Dataset selection varying privacy levels}\label{app:dataset_selection}

In Tables \ref{tab:moderate_largest_selection_percentage_q_0}-\ref{tab:moderate_largest_selection_percentage_q_16}, we present results of dataset selection experiments for different privacy noise levels. As before, the values shown are changes in the target excess risk, expressed as a percentage of each repetition's target-only test risk, for varying auxiliary privacy noise variance $\sigma_{\mathrm{DP}}^2$. Each entry is the mean $\pm$ $90\%$ confidence interval computed over $100$ repetitions, in percentage points. Lower is better. Bold marks the lowest mean in each column and dataset. 

\begin{table}[htbp]
\centering
\small
\begin{tabular}{llccc}
\toprule
Dataset & Method & 1 Client & 2 Clients & 3 Clients \\
\midrule
 & Ours & $\mathbf{-2.41 \pm 0.32}$ & $\mathbf{-3.05 \pm 0.33}$ & $\mathbf{-3.05 \pm 0.33}$ \\
 & Random & $-1.35 \pm 0.18$ & $-2.20 \pm 0.26$ & $-2.75 \pm 0.29$ \\
LSAC & Largest sample & $-1.32 \pm 0.26$ & $-2.81 \pm 0.32$ & $-2.88 \pm 0.32$ \\
 & Best local risk & $-0.94 \pm 0.20$ & $-1.93 \pm 0.25$ & $-2.72 \pm 0.30$ \\
\midrule
 & Lower bound (oracle) & $-2.94 \pm 0.31$ & $-3.39 \pm 0.32$ & $-3.52 \pm 0.33$ \\
\midrule
\midrule
 & Ours & $\mathbf{-2.15 \pm 0.70}$ & $\mathbf{-3.97 \pm 0.83}$ & $\mathbf{-4.40 \pm 0.85}$ \\
 & Random & $-1.69 \pm 0.46$ & $-2.64 \pm 0.64$ & $-3.29 \pm 0.74$ \\
Lianjia & Largest sample & $-1.70 \pm 0.57$ & $-2.85 \pm 0.76$ & $-3.11 \pm 0.77$ \\
 & Best local risk & $-1.65 \pm 0.54$ & $-2.51 \pm 0.66$ & $-3.17 \pm 0.72$ \\
\midrule
 & Lower bound (oracle) & $-3.42 \pm 0.69$ & $-4.22 \pm 0.83$ & $-4.51 \pm 0.84$ \\
\bottomrule
\end{tabular}
\caption{Change in target excess risk for auxiliary privacy noise
variance $\sigma_{\mathrm{DP}}^2=0$.}
\label{tab:moderate_largest_selection_percentage_q_0}
\end{table}



\begin{table}[htbp]
\centering
\small
\begin{tabular}{llccc}
\toprule
Dataset & Method & 1 Client & 2 Clients & 3 Clients \\
\midrule
 & Ours & $\mathbf{-1.56 \pm 0.27}$ & $\mathbf{-2.52 \pm 0.32}$ & $\mathbf{-2.84 \pm 0.33}$ \\
 & Random & $-1.04 \pm 0.15$ & $-1.79 \pm 0.23$ & $-2.36 \pm 0.27$ \\
LSAC & Largest sample & $-0.95 \pm 0.20$ & $-2.07 \pm 0.27$ & $-2.25 \pm 0.28$ \\
 & Best local risk & $-0.80 \pm 0.19$ & $-1.67 \pm 0.26$ & $-2.35 \pm 0.30$ \\
\midrule
 & Lower bound (oracle) & $-2.53 \pm 0.29$ & $-3.02 \pm 0.32$ & $-3.19 \pm 0.33$ \\
\midrule
\midrule
 & Ours & $\mathbf{-1.95 \pm 0.63}$ & $\mathbf{-3.54 \pm 0.81}$ & $\mathbf{-4.08 \pm 0.82}$ \\
 & Random & $-1.60 \pm 0.44$ & $-2.51 \pm 0.61$ & $-3.14 \pm 0.71$ \\
Lianjia & Largest sample & $-1.63 \pm 0.56$ & $-2.85 \pm 0.76$ & $-3.11 \pm 0.76$ \\
 & Best local risk & $-1.72 \pm 0.59$ & $-2.60 \pm 0.67$ & $-3.09 \pm 0.74$ \\
\midrule
 & Lower bound (oracle) & $-3.33 \pm 0.69$ & $-4.08 \pm 0.80$ & $-4.36 \pm 0.83$ \\
\bottomrule
\end{tabular}
\caption{Change in target excess risk for auxiliary privacy noise
variance $\sigma_{\mathrm{DP}}^2=1$.}
\label{tab:moderate_largest_selection_percentage_q_1}
\end{table}


\begin{table}[htbp]
\centering
\small
\begin{tabular}{llccc}
\toprule
Dataset & Method & 1 Client & 2 Clients & 3 Clients \\
\midrule
 & Ours & $\mathbf{-1.02 \pm 0.20}$ & $\mathbf{-1.93 \pm 0.29}$ & $\mathbf{-2.31 \pm 0.29}$ \\
 & Random & $-0.78 \pm 0.12$ & $-1.42 \pm 0.18$ & $-1.94 \pm 0.23$ \\
LSAC & Largest sample & $-0.69 \pm 0.17$ & $-1.45 \pm 0.22$ & $-1.66 \pm 0.23$ \\
 & Best local risk & $-0.69 \pm 0.18$ & $-1.50 \pm 0.26$ & $-2.10 \pm 0.26$ \\
\midrule
 & Lower bound (oracle) & $-2.12 \pm 0.24$ & $-2.59 \pm 0.28$ & $-2.77 \pm 0.30$ \\
\midrule
\midrule
 & Ours & $\mathbf{-1.87 \pm 0.55}$ & $\mathbf{-3.14 \pm 0.77}$ & $\mathbf{-3.67 \pm 0.78}$ \\
 & Random & $-1.37 \pm 0.37$ & $-2.28 \pm 0.56$ & $-2.92 \pm 0.67$ \\
Lianjia & Largest sample & $-1.52 \pm 0.53$ & $-2.69 \pm 0.75$ & $-3.01 \pm 0.75$ \\
 & Best local risk & $-1.63 \pm 0.57$ & $-2.48 \pm 0.67$ & $-3.06 \pm 0.72$ \\
\midrule
 & Lower bound (oracle) & $-3.19 \pm 0.67$ & $-3.99 \pm 0.77$ & $-4.21 \pm 0.81$ \\
\bottomrule
\end{tabular}
\caption{Change in target excess risk for auxiliary privacy noise
variance $\sigma_{\mathrm{DP}}^2=4$.}
\label{tab:moderate_largest_selection_percentage_q_4}
\end{table}


\begin{table}[htbp]
\centering
\small
\begin{tabular}{llccc}
\toprule
Dataset & Method & 1 Client & 2 Clients & 3 Clients \\
\midrule
 & Ours & $-0.60 \pm 0.16$ & $\mathbf{-1.49 \pm 0.27}$ & $-1.90 \pm 0.28$ \\
 & Random & $-0.62 \pm 0.10$ & $-1.17 \pm 0.17$ & $-1.66 \pm 0.22$ \\
LSAC & Largest sample & $-0.52 \pm 0.16$ & $-1.05 \pm 0.20$ & $-1.32 \pm 0.21$ \\
 & Best local risk & $\mathbf{-0.66 \pm 0.19}$ & $-1.28 \pm 0.25$ & $\mathbf{-1.93 \pm 0.28}$ \\
\midrule
 & Lower bound (oracle) & $-1.87 \pm 0.24$ & $-2.30 \pm 0.27$ & $-2.48 \pm 0.29$ \\
\midrule
\midrule
 & Ours & $\mathbf{-1.40 \pm 0.54}$ & $-2.43 \pm 0.67$ & $\mathbf{-3.17 \pm 0.68}$ \\
 & Random & $-1.06 \pm 0.28$ & $-1.92 \pm 0.47$ & $-2.57 \pm 0.60$ \\
Lianjia & Largest sample & $-1.28 \pm 0.46$ & $-2.27 \pm 0.66$ & $-2.74 \pm 0.70$ \\
 & Best local risk & $-1.36 \pm 0.45$ & $\mathbf{-2.46 \pm 0.70}$ & $-2.97 \pm 0.75$ \\
\midrule
 & Lower bound (oracle) & $-3.09 \pm 0.67$ & $-3.83 \pm 0.77$ & $-4.08 \pm 0.80$ \\
\bottomrule
\end{tabular}
\caption{Change in target excess risk for auxiliary privacy noise
variance $\sigma_{\mathrm{DP}}^2=16$.}
\label{tab:moderate_largest_selection_percentage_q_16}
\end{table}

\end{document}